\documentclass{article}
\usepackage{ijcai23}

\usepackage{times}
\usepackage{soul}
\usepackage{url}
\usepackage[hidelinks]{hyperref}
\usepackage[utf8]{inputenc}
\usepackage[small]{caption}
\usepackage{graphicx}
\usepackage{amsmath}
\usepackage{amsthm}
\usepackage{booktabs}
\usepackage{algorithm}
\usepackage{algorithmic}
\usepackage[switch]{lineno}
\usepackage{times}
\usepackage{lipsum}
\usepackage{latexsym}
\usepackage{tikz}
\newcommand{\tikzxmark}{%
\tikz[scale=0.23] {
    \draw[line width=0.7,line cap=round] (0,0) to [bend left=6] (1,1);
    \draw[line width=0.7,line cap=round] (0.2,0.95) to [bend right=3] (0.8,0.05);
}}
\newtheorem{theorem}{Theorem}

\usepackage[utf8]{inputenc} % allow utf-8 input
\usepackage{url}            % simple URL typesetting
\usepackage{amsfonts}       % blackboard math symbols
\usepackage{nicefrac}       % compact symbols for 1/2, etc.
\usepackage{microtype}      % microtypography
\usepackage{xcolor}         % colors

\usepackage[american]{babel}
\usepackage{mathtools} % amsmath with fixes and additions
\usepackage{siunitx} % for proper typesetting of numbers and units
\usepackage{tikz} % nice language for creating drawings and diagrams
\usepackage{overpic}
\usepackage{amssymb, amsmath, amsthm, stmaryrd}
\usepackage{dsfont, mathrsfs, enumitem, url, xr}
\usepackage{enumitem}
\usepackage{color}
\usepackage{tcolorbox}
\usepackage{subfigure}
\usepackage{bm, bbm}
\usepackage{multirow, titlecaps, accents, makebox, epstopdf, mathtools}
\renewcommand{\div}{\mathrm{div}}
\newcommand{\grad}{\mathrm{grad}\,}

\newcommand{\tr}{\text{tr}}

\newcommand{\R}{{\mathbb R}}

\newcommand{\D}{{\mathbb D}}

\newcommand{\mS}{{\mathbb S}}

\newcommand{\e}{\varepsilon} 

\newcommand{\m}{\mathbf{m}}

\newcommand{\calP}{{\mathcal P}}

\newcommand{\calF}{{\mathcal F}}

\newcommand{\rmL}{{\mathrm L}}

\newcommand{\dd}{\mathrm{d}}

\newcommand{\Min}{\min\limits_}

\newcommand{\mc}{\mathcal}
\newcommand{\Let}{\coloneqq}
\newcommand{\be}{\begin{equation}}
\newcommand{\ee}{\end{equation}}

\newcommand{\opt}{^\star}

\newcommand{\mbb}{\mathbb}

\def\longequals{\mathbin{=\kern-2pt=}}

\newtheorem{definition}[theorem]{Definition}
\newtheorem{remark}[theorem]{Remark}
\newtheorem{lemma}[theorem]{Lemma}
\newtheorem{proposition}[theorem]{Proposition}

\numberwithin{equation}{section}

\newcommand{\beq}{\begin{equation}}
\newcommand{\eeq}{\end{equation}}

\usepackage{lipsum} % for dummy text
\usepackage{enumitem}
\setlist{nosep} % or \setlist{noitemsep} to leave space around whole list

\title{Dynamic Flows on Curved Space Generated by Labeled Data}

\author{
Xinru Hua$^1$
\and
Truyen Nguyen$^2$\and
Tam Le$^{3}$\And
Jose Blanchet$^1$\and
Viet Anh Nguyen$^4$
\affiliations
$^1$ Stanford University\\
$^2$The University of Akron\\
$^3$The Institute of Statistical Mathematics / RIKEN AIP\\
$^4$Chinese University of Hong Kong\\
}

\begin{document}

\maketitle

\begin{abstract}
The scarcity of labeled data is a long-standing challenge for many machine learning tasks. We propose our gradient flow method to leverage the existing dataset (i.e., source) to generate new samples that are close to the dataset of interest (i.e., target). We lift both datasets to the space of probability distributions on the feature-Gaussian manifold, and then develop a gradient flow method that minimizes the maximum mean discrepancy loss. To perform the gradient flow of distributions on the curved feature-Gaussian space, we unravel the Riemannian structure of the space and compute explicitly the  Riemannian gradient of the loss function induced by the optimal transport metric. For practical applications, we also propose a discretized flow, and provide conditional results guaranteeing the global convergence of the flow to the optimum. We illustrate the results of our proposed gradient flow method on several real-world datasets and show our method can improve the accuracy of classification models in transfer learning settings.
\end{abstract}

%%%%%%%%%%%%%%%%%%%%%%%
%%%%%%%%%%%%%%%%%%%%%%%
%\tam{Submissions must be formatted using the UAI latex template and formatting instructions. Papers must be submitted as a PDF file and are \textbf{limited to 8 pages in length}, including all figures and tables. \textbf{At most two additional pages containing only references} are allowed.}

\section{Introduction}
A major challenge in many data science applications is the scarcity of labeled data. Data augmentation methods have been studied in the literature; see for example, the noise injection methods~\cite{8628917}, generative models~\cite{yi2019generative}, and~\cite{shorten2019survey} for a survey.
% Another popular approach is developing learning methods that can interpolate,  adapt, or  transfer knowledge across datasets and domains. Some well-known methods for these tasks are domain adaptation
% %NIPS2008_0e65972d
% \cite{NIPS2006_b1b0432c,NIPS2017_0070d23b, DBLP:journals/corr/abs-1803-10081,6247911,taigman2016unsupervised,pmlr-v80-hoffman18a},  transfer learning \cite{pmlr-v70-long17a,5288526, Zamir_2018_CVPR}, and meta-learning \cite{pmlr-v70-finn17a,NEURIPS2019_f4aa0dd9}.
We consider a setting where one domain has only a few labeled samples for each class, so we cannot train a well-performing classifier with the available data. To alleviate the data scarcity problem in this setting, we propose to enrich the target dataset by generating additional labeled samples. Using generative models is not possible in our setting because they usually require more than a few samples for each class to learn and generate high-quality new samples~\cite{gao2018low}. In our work, we choose a source dataset with extensive labeled data and then flow the labeled data to the target dataset. Precisely, we introduce a novel data augmentation methodology based on a gradient flow approach that minimizes the maximum mean discrepancy ($\mathrm{MMD}$) distance between the target and the augmented data. Therefore, when minimizing the $\mathrm{MMD}$ distance, we are able to obtain an efficient scheme which generates additional labeled data from the target distribution. Our scheme is model-independent and can be applied to any datasets regardless of the number of classes or dimensionality.
% Recently, these methods have produced promising results for important tasks in autonomous driving and robotics~\cite{wang2018pix2pixHD,46717}.
%{\color{blue}For example, domain adaptation is a promising direction in robotics and autonomous driving to reduce the sim-to-real gap~\cite{wang2018pix2pixHD,46717}. }

Mathematically, we consider a feature space $\mc X = \R^m$ and a \textit{categorical} label space $\mc Y$. We have a source domain dataset consisting of $N$ samples $(x_i, y_i) \in \mc X \times \mc Y$ for $i = 1, \ldots, N$, and a target domain dataset of $M$ samples $(\bar x_j, \bar y_j) \in \mc X \times \mc Y$ for $j = 1, \ldots, M(M\ll N)$. The ultimate goal of this paper is to generate new samples in the target domain, and we aim to generate new samples whose distribution is as close as possible to the distribution that governs the target domain.

We here introduce a gradient flow method
\cite{arbelKSG19,pmlr-v89-mroueh19a}
to synthesize new, unseen data samples.
Gradient flow is a continuous flow along the path where a considered loss function decreases its value. Because we have extensive source domain samples, it is possible to flow each source sample towards the target data while minimizing the loss function. The terminal product of the flow will be new samples that can sufficiently approximate the distribution of the target domain. Thus, gradient flow is an approach to synthesize new target domain samples, and is a complement to data augmentation methods, like adding random noise.

Unfortunately, formulating a gradient flow algorithm for labeled data with categorical set $\mc Y$ is problematic. Indeed, there is no clear metric structure on $\mc Y$ in order to define the topological neighborhood, this in turn leads to the difficulty of forming the gradients with respect to the categorical component. To overcome this difficulty, we lift each individual label to a richer structure. For example, a label such as ``0" is replaced by a mean vector and a covariance matrix based on the whole distribution of the information associated to this particular label. Then it will be much more natural to apply gradient flow algorithms in the space of the lifted representation. A gradient flow on the dataset space with this idea was recently proposed in~\cite{AlvarezGF20} by leveraging a new notion of distance between datasets in~\cite{NEURIPS2020_f52a7b26,NIPS2017_0070d23b,DBLP:journals/corr/abs-1803-10081}. The main idea behind this approach is to reparametrize the categorical space $\mc Y$ using the conditional distribution of the features, which is assumed to be Gaussian, and then construct a gradient flow on the feature-Gaussian space. Nevertheless, the theoretical analysis in~\cite{AlvarezGF20} focuses solely on the gradients with respect to the feature with no treatment of the flow with respect to the Gaussian component. In fact, the space of Gaussian distributions is not a (flat) vector space, and extracting gradient information depends on the choice of the metric imposed on this Gaussian space. On the other hand, our method computes the full gradient with respect to the Gaussian component (the mean and covariance matrix component that correspond to the label component).

%Choosing $\mathrm{MMD}$ as the objective function circumvents the \textit{curved} space $\mc P(\mc Z)$ by comparing distributions via their kernel mean embeddings on a \textit{flat} reproducing kernel Hilbert space (RKHS). In contrast to the Kullback-Leibler divergence flow, the $\mathrm{MMD}$ flow can employ a sample approximation for the target $\varrho$ and does not require the knowledge of an analytical form for $\varrho$~\cite{NIPS2017_17ed8abe}. Further, the squared $\mathrm{MMD}$ possesses unbiased sample gradients~\cite{binkowski2018demystifying,Bellemare2017}.

\begin{table*}[!ht]
  \centering 
    {
  \centering 
\resizebox{\textwidth}{!}{%
  \begin{tabular}{llcc}
    \toprule
    Paper     & Dataset & \begin{tabular}{@{}c@{}}On curved Riemannian space\end{tabular} &\begin{tabular}{@{}c@{}}Gradient has mean\\ and covariance component\end{tabular} \\
    \midrule
    \cite{AlvarezGF20} & synthetic, *NIST, and CIFAR10 & \tikzxmark  & \tikzxmark  \\
    \cite{arbelKSG19}  & synthetic  &\tikzxmark  & \tikzxmark  \\
    Ours & synthetic, *NIST, and TinyImageNet       & \checkmark & \checkmark  \\
    \bottomrule
  \end{tabular}
  }
  }
  \caption{\label{table:comparison} To the best of our knowledge, we provide the \emph{first} results on the full gradient of the features and lifted labels on a curved Riemannian space. We also conduct numerical experiments on the highest-dimension real-world datasets.}
\end{table*}

Our gradient flows minimize the $\mathrm{MMD}$ loss function, thus it belongs to the family of $\mathrm{MMD}$ gradient flows that was pioneered in~\cite{pmlr-v89-mroueh19a,arbelKSG19}, and further extended in~\cite{Mroueh21}. The $\mathrm{MMD}$ function compares two distributions via their kernel mean embeddings on a \textit{flat} reproducing kernel Hilbert space (RKHS). In contrast to the Kullback-Leibler divergence flow, the $\mathrm{MMD}$ flow can employ a sample approximation for the target distribution~\cite{NIPS2017_17ed8abe}. Further, the squared $\mathrm{MMD}$ possesses unbiased sample gradients~\cite{binkowski2018demystifying,Bellemare2017}. However, existing literature on $\mathrm{MMD}$ flows focus on distributions on flat Euclidean spaces. The flow developed in our paper here is for distributions on a \textit{curved} Riemannian feature-Gaussian space. Moreover, our approach is distinctive from the flow in~\cite{AlvarezGF20} because we impose a specific metric on the Gaussian component, and we compute explicitly the Riemannian gradient of the $\mathrm{MMD}$ loss function with respect to this metric to formulate our flow. Table~\ref{table:comparison} compares our work with two recent papers on gradient flow in theory and numerical experiments.

Recently, generative models~\cite{rezende2016one,pmlr-v139-wang21l} are successful in generating image samples from given distributions. The most important difference with our method is that generative models learn a prior distribution from massive data that are similar to the target data and generate new target samples conditioning on the prior distribution~\cite{wang2020generalizing,gao2018low}. Comparatively, our algorithm can transfer between two non-similar and non-related distributions, for example, from random Gaussian noise to MNIST in Supplementary~\ref{appendix:random_gauss_mnist}. Another benefit of our method is that we provide conditions for global convergence of our algorithms in Section~\ref{sec:MMD-flow}, whereas generative models or more specific, generative adversarial networks (GANs), currently do not guarantee global convergence~\cite{wiatrak2019stabilizing}.

The application of our gradient flow is few-shot transfer learning, where we want to train classifiers with limited labeled data in the target domain. The numerical experiments in Section~\ref{sec:numerical} demonstrate that our gradient flows can effectively augment the target data, and thus can significantly boost the accuracy in the classification task in the few-shot learning setting. Moreover, we run experiments on Tiny ImageNet datasets to highlight that our algorithm is scalable to higher-dimensional image data, that is higher than recent gradient flow works~\cite{AlvarezGF20,fan2022generating}. We also compare our method with~\cite{AlvarezGF20}, mixup method~\cite{zhang2017mixup}, and traditional data augmentation methods. Results in Supplementary~\ref{sec:comparison}--\ref{sec:comparison_transformation} show that our method improves the accuracy in transfer learning more than these methods.

Some related works study nonparametric gradient flows using the $2$-Wasserstein distance between distributions~\cite{ref:ambrosio2008gradient,JKO,Otto2001,ref:villani2008optimal,Santambrogio15,Santambrogio17,frogner2020approximate}, but only for distributions on Euclidean spaces and a different distance. Nonparametric gradient flows with other metrics include  Sliced-Wasserstein Descent~\cite{pmlr-v97-liutkus19a}, Stein Descent~\cite{NIPS2017_17ed8abe,NIPS2016_b3ba8f1b}, and Sobolev Descent~\cite{pmlr-v89-mroueh19a}. However, they also only consider distributions on Euclidean spaces. In particular, \cite{NIPS2017_17ed8abe} introduce Riemannian structures for the Stein geometry on flat spaces, while ours is for an optimal transport metric on a curved space. Parametric flows for training GANs are studied in~\cite{NEURIPS2018_a1afc58c,Chen2018,Arbel2020Kernelized,Mroueh21}. 

\paragraph{Contributions.} We study a gradient flow approach to synthesize new labeled samples related to the target domain. To construct this flow, we consider the space of probability distributions on the feature-Gaussian manifold, and we are metrizing this space with an optimal transport distance. We summarize the contributions of this paper as follows.
\begin{itemize}[leftmargin=4mm, topsep=2pt,itemsep=1pt, partopsep=1pt, parsep=1pt]
    \item We study in details the Riemannian structure of the feature-Gaussian manifold in Section~\ref{sec:riemannian}, as well as the Riemannian structure of the space of 
    probability 
    measures supported on this manifold in Supplementary~\ref{sec:geometry-PZ}.
    %of the appendix.
    \item We consider a gradient flow that minimizes the squared $\mathrm{MMD}$ loss function to the target distribution. We describe explicitly the (Riemannian) gradient of the squared $\mathrm{MMD}$ in Lemma~\ref{grad-formula}, and we provide a partial differential equation describing the evolution of the gradient flow that follows the (Riemannian) steepest descent direction.
    \item We propose two discretized schemes to approximate the continuous gradient flow equation in Section~\ref{sec:euler} and~\ref{sec:flow}. We provide conditions guaranteeing the global convergence of our gradient flows to the optimum 
    in both 
    %continuous and discretized
    schemes. 
    \item In Section~\ref{sec:numerical}, we demonstrate numerical results with our method on real-world image datasets. We show that our method can generate high-fidelity images and improve the classification accuracy in transfer learning settings.
\end{itemize}

\paragraph{Notations.} 
We use $\mS^n$ to denote the set of $n\times n$ real and symmetric matrices, and
 $\mS^n_{++} \subset \mS^n$ consists of all positive definite matrices. 
 For  $A\in \mS^n$, $\tr(A) \Let \sum_{i} A_{ii}$.
 We use $\langle \cdot, \cdot\rangle$ and $\|\cdot\|_2$ to denote the standard inner product and norm on Euclidean spaces.
% $\langle x, y\rangle \Let \sum_{i} x_i y_i$ and 
%$\|x\|_2^2  \Let  \langle x, x\rangle$.
 %the Frobenius inner product and norm on $\mS^n$: 
%$\langle A, B\rangle \Let \tr(AB) = \sum_{i, j} A_{i j} B_{i j}$ and 
%$\|A\|_2  \Let  \sqrt{\langle A, A\rangle} = \sqrt{\tr(A^2)}$.
%The same notations are also adopted for vectors, but with standard Euclidean inner product and norm.
%For $a\in \R^n$,   $aa^\top$ denotes the $n\times n$ matrix whose $(i,j)$ entry is $a_i a_j$.
 %Let $\calP(X)$ be the collection of all Borel probability distributions with finite second moment on the given set $X$.
 Let $\calP(X)$  be the collection of all probability distributions with finite second moment on  metric space $X$. 
 If $\varphi: X\to Y$ is a Borel map and $\nu \in \calP(X)$, then the push-forward $\varphi_{\#}\nu$ is the distribution on $Y$ given by 
 $\varphi_{\#}\nu(E) = \nu(\varphi^{-1}(E))$ for all Borel sets $E\subset Y$. 
 For a function $f$ of the continuous time variable $t$, $f_t$ denotes the value of $f$ at  $t$ while $\partial_t f$ denotes the standard derivative of $f$ w.r.t. $t$.
Also, $\delta_z$ denotes the Dirac delta measure at $z$.

All proofs are provided in the Supplementary material.
 %For a distribution $\rho$ on a space $X$, the $L^2_\rho$ inner product is defined by $\langle f, g\rangle_{L^2_\rho} \Let \int_{X} f(x) g(x) \, \rho(\dd x)$ for  two functions $f$ and  $g$ on $X$.

%%%%%%%%%%%%%%%%%%%%%%
%%%%%%%%%%%%%%%%%%%%%%

\section{Labeled Data Synthesis via Gradient Flows of Lifted Distributions}\label{sec:sampling}
\label{sec:dataset}

In this section, we describe our approach to synthesize target domain samples using gradient flows. A holistic view of our method is presented in Fig.~\ref{fig:workflow}.

\begin{figure*}[h!]
\centering
   \includegraphics[width=16cm]{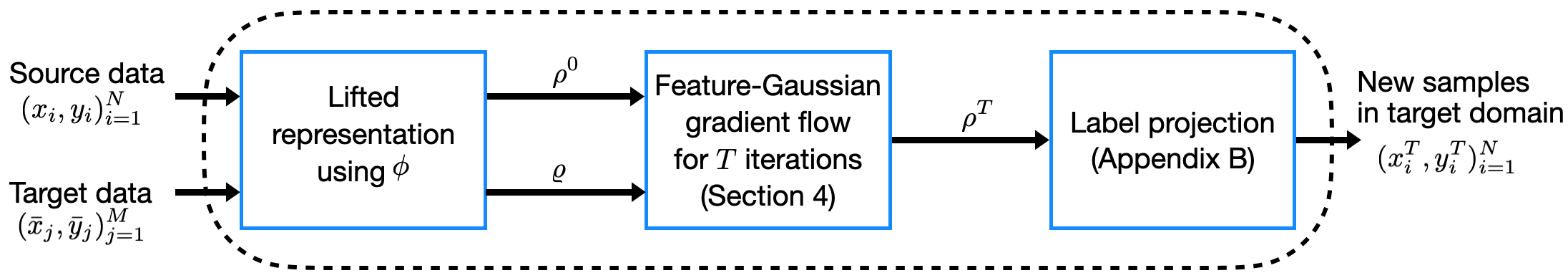}   \caption{Schematic view of our approach: The source and target datasets are first lifted to distributions $\rho^0$ and $\varrho$ on the feature-Gaussian space (left box). We then run a gradient flow for $T$ iterations to get a terminal distribution $\rho^T$ (middle). Atoms of $\rho^T$ are projected to get labeled target samples (right).}
\label{fig:workflow} 
\end{figure*}
In the first step, we would need to lift the feature-label space $\mc X \times \mc Y$ to a higher dimensional space where a metric can be defined. Consider momentarily the source data samples $(x_i, y_i)_{i=1}^N$. Notice that this data can be represented as an empirical distribution $\nu$ on $\mc X \times \mc Y$. More precisely, we have $\nu = N^{-1}\sum_{i=1}^N \delta_{(x_i, y_i)}$. As $\mc Y$ is discrete, the law of conditional probabilities allows us to dis-integrate $\nu$ into the conditional distributions $\nu_y$ of $X | Y = y$ satisfying
%the identity
$\nu(E\times F) = \int_{F} \nu_y(E) \nu^2(\dd y)$ for every $E\subset \mc X$ and $F\subset \mc Y$, where  $\nu^2\Let N^{-1}\sum_{i=1}^N \delta_{y_i} $ is the second marginal of $\nu$ 
\cite[Theorem~5.3.1]{ref:ambrosio2008gradient}.
%\cite[Theorem~9.2.2]{ref:stroock2011probability}.
The lifting procedure is obtained by employing a pre-determined mapping $\phi: \mc X \to \R^n$, and any categorical value $y \in \mc Y$ can now be represented as an $n$-dimensional distribution $\phi_{\#} \nu_y$. Using this lifting, any source sample $(x_i, y_i) \in \mc X \times \mc Y$ is lifted to a point $(x_i, \phi_{\#} \nu_{y_i}) \in \mc X \times \mc P(\R^n)$ and the source dataset is representable as an empirical distribution of the form $N^{-1}\sum_{i=1}^N \delta_{(x_i, \phi_{\#} \nu_{y_i})}$.

The lifted representation of a categorical value $y \in \mc Y$ as an $n$-dimensional distribution $\phi_{\#} \nu_y \in \mc P(\R^n)$ is advantageous because $\mc P(\R^n)$ is metrizable, for example, using the $2$-Wasserstein distance. The downside
%, unfortunately, 
is that $\mc P(\R^n)$ is infinite dimensional, and encoding the datasets in this lifted representation is not efficient. To resolve this issue, we assume that $\phi_{\#} \nu_{y}$ is Gaussian for all $y \in \mc Y$, and thus any distribution $\phi_{\#} \nu_y$ can be characterized by the mean vector $\mu_y \in \R^n$ and covariance matrix $\Sigma_y \in \mS_{++}^n$ defined as
%\begin{align*}
  $  \mu_y = \int_{\mc X} \phi(x) \nu_y(\dd x) 
    \mbox{ and }
    \Sigma_y = \int_{\mc X} \big[ \phi(x)-\mu_y\big] \big[\phi(x) - \mu_y\big]^\top \nu_y(\dd x )$
%\end{align*}
for all $y \in \mc Y$, where $^\top$ denotes the transposition of a vector. In real-world settings, the conditional moments of $\phi(X)|Y$ are sufficiently different for $y \neq y'$, and thus the representations using $(\mu_y, \Sigma_y)$ will unlikely lead to any loss of label information. With this lifting, the source data thus can be represented as an empirical distribution $\rho^0$ on $\R^m \times \R^n \times \mS_{++}^n$ via
    %\begin{align*}
     $   \rho^0 = N^{-1} \sum\nolimits_{i=1}^N \delta_{(x_i, \mu_{y_i}, \Sigma_{y_i})}$.
    %\end{align*}
By an analogous construction to compute $\bar \mu_{y}$ and $\bar \Sigma_y$ using the target data, the target domain data $(\bar x_j, \bar y_j)_{j=1}^M$ can be represented as another empirical distribution
$
    \varrho = M^{-1} \sum\nolimits_{j=1}^M \delta_{(\bar x_j, \bar \mu_{\bar y_j}, \bar \Sigma_{\bar y_j})}$.
Let us denote the shorthand $\mc Z = \R^m \times \R^n \times \mS_{++}^n$, then $\rho^0$ and $\varrho$ are both probability measures on $\mc Z$. We refer to $\rho^0$ and $\varrho$ as the feature-Gaussian representations of the source and target datasets.

% with distance $d$ between two points $(x_1, \mu_1, \Sigma_1)$ and $(x_2, \mu_2, \Sigma_2)$ being  defined as the square root of $\| x_1 - x_2\|_2^2   + \calW\big(\mathcal N(\mu_1, \Sigma_1), \mathcal N(\mu_2, \Sigma_2)\big)^2$. 
We now consider the gradient flow associated with the optimization problem
\begin{equation*}
    \Min{\rho \in \mc P(\mc Z)}~\Big\{ \mc F(\rho) \Let \frac12 \mathrm{MMD}(\rho, \varrho)^2 \Big\}
\end{equation*}
    under the initialization $\rho = \rho^0$. The objective function $\mc F(\rho)$ quantifies how far an incumbent solution $\rho$ is from the target distribution $\varrho$, measured using the $\mathrm{MMD}$ distance. In Sections~\ref{sec:riemannian} and~\ref{sec:MMD-flow}, we will provide the necessary ingredients to construct this flow.
    
    %Our choice of $\mathrm{MMD}$ as the objective function stems from the fact that $\calP(\mc Z_{++})$ is a curved space, and to circumvent this problem we  propose to use a positive definite kernel on $\mc Z_{++}$ to map it via the mean embedding to a reproducing kernel Hilbert space (RKHS). The $\mathrm{MMD}$ between two distributions on $\mc Z_{++}$ is simply the distance in the RKHS between their mean embeddings \cite{JMLR:v13:gretton12a}, and  the squared $\mathrm{MMD}$ possesses unbiased sample gradients \cite{binkowski2018demystifying,Bellemare2017}. Unlike the use of the Kullback-Leibler divergence in Stein Descent \cite{NIPS2017_17ed8abe}, the $\mathrm{MMD}$ flow has an advantage that it may employ a sample approximation for the target $\varrho$ and  it does not require  the knowledge of an analytical form for $\varrho$.
    
   Suppose that after $T$ iterations of the discretized gradient flow algorithm, we obtain a distribution $\rho^T \in \mc P(\mc Z)$ that is sufficiently close to $\varrho$, i.e., $\mc F(\rho^T)$ is close to zero. Then we can recover new target labels by projecting the samples of the distribution $\rho^T$ to the locations on $\mc X \times \mc Y$. This projection can be computed efficiently by solving a linear optimization problem, as discussed in Supplementary~\ref{sec:label-projection}. 

\begin{remark}[Reduction of dimensions] If $m = n$ and $\phi$ is the identity map, then our lifting procedure coincides with that proposed in~\cite{NEURIPS2020_f52a7b26}. However, a large $n$ is redundant, especially when the cardinality of $\mc Y$ is low. If $n \ll m$, then $\phi$ offers significant reduction in the number of dimensions, and will speed up the gradient flow algorithms.
\end{remark}

\begin{remark}[Generalization to elliptical distributions]
Our framework can be extended to the symmetric elliptical distributions because the Bures distance for elliptical distributions admits the same closed-form as for the Gaussian  distributions 
\cite{ref:gelbrich1990formula}. In this paper, we use $\phi$ as the t-SNE embedding. According to \cite{vanDerMaaten2008}, t-SNE’s low-dimensional embedded space forms a Student-t distribution, which is an elliptical distribution. 
\end{remark}

%\viet{Here we need a workflow drawing}

%\viet{Some questions to be addressed:
%    \begin{itemize}
%        \item Why does the ground metric for $\mc P(\mc Z_+)$ is Wasserstein, but the objective function is MMD? We need a table with combinations of Obj/metric and detailed comparisons.
%        \item why do we need the Bures as a metric on $\mS_{++}^n$ but not the usual riemannian metric? (in the end, we can't handle singular matrices, either)... is it for computational issue?
%    \end{itemize}
%}

%%%%%%%%%%%%%%%%%%%%%%
%%%%%%%%%%%%%%%%%%%%%%

\section[Riemannian Geometry of the Spaces]{Riemannian Geometry of $\mc Z$ and $\mc P(\mc Z)$}
\label{sec:riemannian}
%For the space $\mc Z_{+}$ defined in Section~\ref{sec:sampling}, the feature and the covariance  have the same dimension $n$. However, for practical purpose from now on we take $\mc Z_+ = \R^m \times \R^n \times \mS_+^n$ and $\mc Z_{++} = \R^m \times \R^n \times \mS_{++}^n$.  This setup allows the possibility to reduce the dimension for high dimensional input space. For example, images can have millions of pixels, but if they belong to only $2$ classes, then $n$ small is enough to capture label information.
%Let $\mc Z_+ = \R^m \times \R^n \times \mS_+^n$ and $\mc Z_{++} = \R^m \times \R^n \times \mS_{++}^n$, which are clearly not linear vector spaces. 
If we opt to measure the distance between two Gaussian distributions using the 2-Wasserstein metric, then this choice would induce a natural distance $d$ on the space $\mc Z = \R^m \times \R^n \times \mS_{++}^n$ prescribed as
\begin{align}\label{ground-cost}
&d\big((x_1, \mu_1, \Sigma_1), (x_2, \mu_2, \Sigma_2)\big) \nonumber\\
&\hspace{1em}\Let  \big[\| x_1 - x_2\|_2^2  + \| \mu_1 - \mu_2\|_2^2 + \mathbb{B}(\Sigma_1,\Sigma_2)^2\big]^{\frac12},
\end{align}
where $\mathbb{B}$ is the Bures metric on $\mS^n_{++}$ given by
$\mathbb{B}(\Sigma_1,\Sigma_2) \Let \big[  \tr(\Sigma_1 + \Sigma_2 - 2  [\Sigma_1^{\frac12} \Sigma_2 \Sigma_1^{\frac12}]^{\frac12})\big]^{\frac12}$.
%\quad \mbox{for}\quad \Sigma_1, \Sigma_2\in\mS^n_+.
%\end{equation*}

As $\mathbb{B}$ is a metric on $\mS^n_+$ \cite[p.167]{ref:bhatia2019on}, $d$  is hence a product metric on $\mc Z$. In this section, first, we study the non-Euclidean geometry of  $\mc Z$ 
under the ground metric $d$. Second, we investigate the Riemannian structure on $\mc P(\mc Z)$, the space of all distributions supported on $\mc Z$ and with finite second moment, that is induced  by the optimal transport distance.
These Riemannian structures are required to define the Riemannian gradients 
%(w.r.t.~distance $\mathbb{W}$) 
of any loss functionals on  $\mc P(\mc Z)$, and will remain important in our development of the gradient flow for the squared $\mathrm{MMD}$.
%loss function.

% \subsection[Geometry]{Geometry of $\mc Z$}
% \label{sec:geometry-Z} 

The space $\mc Z$ is not a linear vector space. In this section, we reveal the Riemannian structure on $\mc Z$ associated to the ground 
metric $d$. As we shall see, $\mc Z$ is a curved space as its geodesics are
not straight lines and involve solutions to the Lyapunov equation. For any  positive definite matrix $\Sigma \in \mS^n_{++}$ and any symmetric matrix  $V\in \mS^n$,  the Lyapunov equation
\begin{equation}\label{Lyapunov}
H\Sigma + \Sigma H = V
\end{equation}
has a unique solution $H\in \mS^n$ \cite[Theorem~VII.2.1]{Bhatia97}.
%, \cite[p.144]{ref:malago2018wasserstein}.
Let $\rmL_\Sigma[V]$ denote this unique solution $H$.
%\begin{equation}\label{Lyapunov}
%\rmL_A[V] \Let \{ H \in \mS^n : HA + AH = V \}. 
%\end{equation}
%For brevity, we rewrite the product space $\mc Z_+$ compactly as  $\R^{2 n} \times \mS^n_{+}$. 
%Also, let $\mc Z_{++} \Let \R^{m+ n} \times \mS^n_{++}$.

The space $\mS^n_{++}$ is a Riemannian manifold 
with the Bures metric $\mathbb{B}$ as the associated distance function,  see \cite[Proposition A]{ref:takatsu2011wasserstein}. Since $\mc Z$ is the product of two Euclidean spaces 
and $\mS^n_{++}$, this gives rise to the following geometric structure for $\mc Z$. 
\begin{proposition}[Geometry of $\mc Z$]\label{prop:Riemmanian}
The space $\mc Z$ is a Riemannian manifold: at each point $z = (x, \mu, \Sigma)\in \mc Z$, the  tangent space is  
$\mathrm T_{z} \mc Z = \R^m \times \R^n \times \mS^n$  and  the Riemannian metric is
\begin{align}\label{R-metric-Z}
&\big\langle (w_1, v_1, V_1), (w_2, v_2, V_2) \big\rangle_{z}\notag\\
&\hspace{5em}\Let \langle w_1, w_2\rangle + \langle v_1, v_2\rangle + \langle V_1, V_2\rangle_\Sigma
\end{align}
for  two tangent vectors $(w_1,v_1,V_1)$ and $(w_2, v_2, V_2)$ in $\R^{m}\times\R^n\times \mS^n$,
where $\langle V_1, V_2\rangle_{\Sigma}   \Let \tr \Big(\rmL_\Sigma[V_1] \, \Sigma \, \rmL_\Sigma[V_2]\Big)$. 
%= \frac12\langle \rmL_\Sigma[V_1], V_2 \rangle
Moreover, the distance function corresponding to this Riemannian metric coincides with the distance $d$ given by 
\eqref{ground-cost}.
\end{proposition}
%Supplementary (Appendix~\ref{sec:proofs}).

As $\mc Z$ is a product Riemannian manifold, any geodesic in $\mc Z$ is of the form 
$ (\theta,\gamma,\Gamma)$ with $\theta$, $\gamma$ being the Euclidean geodesics (straight lines) and $\Gamma$ being a geodesic in the Riemannian manifold $\mS^n_{++}$. 
%For any two points $A_1, A_2 \in \mS^n_{++}$, the geodesic in $\mS^n_{++}$  connecting $A_1$ 
%and $A_2$ is the curve 
%\begin{align}\label{eq-geodesic}
%A(t) &= [(1-t) I + t M] A_1  [(1-t) I + t M]
%\\
%&= (1-t)^2 A_1 + t^2 A_2 + t(1-t) \Big[(A_1 A_2)^\frac12 + (A_2 A_1)^\frac12 \Big], \nonumber
%\end{align}
%where $M$ is the unique solution in $\mS^n_+$ of the Riccati equation $M A_1 M = A_2$. We note that $M = A_1^{-1} \# A_2 \Let %A_1^{-\frac12} (A_1^\frac12 A_2 A_1^\frac12 )^\frac12 A_1^{-\frac12}$, which is the geometric mean of $A_1^{-1}$ and $A_2$. 
More precisely, for each $\Sigma\in \mS^n_{++}$ and  each tangent vector $V\in \mS^n$, the geodesic in the manifold $\mS^n_{++}$ emanating from $\Sigma$ with direction $V$ is given by
\begin{equation}\label{geodesic}
\Gamma(t) =  ( I + t \rmL_\Sigma[V] ) \Sigma  ( I + t \rmL_\Sigma[V] )\quad \mbox{for } t\in J^*,
\end{equation}
where  $ J^*$ is the open interval about the origin given by $
 J^* = \{t\in \R: \,  I + t \rmL_\Sigma[V]  \in  \mS^n_{++} \}$~\cite{ref:malago2018wasserstein}.
As a consequence, for  each point $(x,\mu,\Sigma)\in \mc Z$ and each tangent vector  $(w,v,V)\in \R^{m}\times\R^n\times \mS^n$, the Riemannian exponential map  in $\mc Z$ for  $t\in J^*$ is given by
  \begin{equation}\label{exponential_map}
  \exp_{(x,\mu,\Sigma)}(t (w, v, V)) \Let (\theta(t),\gamma(t),\Gamma(t)).
  \end{equation}
where $\theta(t) \Let x + tw$,  $\gamma(t) \Let \mu + t v$,  and 
$\Gamma(t)$ is defined by \eqref{geodesic}. By definition, $t\mapsto \exp_{(x,\mu,\Sigma)}(t (w, v, V))$ 
 is the geodesic emanating from $(x,\mu, \Sigma)$ with direction $(w,v, V)$.

Given the Riemannian metric \eqref{R-metric-Z}, one can define the corresponding notion of gradient and divergence~\cite{ref:lee2003manifold}.
For a differentiable function $\varphi: \mc Z \to \R$, its gradient $\widetilde\nabla_d \varphi(z)$ w.r.t.~the metric $d$ defined by \eqref{ground-cost} is the unique element in the tangent space $\R^m \times \R^n \times \mS^n$ satisfying 
\begin{equation*}
 \big\langle \widetilde\nabla_d \varphi(z), (w,v,V) \big\rangle_{z} = D \varphi_z(w,v, V)
\end{equation*}
for all $(w,v, V) \in \R^{m}\times\R^n \times \mS^n$ with $ D \varphi_z(w,v,V)$ denoting the standard directional derivative of $\varphi$ at $z$ in the direction $(w,v,V)$.
By exploiting the special form of $\langle \cdot, \cdot\rangle_z$ in \eqref{R-metric-Z}, we can compute $\widetilde\nabla_d \varphi(z)$ explicitly:
%in terms of the Lyapunov operator:
\begin{lemma}[Gradients]\label{gradient-formula} For a differentiable function $\varphi: \mc Z \to \R$, we have for $z=(x,\mu,\Sigma)$ that
\begin{equation}\label{metric-gradient-dataset}
  \resizebox{0.99\linewidth}{!}{$
\widetilde\nabla_d \varphi(z)
=\big(\nabla_x\varphi(z),~\nabla_\mu\varphi(z),~2[\nabla_\Sigma \varphi(z)] \Sigma + 2\Sigma [ \nabla_\Sigma \varphi(z)]\big),$}
\end{equation}
where $(\nabla_x, \nabla_\mu, \nabla_\Sigma)$ are the standard (Euclidean) gradients of the respective components.
\end{lemma}
The last component in formula \eqref{metric-gradient-dataset} for $\widetilde\nabla_d \varphi$ reflects the curved geometry of $\mc Z$, and  can be interpreted as the Riemannian gradient of 
the function $\Sigma\mapsto \varphi(x,\mu,\Sigma)$ w.r.t.~the Bures distance $\mathbb{B}$.

%\viet{maybe we can write
%\[
%\nabla_d \varphi(z) = \Big(\nabla_x\varphi(z), \nabla_\mu\varphi(z),  2[\nabla_\Sigma \varphi(z)] \Sigma + 2\Sigma [ %\nabla_\Sigma \varphi(z)]\Big), 
%\]
%Throughout, I think we can write $\nabla_d$ for the gradient with the metric $d$, while $(\nabla_x, \nabla_\mu, \nabla_\Sigma)$ are the usual gradients (wrt the Euclidean metric) of the respective components?
%}
%Notice that the last component is the Bures-Wasserstein gradient  of $\varphi$ w.r.t. the $\Sigma$ variable.

For a continuous vector field $\Phi: \mc Z \to  \R^{m}\times\R^n\times \mS^n$ and a  distribution  $\rho\in\calP(\mc Z)$, the divergence  $\div_d (\rho \Phi)$ is the signed measure on $\mc Z$ satisfying the integration by parts formula
\begin{equation*}
 \int_{\mc Z }  \varphi(z)  \,\div_d (\rho \Phi)(\dd z) = - \int_{\mc Z } \langle \Phi(z), \widetilde\nabla_d \varphi(z) \rangle_z \,  \, \rho(\dd z)
 %\qquad \forall \varphi: \mc Z \to \R \text{ differentiable.}
\end{equation*}
for every differentiable function $\varphi: \mc Z \to \R$ with compact support. In case $\rho$ has a density w.r.t. the Riemannian volume form on  $\mc Z$, then this definition coincides with the standard divergence operator induced by  Riemannian metric \eqref{R-metric-Z}. 
%\viet{is this sentence needed -- (see \cite[p.371 and p.383]{ref:lee2003manifold} for the definition and   integration by parts formula).  }
%\viet{For our flow, $\rho$ is gonna be discrete. What is the implication on the divergence operator?}
The optimal transport distance and its induced Riemannian metric on the space $\calP(\mc Z)$ are relegated to Supplementary~\ref{sec:geometry-PZ}

%%%%%%%%%%%%%%%%%%%%%%
%%%%%%%%%%%%%%%%%%%%%%

\section{Gradient Flow for Maximum Mean Discrepancy}\label{sec:MMD-flow}
As $\calP(\mc Z)$ is an infinite dimensional curved space, many machine learning methods based on finite dimensional or linear structure cannot be directly applied 
to this  manifold. To circumvent this problem, we use a positive definite kernel to map $\calP(\mc Z)$ to a RKHS
%reproducing kernel Hilbert space (RKHS) 
and then perform our analysis on it. Let $k$ be a positive definite kernel on $\mc Z$, and let  $\mathcal H$ be the RKHS generated by $k$. The inner product on $\mathcal H$ is denoted by $\langle\cdot,\cdot \rangle_\mathcal H$, and the kernel mean embedding $\rho \in \calP(\mc Z) \longmapsto \m_\rho(\cdot) \in \mathcal H$ is given by 
$\m_\rho(z) \Let \int_{\mc Z} k(z,w) \, \rho(\dd w)$ for $z$ in $\mc Z$.
%{\color{red}Let $\varrho \in \mc P(\mc Z)$ denote the target distribution.} 
The $\mathrm{MMD}$ \cite{JMLR:v13:gretton12a} between  $\rho \in \mc P(\mc Z)$ and the target $\varrho$  is defined as the maximum of the mean difference between the two distributions over all test functions in the unit ball of $\mathcal H$ (see Supplementary A.3). Moreover, it can be expressed by 
$\mathrm{MMD}(\rho, \varrho) = \|\m_{\rho} -\m_{\varrho}\|_{\mathcal H}$. When $k$ is characteristic, 
the  kernel mean embedding $\rho\mapsto \m_\rho$ is injective and therefore, $\mathrm{MMD}(\rho, \varrho)=0$ if and only if  $\rho=\varrho$.
%as distributions.

Consider the loss function $\calF[\rho] \Let  \frac12 \mathrm{MMD}(\rho, \varrho)^2
=\frac12\|\m_{\rho} -\m_{\varrho}\|_{\mathcal H}^2$.
As explained in the introduction, there are three advantages of  $\mathrm{MMD}$  over  Kullback-Leibler divergence: 
its associated gradient flow can employ a sample approximation for the target distribution,  the input distribution $\rho$ does not have to be absolutely continuous w.r.t.~the target distribution $\varrho$, and 
the squared $\mathrm{MMD}$ possesses unbiased sample gradients.
For each $\rho$, the Riemannian gradient $\grad   \calF[\rho]$ is defined as 
the unique element in $\mathrm T_\rho\calP(\mc Z)$ satisfying 
%\begin{equation*}
   $ g_\rho(\grad   \calF[\rho], \zeta) =  \frac{\dd}{\dd t}\Big|_{t=0} \calF[\rho_t]$
%\end{equation*}
for every  differentiable  curve  $t\mapsto \rho_t\in \calP(\mc Z)$ passing through $\rho$ at $t=0$ with tangent vector
$\partial_t \rho_t|_{t=0}=\zeta$.
%$ -\div_d (\rho  \nabla_d \varphi)$ for some differentiable function $\varphi$ on $\mc Z$. 
By using the Riemannian metric tensor~\eqref{rie-metric}, we can compute  explicitly  this  gradient.

 \begin{lemma}[Gradient formula]\label{grad-formula}
% The gradient of the functional $\calF[\rho]$ w.r.t.~the Riemannian metric  $g_\rho$ on $\calP(\mc Z)$ is:
The Riemannian gradient of $\mc F$ satisfies
 $
 \grad  \calF[\rho] = -\div_d \left(\rho \widetilde\nabla_d [\m_\rho- \m_\varrho]\right). 
$
\end{lemma}

The Riemannian gradient $ \grad  \calF$ on $\calP(\mc Z)$ depends not only on the gradient operator $\widetilde\nabla_d$ but also on the divergence
operator.
Using Lemma~\ref{grad-formula}, we  can rewrite the gradient flow equation $
 \partial_t \rho_t = -\grad  \calF[\rho_t]$ explicitly as
\begin{equation}\label{eq:gradient-flow}
\partial_t  \rho_t = \div_d \big(\rho_t \widetilde\nabla_d [\m_{\rho_t} - \m_\varrho ]\big)\quad\mbox{for} \quad t\geq 0.
\end{equation}

%of the manifold $\mc Z$. 
The next result exhibits the rate at which  $\calF$ decreases its value along the  flow.
%Its  discrete version is given in  Proposition~\ref{quantified_decrease}.
%This is a continuous version of the discrete case obtained in  Proposition~\ref{quantified_decrease}.
%\viet{Below proposition is equivalent to Proposition~2 in~\cite{arbelKSG19}.}
\begin{proposition}[Rate of decrease] \label{decrease-rate}
%Assume in addition that the kernel $k$ is well-defined on $\mS^n_+$. 
Along the gradient flow $t\mapsto \rho_t\in \calP(\mc Z)$ given by \eqref{eq:gradient-flow}, we have 
\[
\frac{\dd}{\dd t} \calF[\rho_t] = -  \int_{\mc Z} \big\| \widetilde\nabla_d [ \m_{\rho_t} - \m_\varrho]\big\|_z^2 \, \rho_t(\dd z)\quad \mbox{for}\quad t\geq 0.
\]
\end{proposition}
Proposition~\ref{decrease-rate} implies that  $\frac{\dd}{\dd t} \calF[\rho_t]=0$ if and only if  $\widetilde\nabla_d [ \m_{\rho_t} - \m_\varrho](z) = 0$ for
every $z$ in the support of the distribution $\rho_t$. Thus, the objective function will decrease whenever the gradient $\widetilde\nabla_d [ \m_{\rho_t} - \m_\varrho]$ is not 
identically zero.
%\viet{Need some discussion here. Can we claim anything extra? For example, can we say that $\frac{\dd}{\dd t} \calF[\rho_t]$ is 0 if and only if bla bla? Otherwise, we need some connective sentences to %lead... maybe like: "The continuous flow~\eqref{eq:gradient-flow} is not implementable... so we need to discretize... bla bla bla"}

\subsection{Riemannian Forward Euler Scheme} \label{sec:euler} 
%\viet{we may need to relegate this section to the appendix}
We propose the Riemannian version of the forward Euler scheme to discretize continuous flow~\eqref{eq:gradient-flow}:
\begin{tcolorbox}[colback=white!5!white,colframe=black!75!black,top=-2pt,bottom=0pt]
\begin{equation}\label{k_iteration}
   \begin{aligned}
    &\rho^{\tau+1} = \exp(s_\tau \Phi^\tau)_{\#}\rho^\tau \qquad\\ &\mbox{with}~~\Phi^\tau \Let - \widetilde\nabla_d [\m_{\rho^\tau} -\m_\varrho],
   \end{aligned}
\end{equation}
\end{tcolorbox}
where $s_\tau>0$ is the step size. Here, for a vector field $\Phi = (\Phi_1, \Phi_2,\Phi_3): \mc Z\to \R^{m}\times\R^n\times\mS^n$ and  for $\e\geq 0$,  $\exp(\e\Phi): \mc Z\to \mc Z$ is the Riemannian exponential map induced by \eqref{exponential_map}, i.e., for $z= (x,\mu, \Sigma)\in \mc Z$:

\begin{align*}
\exp_z(\e\Phi(z)) 
\hspace{-0.3em}
=
\hspace{-0.3em}
\begin{pmatrix}
x + \e \Phi_1(z) \\ \mu +\e \Phi_2(z) \\ ( I + \e \rmL_\Sigma[\Phi_3(z)] ) \Sigma  ( I + \e \rmL_\Sigma[\Phi_3(z)] )
\end{pmatrix}.
\end{align*}

Notice in the above equation that the input $z$ affects simultaneously the bases of the exponential map $\exp_{z}$ as well as the direction $\Phi(z)$. This map is the $\e$-perturbation of the identity map along geodesics with directions $\Phi$. When $\rho^\tau = N^{-1} \sum_{i=1}^N \delta_{z^\tau_i}$ is an empirical distribution,
%that is supported on $(z_i^\tau)_{i=1}^N$,  
scheme~\eqref{k_iteration} flows each particle $z^\tau_i$ to the new position $z^{\tau+1}_i = \exp_{z^\tau_i}(s_\tau\Phi(z^\tau_i))$.
The next lemma shows that  $\Phi^\tau$  is the steepest descent direction for $\calF$ w.r.t.~the exponential map among all directions in the space 
$\mathbb{L}^2(\rho^\tau)$, which is the collection of all vector fields $\Phi$ on $\mc Z$ satisfying $\|\Phi\|_{\mathbb{L}^2(\rho^\tau)}^2 \Let 
\int_{\mc Z} \|\Phi(z)\|_z^2 \rho^\tau(\mathrm{d}z)<\infty$. 

%the Forward Euler scheme can be seen as a Riemannian gradient descent 

%We first identity directions along which the functional $\calF[\rho] \Let  \frac12 \mathrm{MMD}(\rho, \varrho)^2$ decreases its %value.  

%$\frac{\dd}{\dd\e}\big|_{\e=0}  \calF[\exp(\e \Phi)_{\#}\rho]$, where 

\begin{lemma}[Steepest descent direction]\label{descent}
Fix a distribution $\rho^\tau \in \mc P(\mc Z)$. For any vector field $\Phi : \mc Z\to \R^{m}\times\R^n\times\mS^n$, we have 
\begin{equation*}
    \resizebox{\linewidth}{!}{$
\frac{\dd}{\dd\e}\Big|_{\e=0}  \calF[\exp(\e \Phi)_{\#}\rho^\tau] = \int_{\mc Z} \langle \widetilde\nabla_d [\m_{\rho^\tau} -\m_\varrho] (z), \Phi(z)\rangle_z ~ \rho^\tau(\dd z ).$}
 %= \langle \nabla_d [\m_\rho -\m_\varrho],\Phi\rangle_{L^2_\rho}.
\end{equation*}
If $\hat\Phi^\tau$ is the unit vector field (w.r.t.~the $\|\cdot\|_{\mathbb{L}^2(\rho^\tau)}$ norm) in the direction of $\Phi^\tau$ given in \eqref{k_iteration}, then  
\[
\frac{\dd}{\dd\e}\big|_{\e=0}  \calF[\exp(\e \hat\Phi^\tau )_{\#}\rho^\tau]
= - \| \widetilde\nabla_d [\m_{\rho^\tau} -\m_\varrho]\|_{\mathbb{L}^2(\rho^\tau)}
\]
and this is the fastest decay rate among all unit directions $\Phi$ in $\mathbb{L}^2(\rho^\tau)$.
%$\Phi$ is a steepest descent direction if $\, \, \Phi = - \nabla_d [\m_\rho -\m_\varrho]$.
     \end{lemma}

 It follows from Lemma~\ref{descent} that the discrete scheme \eqref{k_iteration}  satisfies the Riemannian gradient descent property: if  $\widetilde\nabla_d [\m_{\rho^\tau} -\m_\varrho]$ is nonzero and if $s_\tau >0$ is chosen sufficiently small, then   $\calF[\rho^{\tau+1}] < \calF[\rho^\tau]$. 
 In Proposition 14 in the Supplementary, we quantify the amount of decrease of $\mc F$ at each iteration.
 Algorithm~\ref{alg:nonoise} implements the flow~\eqref{k_iteration} iteratively. Each iteration in Algorithm~\ref{alg:nonoise} has complexity $O(N(Nm + n^3))$, where $m$ is the feature’s dimension, $n$ is the reduced dimension ($n \ll m$), $N$ is the number of particles.
 
\begin{algorithm}[!tb]
	\caption{Discretized Gradient Flow Algorithm for Scheme~\eqref{k_iteration}}
	\label{alg:nonoise}
	\begin{algorithmic}[1]
		\STATE {\bfseries Input:} a source distribution $\rho^0 = N^{-1} \sum_{i=1}^N \delta_{z^0_i}$, a target distribution $\varrho = M^{-1} \sum_{j=1}^M \delta_{\bar z_j}$,  a number of iterations $T$,  a sequence of step sizes $s_\tau>0$ with $\tau=0,1,...,T$ and a kernel $k$
		\STATE {\bfseries Initialization:} Compute  $\bar\Psi(z)\hspace{-0.2em}
		=
		\hspace{-0.2em}
		M^{-1}\sum_{j=1}^M \widetilde\nabla_d^1 k(z,\bar z_j)$
		with $\widetilde\nabla_d^1 k(z,\bar z_j)$ is  $\widetilde\nabla_d$ of $z\mapsto k(z,\bar z_j)$\;
        \vspace{0.2em}
        \STATE \textbf{repeat for each $\tau = 0, \ldots, T-1$:}
        \STATE \hspace{5mm} Compute $\Psi^\tau(z)  =  N^{-1} \sum_{i=1}^N  \widetilde\nabla_d^1  k(z,z^{\tau}_i)$\;
            \STATE \hspace{5mm} \textbf{for} {$i = 1, \ldots, N$} \\
            \STATE \hspace{10mm} \textbf{ do }  $z^{\tau+1}_i \leftarrow  \exp_{z^\tau_i}\big(s_\tau (\bar\Psi -\Psi^\tau)(z^\tau_i)\big)$\\ \STATE \hspace{5mm} \textbf{end for}
        %\STATE Set $\tau \leftarrow \tau+ 1$ 
        %\ENDWHILE
		\STATE{\bfseries Output:} $\rho^T = N^{-1}\sum_{i=1}^N \delta_{z^T_i}$
	\end{algorithmic}
\end{algorithm}

\paragraph{Convergence.} We now study the (weak) convergence of the solution $\rho_t$ of the continuous gradient flow~\eqref{eq:gradient-flow}, as well as the discretized counterpart $\rho^\tau$ of 
flow~\eqref{k_iteration}, to the target distribution $\varrho$. When the kernel $k$ is characteristic, this convergence is equivalent to $\lim_{t\to \infty}\mathrm{MMD}(\rho_t, \varrho)=0$.
Because the objective function $\calF$ is not displacement convex \cite[Section~3.1]{arbelKSG19}, the convergent theory for gradient flows in \cite{ref:ambrosio2008gradient} does not apply even in the case of Euclidean spaces. In general, there is a possibility that $\mathrm{MMD}(\rho_t, \varrho)$ does not decrease to zero as $t\rightarrow\infty$. In view of  Proposition~\ref{decrease-rate}, this happens if
the solutions $\rho_t$ are trapped inside the set $
\big\{ \rho: \,  \int_{\mc Z} \big\| \widetilde\nabla_d [ \m_{\rho} - \m_{\varrho}]\big\|_z^2 \, \rho(\dd z) = 0\big\}$.
For each distribution $\rho$ on $\mc Z$, we define in Supplementary~\ref{appendix:MMD} a symmetric linear and positive operator $\mathbb{K}_\rho: \mathcal H \to \mathcal H$ with the property that  $\langle \mathbb{K}_\rho [ \m_{\rho} - \m_{\varrho}],   \m_{\rho} - \m_{\varrho}\rangle_{\mathcal H} = \int_{\mc Z} \big\| \widetilde\nabla_d [ \m_{\rho} - \m_{\varrho}]\big\|_z^2 \, \rho(\dd z)$ 
% (Lemma A.6 in the Supplementary).
We further show in Proposition~\ref{convergence}  that $\rho_t$  globally converges 
in $\mathrm{MMD}$ if the minimum eigenvalue $\lambda_t$ of the operator $\mathbb{K}_{\rho_t}$ satisfies an integrability condition.

%see Lemma~\ref{operator}

%\viet{Do we need $k$ to be characteristic in the below proposition?}

%\ref{def:L-kernel}

%However, this condition  remains hard to satisfy.
%We show in Lemma~\ref{linearized-cond} that the satisfaction of the condition for convergence  in terms of the weighted negative Sobolev distance  in  \cite[Proposition~7]{arbelKSG19} implies that   conditions $\int_0^\infty \lambda_t\, \dd t =+\infty$ and $ \sum_{\tau=0}^\infty s_\tau \lambda_\tau = +\infty$ in  Proposition~\ref{convergence} are satisfied.

%\subsection[A Practical Algorithm for MMD Flow]{A Practical Algorithm for $\mathrm{MMD}$ Flow}
\subsection{Noisy Riemannian Forward Euler Scheme}
\label{sec:flow}

%Due to $\langle \mathbb{K}_\rho ( \m_{\rho} - \m_{\varrho}),  \m_{\rho} - \m_{\varrho}\rangle_{\mathcal H} = \int \big\| %\nabla_d [ \m_{\rho} - \m_{\varrho}]\big\|_z^2 \, \rho(\dd z)$ by Lemma~\ref{operator}, 
The analysis in Section~\ref{sec:euler} reveals that the gradient flows suffer from convergence issues if the residual $\m_{\rho_t} - \m_{\varrho}$ belongs to the null space of the operator $\mathbb{K}_{\rho_t}$. To resolve this, 
%ones might perturb the steepest descent direction and/or to regularize the objective function in a similar way as was done for parametric flows on $\R^n$ in~\cite{NIPS2017_892c3b1c,NEURIPS2018_07f75d91,Mroueh21}. In this paper, 
we employ graduated optimization  \cite{arbelKSG19,pmlr-v48-gulcehre16,GulcehreMVB17,pmlr-v48-hazanb16}
used for non-convex optimization in Euclidean spaces. Specifically, we modify algorithm \eqref{k_iteration} by injecting Gaussian noise into the exponential map at each iteration $\tau$ to obtain
\begin{tcolorbox}[colback=white!5!white,colframe=black!75!black, top=-6pt,bottom=-0.4pt]
\begin{align}
\hspace{-0.5em}&\rho^{\tau+1} = \exp(s_\tau \Phi^\tau)_{\#}\rho^{\tau, \beta_\tau} \quad \label{noisy-gradient-algo}\\
\hspace{-1em}&\text{with}\hspace{0.2em}
f^{\beta_\tau} \hspace{-0.35em}: \hspace{-0.3em}(z, u) \hspace{-0.05em}\mapsto \hspace{-0.05em} \exp_z(\beta_\tau u),
\,\,\,\hspace{-0.3em}
\rho^{\tau, \beta_\tau} 
\hspace{-0.3em}
\Let 
\hspace{-0.3em}
{f^{\beta_\tau}}_{\hspace{-0.4em}\#} (\rho^\tau \otimes g). \notag
\end{align}
\end{tcolorbox} 
Here $g$ is a Gaussian measure with distribution $\mc N_{\R^m} (0, 1) \otimes \mc N_{\R^n}(0, 1) \otimes \mc N_{\mS^n}(0, 1)$ on the tangent space and
$\mc N_{\mS^n}(0, 1)$ denotes an $n$-by-$n$ symmetric matrix whose upper triangular elements are i.i.d.~standard Gaussian random variables. 
When $\rho^\tau = N^{-1} \sum_{i=1}^N \delta_{z^\tau_i}$,  
scheme~\eqref{noisy-gradient-algo} flows each particle $z^\tau_i$ first to $z^{\tau,\beta_\tau}_i \Let \exp_{z^\tau_i}(\beta_\tau U)$ with noise $U\sim g$
and then to  $z^{\tau+1}_i = \exp_{z^{\tau,\beta_\tau}_i}(s_\tau\Phi(z^{\tau,\beta_\tau}_i))$.
Our next result 
extends Proposition~8 
in \cite{arbelKSG19} for the standard quadratic cost  on the Euclidean space  to  the  nonstandard cost function $d^2$  on the  \textit{curved} Riemannian manifold $\mc Z_{++}$. It demonstrates that  scheme \eqref{noisy-gradient-algo} achieves the global minimum of $\calF$ provided that $k$ is  a Lipschitz-gradient kernel and both the noise level $\beta_\tau$ and the step size $s_\tau$ are well controlled. The proof of Proposition~\ref{conv-noisy-gradient} is given in Supplementary~\ref{appendix:MMD} and relies on arguments that are different from  that of \cite{arbelKSG19}. 
%This result is the Riemannian extension 
%of~\cite[Proposition~8]{arbelKSG19}.
%to the  curved space $\mc Z_{++}$.
\begin{proposition}[Objective value decay for noisy scheme] \label{conv-noisy-gradient}
Suppose  that $k$ is  a Lipschitz-gradient kernel\footnote{See Definition A.3 for the technical definition} with constant $L$, and  the noise level $\beta_\tau$ satisfies
\begin{equation}\label{noise-level-cond}
\lambda \beta_\tau^2 \calF[\rho^\tau]\leq \int_{\mc Z} \|\Phi^\tau(z)\|_z^2 \, \rho^{\tau, \beta_\tau}(\dd z)
\end{equation}
for some constant $\lambda >0$. Then for $\rho^{\tau+1}$ obtained from scheme \eqref{noisy-gradient-algo},
%with $s_\tau\in (0, \e_0]$,
we have
\[
 \calF[\rho^{\tau+1}]  \leq  \calF[\rho^0]\,\exp \Big(-\lambda \sum \nolimits_{i=0}^\tau  [s_i\big(1- 2L s_i\big)\beta_i^2]\Big).
 \]
%with $L$ being the Lipschitz constant  in condition \eqref{Lipschitz-cond}.
\end{proposition}

In particular, $\calF[\rho^{\tau}]$ tends to zero if the sequence $\sum_{i=0}^\tau s_i\big(1- 2L s_i\big)\beta_i^2$ goes to positive infinity. For an adaptive step size $s_\tau\leq 1/4L$,  this condition is met if, for example, $\beta_\tau$ is chosen of the form $(\tau s_\tau)^{-\frac12}$ while still satisfying~\eqref{noise-level-cond}. The noise perturbs the direction of descent, whereas the step size determines how far to move along this perturbed direction. The noise level needs to be adjusted so that the gradient is not too blurred, but it does not necessarily  decrease at each iteration. When the incumbent distribution $\rho^\tau$ is close to a local optimum, it is helpful to increase the noise level  to escape the local optimum. We demonstrate in Lemma~\ref{c1-imply-lipschitz} in the Supplementary that any positive definite kernel $k$ with bounded Hessian w.r.t.~distance $d$ is a Lipschitz-gradient kernel. Algorithm~\ref{alg:noisy} in the Supplementary describes \eqref{noisy-gradient-algo} in details.

\section{Numerical Experiments}
\label{sec:numerical}
We evaluate the proposed gradient flow on real-world datasets and then illustrate its applications in transfer learning. We augment samples for the target dataset, where only a few samples in the dataset are available. We consider three datasets: the MNIST (M)~\cite{lecun-mnisthandwrittendigit-2010}, Fashion-MNIST (F)~\cite{xiao2017fashion}, Kuzushiji-MNIST (K)~\cite{clanuwat2018deep}. To satisfy the Gaussianity assumption of the conditional distributions, we cluster all the images from each class of the datasets and keep the largest cluster for each class. To demonstrate the scalability of our algorithm to higher-dimensional images, we run experiments on Tiny ImageNet (TIN)~\cite{russakovsky2015imagenet} and upscaled SVHN~\cite{netzer2011reading}  datasets, where images are of $3\times64\times64$ size.

% Another more expensive way that does not decrease the dataset size is to label the datasets with more detailed labels.

Our mapping $\phi$ is from $\mathbb{R}^m$ to $\mathbb{R}^2$ in the lifting procedure. To compute the $\mathrm{MMD}$ distance using kernel embeddings, we use a tensor kernel $k$ on $\mc Z$ composed from three standard Gaussian kernels corresponding for each component of the feature space $\R^m$, the mean space $\R^2$ and the covariance matrix space $\mS^2_{++}$. As a consequence,  $k$ is a characteristic kernel by~\cite[Theorem 4]{JMLR:v18:17-492}.

\begin{figure*}[!ht]
    \centering
    \includegraphics[width=0.8\textwidth]{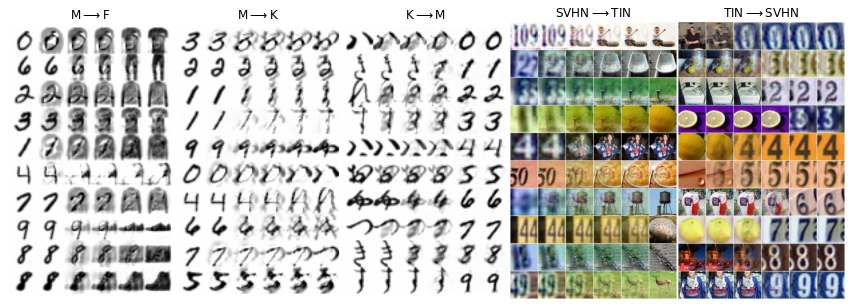}
    \caption{Sample path visualizations for five pairs of source-target domain. The original image and additional results are in the supplementary.}
    \label{fig:NIST_flow}
\end{figure*}
\paragraph{Experiment: Gradient Flow between Datasets.} We visualize the path travelled by each sample from the source domain to the target domain, as depicted in Fig.~\ref{fig:NIST_flow}. We draw randomly $N = 200$ images equally for 10 classes of the source domain, and $M = 50$ images equally for 10 classes of the target domain ($M=10$ for the TIN and SVHN datasets). In each subfigure, each column represents a snapshot of a certain time-step and the samples flow from the source (left) to the target (right) as the number of steps increases. The first column in Fig.~\ref{fig:NIST_flow} are the images from the source domain, where the gradient flows start. Empirically, the algorithm converges after step 140 for *NIST datasets and step 6000 for TIN and SVHN. The experiments are run on a C5.4xlarge AWS instance (a CPU instance) and all finish within one hour.

\begin{figure*}[!ht]
% \begin{wrapfigure}{r}{0.7\textwidth}
\centering
\begin{subfigure}
\centering
\includegraphics[width=0.45\textwidth]{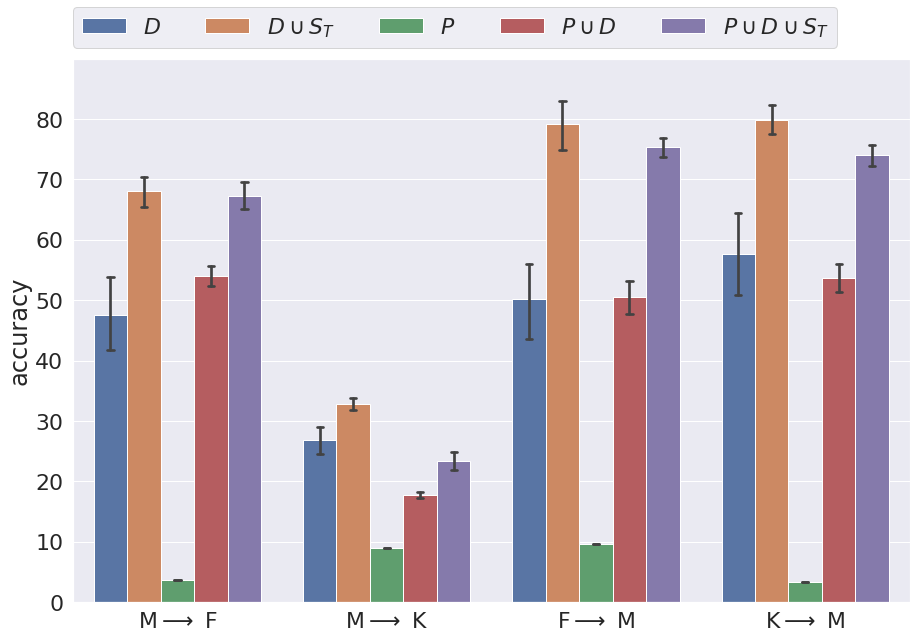}
\end{subfigure}
\begin{subfigure}
\centering
\includegraphics[width=0.45\textwidth]{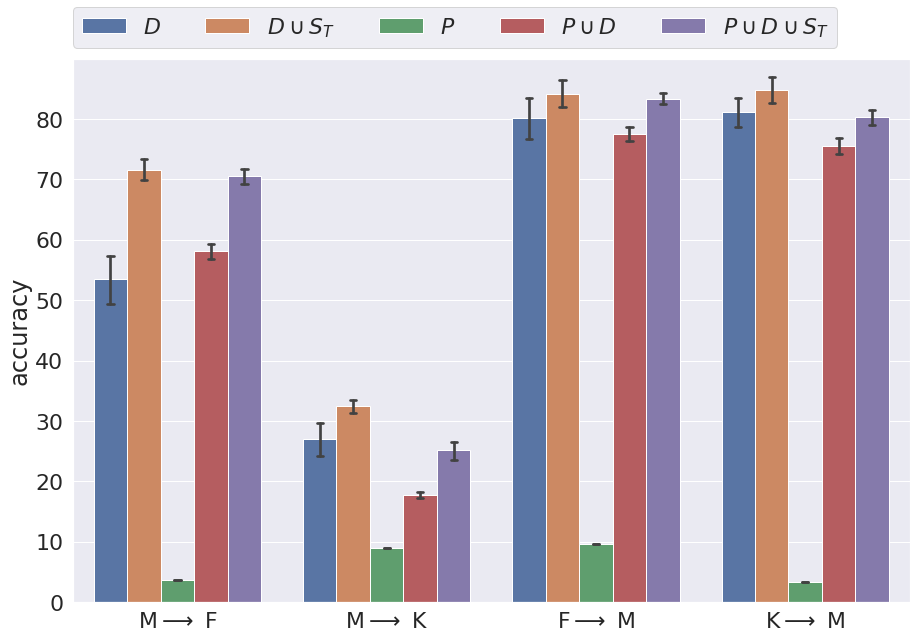}
\end{subfigure}
\caption{Average target domain accuracy on the test split for transfer learning with one-shot (left) and five-shot (right). Results are taken over 10 independent replications, and the range of accuracy is displayed by the error bars.}
\label{fig:few_shot}
\end{figure*}

\subsection{Application in Transfer Learning} Our gradient flow can alleviate the problem of insufficient labeled data by synthesizing new samples to augment the target dataset. In this section, we demonstrate that the generated target domain samples can improve the accuracy in one-shot and five-shot transfer learning tasks. 

First, we fix a source domain and pretrain a classifier $P$ on this domain. We draw randomly $N$ samples from the source domain to form the source dataset $(x_i, y_i)_{i=1}^{N}$. Next, we pick a target domain and draw randomly a few samples from this target domain: for example, in 1-shot learning, only 1 image per class from the target domain is selected to form the target dataset $D = (\bar x_j, \bar y_j)_{j=1}^M$. We then perform a noisy gradient flow scheme~\eqref{noisy-gradient-algo} from the source dataset to the target dataset to get N new samples $S_T = (x_i^T, y_i^T)_{i=1}^{N}$. With the target dataset $D$ and new samples $S_T$, we can retrain the classifier $P$. Similarly, we can also train new classifiers from scratch using datasets $D$ and $D\cup S_T$. Finally, we test the classifiers on the test set of the target domain. 
% We include more details on implementation in Appendix~\ref{sec:implementation_details}.

Fig.~\ref{fig:few_shot} presents the accuracy of five transfer learning strategies on four pairs of source and target domain. For the labels above the plot, labels without $P$ mean training a new classifier from scratch, whereas labels with $P$ mean transferring the pretrained classifier. $D$ and $S_T$ represent the samples in the target domain and our flowed samples. We observe a common trend that the addition of the flowed samples $S_T$ always improves the accuracy of the classifiers, as we compare $D\cup S_T$ with $D$ and compare $P\cup D\cup S_T$ with $P\cup D$. Moreover, the data augmentation with $S_T$ leads to a higher increase of accuracy for the 1-shot learning, where the data scarcity problem is more severe. The transfer learning results for SVHN and TIN datasets are provided in the Supplementary~\ref{sec:SVHN_TIN_supp}. Although Few-shot learning is more challenging due to the high complexity of the datasets, the addition of $S_T$ still improves the accuracy. We also compare with~\cite{AlvarezGF20}\footnote{The only gradient flow work that has experiments on *NIST datasets, but it does not run experiments on TIN and SVHN.}, mixup method~\cite{zhang2017mixup}, and image augmentation methods, see results in Supplementary~\ref{sec:comparison}--\ref{sec:comparison_transformation}.

\paragraph{Conclusions.} This paper focuses on a gradient flow approach to generate new labeled data samples in the target domain. To overcome the discrete nature of the labels, we represent datasets as distributions on the feature-Gaussian space, and the flow is formulated to minimize the $\mathrm{MMD}$ loss function under an optimal transport metric. Contrary to existing gradient flows on linear structure, our flows are developed on the \textit{curved} Riemannian manifold of Gaussian distributions. We provide explicit formula for the Riemannian gradient of the $\mathrm{MMD}$ loss, and analyze in details the flow equations and the convergence properties of both continuous and discretized forms. The numerical experiments demonstrate that our method can efficiently generate high-fidelity labeled training data for real-world datasets, and improve the classification accuracy in few-shot learning. The main limitation exists in the assumption that the data of one label forms an elliptical distribution.
%One distinctive advantage of synthesizing new data with gradient flows is that this framework is model-independent. Indeed, it depends %only on the intrinsic geometry of the %data, and is applicable to any %classification dataset.
%regardless of the size, dimensionality, or the number of classes.

% To provide global convergence, we design our algorithms to compute the Riemannian gradient. However, this increases the complexity of the full algorithms. To run our method on larger-size images and larger-scale datasets efficiently, we will develop fast algorithms to approximately minimize MMD and parallelize Algorithm~\ref{alg:nonoise} and~\ref{alg:noisy} in our follow-up work.

% {\color{blue} TEAM: It seems that reviewers have some concerns about the 1-shot (e.g., how to deal with the covariance), and hyperparameters (e.g., ablation study? otherwise, we may need to describe the experiments clearly enough.}

%We introduce a  Riemannian structure on the space of distributions on the feature-Gaussian manifold, which is induced by an optimal transport metric. Unravelling this geometry  enables us to obtain the gradient flow for the $\mathrm{MMD}$ loss function and  propose a novel discrete
%version of the forward Euler scheme to discretize the  flow. The injection of noise into this scheme yields good convergent results in practice, but other types of regularization are still subject to open questions. One may seek to directly regularize the loss function by imposing some gradient penalty.

\section*{Ethical Statement}
Our work has positive societal impacts, because it can help reduce repetitive data collection and labeling work. It does not have possible negative societal impacts in the current stage.

%%%%%%%%%%%%%%%%%%%%%%
%%%%%%%%%%%%%%%%%%%%%%
\bibliographystyle{named}
\bibliography{bibliography}

%%%%%%%%%%%%%%%%%%%%%%
%%%%%%%%%%%%%%%%%%%%%%
\newpage
\appendix
\onecolumn
The Supplementary Material is organized into two parts. In Section~\ref{sec:proofs}, we provide the proofs and further discussions of the results in the main paper. Section~\ref{appendix:implementation} includes implementation details as well as additional numerical results. All the models and data used to create the results in the paper are in the supplementary file.

\section{Proofs}\label{sec:proofs}

\subsection[Optimal Transport and  Riemannian Structure]{Optimal Transport and  Riemannian Structure on $\calP(\mc Z)$} \label{sec:geometry-PZ}

To define a gradient flow for probability distributions on $\mc Z$, it is essential to have a concept of gradients for functionals defined on $\mc P(\mc Z)$. This requires a meaningful Riemannian structure on $\calP(\mc Z)$, and here, we adopt a Riemannian structure generated by the optimal transport on $\calP(\mc Z)$ with ground cost $d^2$. The optimal transport metric $\mathbb{W}(\rho_0, \rho_1)$ between any two  distributions $\rho_0$, $\rho_1\in \calP(\mc Z)$ is defined by formula \eqref{OT-distance} of Appendix~\ref{appendix:geometry}. 
 As $(\mc Z, d)$ is a Riemannian manifold by Proposition~\ref{prop:Riemmanian}, it follows from the celebrated Benamou-Brenier formula \cite{BenamouB00} 
that~$\mathbb{W}$ can be expressed in terms of  a dynamic 
formulation. Precisely, 
\begin{align}\label{DF}
\hspace{-0.8em} \mathbb{W}(\rho_0, \rho_1)^2 \hspace{-0.2em} = \hspace{-1.1em} \inf_{(\rho,\phi)\in \mathcal A (\rho_0,\, \rho_1)} \hspace{-0.2em}  \int_0^1 \hspace{-0.5em} \int_{\mc Z } \hspace{-0.3em} \| \widetilde\nabla_d \phi_t(z)\|_z^2  \rho_t(\dd z)  \dd t,
\end{align}
where $\mathcal A (\rho_0,\, \rho_1)$ is the collection of all pairs $(\rho,\phi)$ of curve $\rho: [0,1]\to \calP( \mc Z)$ with endpoints $\rho_0$ and $\rho_1$, and function $\phi: [0,1]\times \mc Z \to  \R$ that satisfies the continuity equation 
\begin{equation}\label{continuity-eq}
\partial_t \rho + \div_d(\rho_t \widetilde\nabla_d \phi_t) = 0 
\end{equation}
in the sense of distributions
in  $(0,1)\times \mc Z$. For brevity, hereafter $\rho_t$
and $\phi_t$ denote functions in the $z$ variable defined by $\rho_t(z) = \rho(t,z)$
and $\phi_t(z) = \phi(t,z)$.

\textbf{Riemannian metric on $\calP(\mc Z)$.} The 
%dynamic   
formulation~\eqref{DF} gives rise to the following Riemannian  structure on $\calP(\mc Z)$ induced by distance $\mathbb{W}$. First, the continuity equation enables us to identify a tangent vector $\partial_t \rho$ with the divergence  $-\div_d(\rho_t \widetilde\nabla_d \phi_t)$. Thus  the tangent space of $\calP(\mc Z)$ at a distribution  $\rho$ can be defined as
$\mathrm T_\rho\calP(\mc Z) \Let \big\{-\div_d (\rho \widetilde \nabla_d \varphi):\,  \varphi \, \mbox{is a differentiable function with compact support on}\, \mc Z\big\}$.
Second, we let $g_\rho:\mathrm T_\rho\calP(\mc Z) \times \mathrm T_\rho\calP(\mc Z) \longrightarrow \R $ be the Riemannian metric tensor given by 
\begin{equation}\label{rie-metric}
 g_\rho(\zeta_1, \zeta_2) 
 %= \langle \nabla_d \phi_1,  \nabla_d \phi_2\rangle_{L^2_\nu} 
 \Let \int_{\mc Z} \langle \widetilde\nabla_d \varphi_1(z),  \widetilde \nabla_d \varphi_2(z)\rangle_z \,\, \rho(\dd z)
\end{equation}
for  $\zeta_1=-\div_d(\rho \widetilde \nabla_d \varphi_1 )$ and $\zeta_2=-\div_d(\rho  \widetilde\nabla_d \varphi_2)$. With this definition and due to \eqref{continuity-eq}, formula \eqref{DF} can be rewritten using the metric tensor as  
\[
\mathbb{W}(\rho_0, \rho_1)^2 = \inf_{(\rho,\phi)\in \mathcal A (\rho_0,\, \rho_1)}  \int_0^1 g_{\rho_t}(\partial_t \rho, \partial_t \rho) \, \dd t.
\]
The 
%Riemannian 
metric tensor \eqref{rie-metric} allows us to define a notion of Riemannian gradients for functionals  on $\calP(\mc Z)$. In the next section we shall compute this gradient explicitly 
for the squared $\mathrm{MMD}$ gradient flow.

\subsection{Proofs and Results related to Section~\ref{sec:riemannian}}
\label{appendix:geometry}

\paragraph{For Proposition~\ref{prop:Riemmanian}.}

Recall that the Bures distance defined on $\mS_{++}^n$ is
\begin{equation}\label{metric}
\mathbb{B}(\Sigma_1,\Sigma_2) \Let \big[  \tr(\Sigma_1 + \Sigma_2 - 2  [\Sigma_1^{\frac12} \Sigma_2 \Sigma_1^{\frac12}]^{\frac12})\big]^{\frac12},
%\quad \mbox{for}\quad \Sigma_1, \Sigma_2\in\mS^n_+.
\end{equation}
and $\widetilde\nabla_d  \varphi(z)$ is the unique element in the tangent space $\R^m \times \R^n \times \mS^n$ satisfying 
\begin{equation}\label{grad-dataset}
 \big\langle \widetilde\nabla_d \varphi(z), (w,v,V) \big\rangle_{z} = D \varphi_z(w,v, V) \quad \mbox{for all } (w,v, V) \in \R^{m}\times\R^n \times \mS^n.
\end{equation}
The proof of  Proposition~\ref{prop:Riemmanian} relies on the following result 
from \cite[Proposition A]{ref:takatsu2011wasserstein} (see also~\cite[Theorem~5]{ref:bhatia2019on} and~\cite[Proposition~6]{ref:malago2018wasserstein}).
\begin{proposition}\label{takatsu2011}
The space $\mS^n_{++}$ is a Riemannian manifold with the following structures: at each point $\Sigma\in \mS^n_{++}$, the  tangent space is  
$
\mathrm T_\Sigma \mS^n_{++} = \mS^n$  and  the Riemannian metric is given by 
\[
\langle X, Y\rangle_{\Sigma}   \Let \tr \Big(\rmL_\Sigma[X] \, \Sigma \, \rmL_\Sigma[Y]\Big)
= \frac12\langle \rmL_\Sigma[X], Y \rangle
\quad \mbox{for}\quad X, Y\in \mS^n.
\]
Moreover, the distance function corresponding to this  Riemannian metric coincides with the Bures distance $\mathbb{B}$ given by 
\eqref{metric}.
\end{proposition}

We are now ready to prove Proposition~\ref{prop:Riemmanian}.
\begin{proof}[Proof of Proposition~\ref{prop:Riemmanian}]
As a consequence of  Proposition~\ref{takatsu2011}, $\mc Z$ is  the product Riemannian manifold 
 with tangent space  $\mathrm T_{(x,\mu,\Sigma)}\mc Z = \mathrm T_x \R^{ m} \times \mathrm T_{\mu} \R^{ n}\times \mathrm T_\Sigma \mS^n_{++}$ and 
 with the standard product Riemmanian metric \eqref{R-metric-Z}. The result then follows. We note that if we let 
 $\D((x_1,\mu_1,\Sigma_1), (x_2,\mu_2,\Sigma_2))$ denote 
 the  distance function 
 corresponding to the Riemannian metric $\langle \cdot,\cdot\rangle_{z}$ on $\mc Z$, then its square 
 $\D((x_1,\mu_1,\Sigma_1), (x_2,\mu_2,\Sigma_2))^2$ is the sum of the square of the distance function 
 w.r.t.~standard metric $\langle \cdot,\cdot\rangle$ on $\R^m$, the square of the distance function 
 w.r.t.~standard metric $\langle \cdot,\cdot\rangle$ on $\R^n$, and the square of the distance function 
 w.r.t.~metric $\langle \cdot, \cdot\rangle_{\Sigma} $ on $\mS^n_{++}$. 
 As a result and by Proposition~\ref{takatsu2011}, we have $\D((x_1,\mu_1,\Sigma_1), (x_2,\mu_2,\Sigma_2))^2 = \|x_1- x_2\|_2^2 +
 \|\mu_1 -\mu_2\|_2^2 + \mathbb{B}(\Sigma_1,\Sigma_2)^2$. So, $\D$ is the same~as~$d$.
\end{proof}

\paragraph{For Lemma~\ref{gradient-formula}}
\begin{proof}[Proof of Lemma~\ref{gradient-formula}]
Let us express $\widetilde\nabla_d \varphi(z) = (\Phi_1(z), \Phi_2(z), \Phi_3(z))$ with $\Phi_1(z)\in \R^{m}, \, \Phi_2(z)\in\R^n$ and $\Phi_3(z)\in \mS^n$. Then by using 
the definition of Riemannian metric $\langle \cdot,\cdot \rangle_z$ in \eqref{R-metric-Z}, we can rewrite equation \eqref{grad-dataset} as
\[
\langle\Phi_1(z), v\rangle +\langle\Phi_2(z), w\rangle  + \langle \frac12\rmL_\Sigma[\Phi_3(z)], V\rangle  = \langle \nabla\varphi(z), (v,w,V) \rangle.
\]
This is equivalent to 
\begin{equation}\label{identify_gradient}
\langle\Phi_1(z), v\rangle +\langle\Phi_2(z), w\rangle   + \langle \frac12\rmL_\Sigma[\Phi_3(z)], V\rangle  = \langle \nabla_x\varphi(z), v \rangle + \langle  \nabla_\mu\varphi(z), w \rangle 
+  \langle \nabla_\Sigma\varphi(z), V \rangle, 
\end{equation}
where $\nabla\varphi(z) = \Big(\nabla_x\varphi(z), \nabla_\mu\varphi(z), \nabla_\Sigma\varphi(z)\Big)$ denotes the standard Euclidean gradient. 
Equation \eqref{identify_gradient} is obviously satisfied if $\Phi_1(z) = \nabla_x\varphi(z)$, $\Phi_2(z) =  \nabla_\mu\varphi(z)$,  and $\rmL_\Sigma[\Phi_3(z)] = 2\nabla_\Sigma\varphi(z)$. By the definition of operator $\rmL_\Sigma$ right after~\eqref{Lyapunov}, the third identity is the same as
$\Phi_3(z) = 2[\nabla_\Sigma\varphi(z)] \Sigma + 2\Sigma [\nabla_\Sigma\varphi(z)]$.
Due to  uniqueness of the gradient, we therefore infer that $\widetilde\nabla_d \varphi(z)$ is given by  the  formula:
\begin{equation*}
\widetilde\nabla_d \varphi(z) = \Big(\nabla_x\varphi(z), \nabla_\mu\varphi(z),  2[\nabla_\Sigma\varphi(z)] \Sigma + 2\Sigma [\nabla_\Sigma\varphi(z)]\Big).
\end{equation*}
This completes the proof.
\end{proof}

In this paper, the optimal transport metric  between any two  distributions $\rho_0$, $\rho_1\in \calP(\mc Z)$ is defined by 
\begin{equation}\label{OT-distance}
\mathbb{W}(\rho_0,\rho_1)^2 \Let \inf_{\pi \in \Pi(\rho_0,\rho_1)} \iint_{\mc Z \times \mc Z} d(z_0, z_1)^2 \,\, \pi(\dd z_0, \dd z_1),
\end{equation}
where  $\Pi(\rho_0,\rho_1)$ is  the set of all   probability distributions on $\mc Z \times \mc Z$ whose marginals are $\rho_0$ and $\rho_1$, respectively.

\subsection{Proofs and Results related to Section~\ref{sec:MMD-flow}}\label{appendix:MMD}
The maximum mean discrepancy (MMD) between a distribution $\rho \in \mc P(\mc Z)$ and the target distribution $\varrho$ is defined as
\[
\mathrm{MMD}(\rho, \varrho) \Let \sup_{f\in \mathcal H: \|f\|_{\mathcal H}\leq 1 }{\Big\{\int_{\mc Z} f(z) \, \rho(\dd z) -\int_{\mc Z} f(z) \, \varrho(\dd z)\Big\}}.   
\]
%On the other hand, the kernel mean embedding $\rho \in \calP(\mc Z_{++}) \longmapsto \m_\rho(\cdot) \in \mathcal H$ is defined as 
%\[
%\m_\rho(z) \Let \int_{\mc Z_{++}} k(z,w) \, \rho(\dd w).
%\]
It is well-known that the $\mathrm{MMD}$ admits the following   closed-form formula~\cite[Lemmas~4 and~6]{JMLR:v13:gretton12a}.

\begin{lemma}\label{closed-form} We have $\mathrm{MMD}(\rho, \varrho) = \|\m_\rho -\m_\varrho\|_{\mathcal H}$. As a consequence, 
\begin{align*}\label{squared-MMD}
\mathrm{MMD}(\rho, \varrho)^2
= \int_{\mc Z} \int_{\mc Z} k(z,w) \rho(\dd z) \rho(\dd w) 
-2 \int_{\mc Z}  \m_\varrho(z) \rho(\dd z)   + \| \m_{\varrho}\|_{\mathcal H}^2.
\end{align*}
\end{lemma}
\begin{proof}[Proof of Lemma~\ref{closed-form}]
For any $f\in \mathcal H$, we have $f(z) = \langle f, k(\cdot, z)\rangle_{ \mathcal H}$. Therefore,
\begin{equation}\label{a_representation}
\int_{\mc Z} f(z) \, \rho(\dd z)  = \Big\langle f, \int_{\mc Z} k(\cdot, z)\, \rho(\dd z)  \Big\rangle_{ \mathcal H}  = \langle f, \m_\rho \rangle_{ \mathcal H}\quad \mbox{for all}\quad f\in  \mathcal H.
\end{equation}
It follows that  $\mathrm{MMD}(\rho, \varrho) =\sup_{f\in  \mathcal H: \|f\|_{ \mathcal H}\leq 1 }  \langle f, \m_\rho -\m_\varrho\rangle_{ \mathcal H} 
= \|\m_{\rho} -\m_{\varrho}\|_{ \mathcal H}$. 
Using this closed-form formula and  identity \eqref{a_representation}, we also obtain
\begin{align*}
\mathrm{MMD}(\rho, \varrho)^2
&=\|\m_{\rho} -\m_{\varrho}\|_{ \mathcal H}^2 = \langle \m_\rho ,  \m_\rho \rangle_{ \mathcal H}  - 2 \langle \m_\rho , \m_\varrho\rangle_{ \mathcal H} + \langle \m_\varrho,  \m_\varrho\rangle_{ \mathcal H}\\
&=  \int  \m_\rho(z) \rho( \dd z)
-2 \int \m_\varrho(z) \rho(\dd z)   + \| \m_\varrho\|_{ \mathcal H}^2\\
&=  \iint  k(z,w) \rho( \dd z) \rho(\dd w) 
-2 \int \m_\varrho(z) \rho(\dd z)   + \| \m_\varrho\|_{ \mathcal H}^2.
\end{align*}
This completes the proof.
\end{proof}

%%%%%%%%%%%%%%%%%%%%%%
\paragraph{For Lemma~\ref{grad-formula}}

\begin{proof}[Proof of Lemma~\ref{grad-formula}]
We recall that  $\grad   \calF[\rho]$ is defined as 
the unique element in $\mathrm T_\rho\calP(\mc Z)$ satisfying 
\begin{equation*}
    g_\rho\Big(\grad   \calF[\rho], \partial_t \rho_t|_{t=0}\Big) =  \frac{\dd}{\dd t}\Big|_{t=0} \calF[\rho_t]
\end{equation*}
for every  differentiable  curve  $t\mapsto \rho_t\in \calP(\mc Z)$ passing through $\rho$ at $t=0$.
%with tangent vector
%$\partial_t \rho_t|_{t=0}=\zeta$. 
Let  $t\mapsto \rho_t\in \calP(\mc Z)$ be such a curve. Then since $ \partial_t \rho_t|_{t=0} \in \mathrm T_\rho\calP(\mc Z)$, we can write $\partial_t \rho_t|_{t=0} =
 -\div_d (\rho  \widetilde\nabla_d \varphi)$
for some differentiable function $\varphi$ on $\mc Z$. 
Then by using Lemma~\ref{closed-form} and $k(z,w) = k(w,z)$ we have 
\begin{align*}
%g_\rho(\grad   \calF[\rho], \zeta)
 \frac{\dd}{\dd t}\Big|_{t=0} \calF[\rho_t]
 &= \frac12\frac{\dd}{\dd t}\Big|_{t=0} \left[\iint k(z,w) \rho_t(\dd z) \rho_t(\dd w) 
-2 \int \m_\varrho(z) \rho_t(\dd z) \right]\nonumber\\
&= \frac12\iint k(z,w) \partial_t \rho_t|_{t=0}(\dd z)  \rho(\dd w) 
+\frac12\iint k(z,w)  \partial_t \rho_t|_{t=0}(\dd w)  \rho(\dd z)\nonumber\\
&\quad -  \int \m_\varrho(z) \, \partial_t \rho_t|_{t=0}(\dd z)\nonumber\\
&= -   \int_{\mc Z}  \int_{\mc Z}  k(z,w) \div_d (\rho  \widetilde\nabla_d \varphi)(\dd z)\, \rho(\dd w)
+  \int_{\mc Z} \m_\varrho(z) \, \div_d (\rho  \widetilde\nabla_d \varphi)(\dd z).
\end{align*}
Let $\widetilde\nabla_d^1 k(z,w)$ denote the gradient $\widetilde\nabla_d$ of the function $z\mapsto k(z,w)$. 
It then follows from the definition of the divergence operator $\div_d(\rho \widetilde\nabla_d \varphi)$ at the end of Section~3.1 that
\begin{align*}
 \frac{\dd}{\dd t}\Big|_{t=0} \calF[\rho_t]
 &=   \int_{\mc Z}  \int_{\mc Z} \langle \widetilde\nabla_d^1 k(z,w), \widetilde\nabla_d \varphi(z)\rangle_z \,\rho(\dd z) \rho(\dd w)
 - \int_{\mc Z} \langle \widetilde\nabla_d \m_\varrho(z),    \widetilde\nabla_d \varphi(z)\rangle_z \rho(\dd z) \\
 &=   \int\left[   \Big\langle \int \widetilde\nabla_d^1 k(z,w)\,  \rho(\dd w), \widetilde\nabla_d \varphi(z)\Big\rangle_z  \right]\rho(\dd z)
  - \int_{\mc Z} \langle \widetilde\nabla_d \m_\varrho(z),    \widetilde\nabla_d \varphi(z)\rangle_z \rho(\dd z) \\
  &=   \int\left[   \Big\langle \widetilde\nabla_d \int  k(z,w)\,  \rho(\dd w), \widetilde\nabla_d \varphi(z)\Big\rangle_z  \right]\rho(\dd z)
 - \int_{\mc Z} \langle \widetilde\nabla_d \m_\varrho(z),    \widetilde\nabla_d \varphi(z)\rangle_z \rho(\dd z)  \\
 &= \int_{\mc Z} \langle \widetilde\nabla_d[ \m_\rho - \m_\varrho](z),    \widetilde\nabla_d \varphi(z)\rangle_z \rho(\dd z).
\end{align*}
By the definition   of the Riemannian metric tensor $g_\rho$ given in  \eqref{rie-metric} and due to $\partial_t \rho_t|_{t=0} = -\div_d (\rho  \widetilde\nabla_d \varphi)$, we thus obtain  
\begin{align*}
  \frac{\dd}{\dd t}\Big|_{t=0} \calF[\rho_t]
 &=g_\rho\Big(-\div_d  \big(\rho \widetilde\nabla_d [\m_\rho - \m_\varrho]\big), -\div_d(\rho \widetilde\nabla_d \varphi)\Big)\\
 &= g_\rho\Big(-\div_d  \big(\rho\widetilde \nabla_d [\m_\rho - \m_\varrho]\big), \partial_t \rho_t|_{t=0}\Big).
\end{align*}
Therefore, we infer that $\grad   \calF[\rho]= -\div_d  \big(\rho \widetilde\nabla_d [\m_\rho - \m_\varrho]\big)$ as desired.
\end{proof}

%%%%%%%%%%%%%%%%%%%%%
\paragraph{For Proposition~\ref{decrease-rate}}

\begin{proof}[Proof of Proposition~\ref{decrease-rate}]
 The proof is similar to that of Lemma~\ref{grad-formula} and with the same notation for $\widetilde\nabla_d^1 k(z,w)$. Indeed, by the same computation at the beginning of the proof of 
 Lemma~\ref{grad-formula} we have 
 \begin{align*}
 \frac{\dd}{\dd t} \calF[\rho_t]
&= \frac12\iint k(z,w) \partial_t \rho_t(\dd z)  \rho_t(\dd w) 
+\frac12\iint k(z,w)  \partial_t \rho_t(\dd w)  \rho_t(\dd z) -  \int \m_\varrho(z) \, \partial_t \rho_t(\dd z)\nonumber\\
&=\iint k(z,w) \partial_t \rho_t(\dd z)  \rho_t(\dd w) -  \int \m_\varrho(z) \, \partial_t \rho_t(\dd z).
\end{align*}
This together with the gradient flow equation \eqref{eq:gradient-flow} gives
 \begin{align*}
 \frac{\dd}{\dd t} \calF[\rho_t] 
 &=  \int_{\mc Z}  \int_{\mc Z}  k(z,w) \div_d \big(\rho_t \widetilde\nabla_d [\m_{\rho_t} - \m_\varrho ]\big)(\dd z)  \rho_t(\dd w)\\
 &\quad -  \int_{\mc Z} \m_\varrho(z) \, \div_d \big(\rho_t \widetilde\nabla_d [\m_{\rho_t} - \m_\varrho ]\big)(\dd z).
\end{align*}
Using the definition of the divergence operator $\div_d$ at the end of Section~3.1, we further obtain
\begin{align*}
 \frac{\dd}{\dd t} \calF[\rho_t] 
 &=   - \int_{\mc Z}  \int_{\mc Z} \langle \widetilde\nabla_d^1 k(z,w), \widetilde\nabla_d [\m_{\rho_t} - \m_\varrho ](z)\rangle_z \,\rho_t(\dd z) \rho_t(\dd w)\\
&\quad +\int_{\mc Z} \langle \widetilde\nabla_d \m_\varrho(z),    \widetilde\nabla_d [\m_{\rho_t} - \m_\varrho ](z)\rangle_z \rho_t(\dd z) \\
 &=   -\int\left[   \Big\langle \int \widetilde\nabla_d^1 k(z,w)\,  \rho_t(\dd w), \widetilde\nabla_d [\m_{\rho_t} - \m_\varrho ](z)\Big\rangle_z  \right]\rho_t(\dd z)\\
  &\quad + \int_{\mc Z} \langle \widetilde\nabla_d \m_\varrho(z),    \widetilde\nabla_d [\m_{\rho_t} - \m_\varrho ](z)\rangle_z \rho_t(\dd z) \\
  &=  - \int_{\mc Z} \langle \widetilde\nabla_d \m_{\rho_t}(z),    \widetilde\nabla_d [\m_{\rho_t} - \m_\varrho ](z)\rangle_z \rho_t(\dd z)\\
&\quad + \int_{\mc Z} \langle \widetilde\nabla_d \m_\varrho(z),    \widetilde\nabla_d [\m_{\rho_t} - \m_\varrho ](z)\rangle_z \rho_t(\dd z)  \\
 &= - \int_{\mc Z}\| \widetilde\nabla_d[ \m_{\rho_t} - \m_\varrho](z)\|_z^2 \rho_t(\dd z).
\end{align*}
This yields the desired result.
%Applying the integration by part formula, we then obtain the desired result. Notice that we do not have any boundary term since the measures $\rho_t$ are supported on $\mc Z$.
\end{proof} 

%%%%%%%%%%%%%%%%%%%%
\paragraph{For Lemma~\ref{descent}}

\begin{proof}[Proof of Lemma~\ref{descent}]
From the formula for $\exp_z(\e\Phi(z))$ given at the beginning of Section~4.1, 
we  observe  that
\begin{align}\label{deri-exp}
 \frac{\dd}{\dd\e}\Big|_{\e=0}  \exp_z(\e\Phi(z)) 
 &= \Big(\Phi_1(z), \Phi_2(z), \rmL_\Sigma[\Phi_3(z)]\,  \Sigma + \Sigma \, \rmL_\Sigma[\Phi_3(z)]\Big)\nonumber\\
 &= \Big(\Phi_1(z), \Phi_2(z), \Phi_3(z)\Big) = \Phi(z),
\end{align}
where the second  equality is due to the definition of $\rmL_\Sigma[V]$ given at the beginning of Section~3.1.

%Let $\rho\in \calP(\mc Z)$ be arbitrary. Then 
We  obtain from Lemma~\ref{closed-form} that 
 \begin{align*}
&\mathrm{MMD}(\exp(\e \Phi)_{\#}\rho^\tau, \varrho)^2\\
&=  \iint  k(z,w) \exp(\e \Phi)_{\#}\rho^\tau(\dd z) \exp(\e \Phi)_{\#}\rho^\tau(\dd w) 
-2 \int \m_\varrho(z) \exp(\e \Phi)_{\#}\rho^\tau(\dd z )   + \| \m_\varrho\|_{ \mathcal H}^2\\
&=  \iint  k\Big(\exp_z(\e\Phi(z)),\exp_w(\e\Phi(w))\Big) \rho^\tau(\dd z) \rho^\tau(\dd w) 
-2 \int \m_\varrho\Big(\exp_z(\e\Phi(z))\Big) \rho^\tau(\dd z )   + \| \m_\varrho\|_{ \mathcal H}^2.
\end{align*}
Moreover, we have
\begin{align*}
  & \iint  \frac{\dd}{\dd \e}\Big|_{\e=0}   \Big[k(\exp_z(\e\Phi(z)),\exp_w(\e\Phi(w)))\Big]  \rho^\tau(\dd z) \rho^\tau(\dd w) \\
&=  \iint  \left\{ \frac{\dd}{\dd \e}\Big|_{\e=0}   \Big[k(\exp_z(\e\Phi(z)),w )\Big]  
+  \frac{\dd}{\dd \e}\Big|_{\e=0}   \Big[k(z,\exp_w(\e\Phi(w)))\Big]  \right\}  \rho^\tau(\dd z) \rho^\tau(\dd w)\\
&=  \int \frac{\dd}{\dd \e}\Big|_{\e=0} \Big[ \int  k(\exp_z(\e\Phi(z)),w ) \rho^\tau(\dd w)\Big]
\rho^\tau(\dd z) 
+  \int   \frac{\dd}{\dd \e}\Big|_{\e=0}  \Big[\int  k(z,\exp_w(\e\Phi(w)) ) \rho^\tau(\dd z)\Big] \rho^\tau(\dd w)\\
&=  \int \frac{\dd}{\dd \e}\Big|_{\e=0}  \Big[\m_{\rho^\tau}(\exp_z(\e\Phi(z))) \Big]\rho^\tau(\dd z) 
+  \int   \frac{\dd}{\dd \e}\Big|_{\e=0}  \Big[\m_{\rho^\tau}(\exp_w(\e\Phi(w)) ) \Big] \rho^\tau(\dd w)\\
&= 2 \int \frac{\dd}{\dd \e}\Big|_{\e=0}  \Big[\m_{\rho^\tau}(\exp_z(\e\Phi(z))) \Big]\rho^\tau(\dd z). 
\end{align*}
Thus, it  follows that 
\begin{align*}
 &\frac{\dd}{\dd \e}\Big|_{\e=0}   
\mathrm{MMD}(\exp(\e \Phi)_{\#}\rho^\tau, \varrho)^2\\
&= 2 \int \frac{\dd}{\dd \e}\Big|_{\e=0}  \Big[\m_{\rho^\tau}(\exp_z(\e\Phi(z))) \Big]\rho^\tau(\dd z) 
-2 \int \frac{\dd}{\dd \e}\Big|_{\e=0}  \Big[\m_\varrho(\exp_z(\e\Phi(z))) \Big]\rho^\tau(\dd z )  \\
&= 2 \int D [\m_{\rho^\tau} - \m_\varrho]_z\Big( \frac{\dd}{\dd \e}\Big|_{\e=0} \exp_z(\e\Phi(z)) \Big) \rho^\tau(\dd z)
\end{align*}
with  $ D \varphi_z(w,,v,V)$ denoting the standard directional derivative of $\varphi$ at $z$ in the direction $(w,v,V)$.
Using the definition of $\calF$ together with  \eqref{deri-exp} and the definition of gradient $\widetilde\nabla_d$ in \eqref{grad-dataset}, we obtain 
\begin{align}\label{first-conclusion}
\frac{\dd}{\dd \e}\Big|_{\e=0}   
\calF[\exp(\e \Phi)_{\#}\rho^\tau]
&=  \int D[\m_{\rho^\tau} -\m_\varrho]_z(\Phi(z)) \, \rho^\tau(\dd z ) \nonumber \\
&=   \int \langle \widetilde\nabla_d [\m_{\rho^\tau} -\m_\varrho] (z), \Phi(z)\rangle_z \rho^\tau(\dd z ). 
% =2\langle \nabla_d [\m_\rho -\m_\varrho],\Phi\rangle_{L^2_\rho}.
\end{align}
This yields the first conclusion of the lemma.

Now let $\hat\Phi^\tau \Let \frac{\Phi^\tau}{\|\Phi^\tau\|_{\mathbb{L}^2(\rho^\tau)}}$ be the unit vector field  
in the direction of $\Phi^\tau \Let - \widetilde\nabla_d [\m_{\rho^\tau} -\m_\varrho]$. Then by \eqref{first-conclusion}, we have
\[
\frac{\dd}{\dd \e}\Big|_{\e=0} \calF[\exp(\e \hat\Phi^\tau)_{\#}\rho^\tau]
=  -\|\Phi^\tau\|_{\mathbb{L}^2(\rho^\tau)}^{-1}  \int \|\Phi^\tau(z)\|_z^2 \rho^\tau(\dd z )= -\|\Phi^\tau\|_{\mathbb{L}^2(\rho^\tau)}\leq 0.
\]
On the other hand, for any unit direction $\Phi$ in $\mathbb{L}^2(\rho^\tau)$ we obtain from \eqref{first-conclusion} and H\"older inequality that
\begin{align*}
\Big| \frac{\dd}{\dd \e}\Big|_{\e=0}   
&\calF[\exp(\e \Phi)_{\#}\rho^\tau]\Big|
 \leq  \int \|\Phi^\tau(z)\|_z \|\Phi(z)\|_z \rho^\tau(\dd z )\\
&\leq  \Big(\int \|\Phi^\tau(z)\|_z^2 \rho^\tau(\dd z ) \Big)^\frac12
\Big(\int \|\Phi(z)\|_z^2 \rho^\tau(\dd z ) \Big)^\frac12
= \|\Phi^\tau\|_{\mathbb{L}^2(\rho^\tau)}.
\end{align*}
Therefore, we conclude further that 
\[
\frac{\dd}{\dd \e}\Big|_{\e=0} \calF[\exp(\e \hat\Phi^\tau)_{\#}\rho^\tau] \leq \frac{\dd}{\dd \e}\Big|_{\e=0}   
\calF[\exp(\e \Phi)_{\#}\rho^\tau]
\]
for any unit direction $\Phi$ in $\mathbb{L}^2(\rho^\tau)$. These give the last conclusion of the lemma.
\end{proof}

\begin{definition}[Lipschitz-gradient kernel] \label{def:L-kernel}
 Let $L>0$. A differentiable kernel $k$ on $\mc Z$ is called a Lipschitz-gradient kernel with constant $L$  if
 there exists a number $\e_0\in (0,1)$  such that  
\begin{align}\label{Lipschitz-cond}
    \Big| k(\exp_z(\e\Phi(z)), \exp_w(\delta\Phi(w)))
    &- k(z,w) -\big[  \langle \widetilde\nabla_d^1 k(z, w), \e\Phi(z)\rangle_z + \langle \widetilde\nabla_d^2 k(z, w), \delta\Phi(w)\rangle_w \big]\Big|\nonumber\\
    &\leq  L \Big[\| \e \Phi(z)\|_z^2 +\| \delta \Phi(w)\|_w^2 \Big]
\end{align}
for every $\e,\,\delta \in [0,\e_0]$ and every bounded vector field $\Phi: \mc Z\to \R^{m}\times \R^n\times \mS^n$. Hereafter, $\widetilde\nabla_d^1 k(z, w)$ and $\widetilde\nabla_d^2 k(z, w)$ 
denote respectively the gradient $\widetilde\nabla_d$ of the function $z\mapsto k(z,w)$ and the function $w\mapsto k(z,w)$.
\end{definition}
\begin{remark}
The right hand side of  condition \eqref{Lipschitz-cond} can be expressed in terms of the $d$ distance  as
\[
d\big(\exp_z(\e\Phi(z)), z\big)^2 +d\big(\exp_w(\delta\Phi(w)), w\big)^2.
\]
Thus condition \eqref{Lipschitz-cond} can be interpreted as  the gradient $\widetilde\nabla_d k$
%$\nabla_d k(z,w) \Let (\nabla_d^1 k(z,w), \nabla_d^2 k(z,w) )$ 
is Lipschitz  w.r.t.~the 
 distance $d$.
\end{remark}

Condition~\eqref{Lipschitz-cond} is motivated by the following observation in the Euclidean space. Assume that $G: \R^d\times \R^d \to \R$ is a differentiable function 
such that its
Euclidean gradient $\nabla G(z,w)
\Let (\nabla^1 G(z,w), \nabla^2 G(z,w) )$ satisfies the standard Lipschitz condition 
\[
\| \nabla G(z_1,w_1) - \nabla G(z_2,w_2)\|_2 \leq L \| (z_1,w_1) - (z_2,w_2)\|_2\quad \forall (z_1,w_1), \, (z_2,w_2)\in \R^d\times \R^d.
\]
Then for any point $(z,w)\in \R^d\times \R^d$ and any tangent vector $(u,v)\in \R^d\times \R^d$, we have  by using the mean value theorem that
$G(z+u, w+v) - G(z,w) = \langle \nabla G(z_0,w_0), (u,v)  \rangle$ for some point $(z_0,w_0)$ in the line segment  in $\R^d\times \R^d$ connecting the points
$(z,w)$ and $(z+u, w+v)$. As a consequence, we obtain 
\begin{align*}
  &\Big| G(z+u, w+v) - G(z,w) - \big[\langle \nabla^1 G(z,w), u\rangle + \langle \nabla^2 G(z,w), v\rangle \big] \Big| \\
  &=  \Big| \langle \nabla G(z_0,w_0), (u,v)  \rangle - \langle \nabla G(z,w), (u, v)\rangle \Big|\\
  &=  \Big| \big\langle \nabla G(z_0,w_0) - \nabla G(z,w), (u,v)  \big\rangle \Big|\leq \| \nabla G(z_0,w_0) - \nabla G(z,w)\|_2 \| (u,v)\|_2. 
\end{align*}
Then  we can use the  Lipschitz condition for $\nabla G$ to imply   further that
\begin{align*}
  \Big| G(z+u, w+v) 
  &- G(z,w) - \big[\langle \nabla^1 G(z,w), u\rangle + \langle \nabla^2 G(z,w), v\rangle \big] \Big|\\ 
 &\leq L \| (z_0,w_0) - (z,w)\|_2 \|  (u,v)\|_2
\leq L \|  (u,v)\|_2^2,
\end{align*}
which is the same as 
\begin{align*}
   \Big| G(z+u, w+v) 
  - G(z,w) - \big[\langle \nabla^1 G(z,w), u\rangle + \langle \nabla^2 G(z,w), v\rangle \big] \Big|
 \leq   L \big[\| u\|_2^2 +\|v\|_2^2 \big]. 
\end{align*}
Condition~\eqref{Lipschitz-cond} is the Riemannian version of this last inequality for the Euclidean space, which is a consequence of the standard Lipschitz condition for the gradient.

\paragraph{Bounded Hessian kernels are  Lipschitz-gradient.}

The following lemma gives a sufficient condition  for a kernel to be  Lipschitz-gradient. 
\begin{lemma}\label{c1-imply-lipschitz} Let $k$ be a positive definite kernel such that its Hessian w.r.t.~distance $d$ is   bounded. Then $k$ is a Lipschitz-gradient kernel.
%In particular, the Gaussian kernel described in %Section~\ref{sec:kernel} is Lipschitz-gradient.
\end{lemma}
\begin{proof}[Proof of Lemma~\ref{c1-imply-lipschitz}]
Let $H^1_d k(z,w)$ and $H^2_d k(z,w)$ denote  respectively the Hessian  w.r.t.~distance $d$ of the function  $z\mapsto k(z,w)$ and the function $w\mapsto k(z,w)$. Let $\e, \, \delta>0$, and $\Phi: \mc Z\to \R^{m}\times \R^n\times \mS^n$ be a bounded vector field. 
Define  $\gamma_z(t) \Let \exp_z(t\Phi(z)) $  and  $\theta_w(t) \Let \exp_w(\frac{\delta}{\e}t\Phi(w)) $ for $t\in [0,\e]$. Then we have 
\begin{align*}
   &k(\exp_z(\e\Phi(z)), \exp_w(\delta\Phi(w)))
    - k(z,w)
    = \int_0^\e \frac{\dd}{\dd t}[k(\gamma_z(t), \theta_w(t))] \, \dd t\\
     &= \int_0^\e \Big[ \langle \widetilde\nabla^1_d k(\gamma_z(t), \theta_w(t)), \dot\gamma_z(t) \rangle_{\gamma_z(t)} + \langle \widetilde\nabla^2_d k(\gamma_z(t), \theta_w(t)), \dot\theta_w(t) \rangle_{\theta_w(t)}] \Big] \dd t.
\end{align*}
This together with the facts that 
$\dot\gamma_z(0) = \Phi(z) $ and $\dot\theta_w(0) = \frac{\delta}{\e}\Phi(z)$ yields
\begin{align*}
    A
    &\Let k(\exp_z(\e\Phi(z)), \exp_w(\delta\Phi(w)))
    - k(z,w) -\big[  \langle \widetilde\nabla_d^1 k(z, w), \e\Phi(z)\rangle_z + \langle \widetilde\nabla_d^2 k(z, w), \delta\Phi(w)\rangle_w \big]\\
    &= \int_0^\e \Big[ \langle \widetilde\nabla^1_d k(\gamma_z(t), \theta_w(t)), \dot\gamma_z(t) \rangle_{\gamma_z(t)} - \langle \widetilde\nabla^1_d k(\gamma_z(0), \theta_w(0)), \dot\gamma_z(0) \rangle_{\gamma_z(0)}  \Big] dt\\
    &\quad + \int_0^\e \Big[  \langle \widetilde\nabla^2_d k(\gamma_z(t), \theta_w(t)), \dot\theta_w(t) \rangle_{\theta_w(t)}
    -\langle \widetilde\nabla^2_d k(\gamma_z(0), \theta_w(0)), \dot\theta_w(0) \rangle_{\theta_w(0)}] \Big] \dd t\\
    &= \int_0^\e \int_0^t  \frac{\dd}{\dd s} \Big[ \langle \widetilde\nabla^1_d k(\gamma_z(s),  \theta_w(s)), \dot\gamma_z(s) \rangle_{\gamma_z(s)}  \Big] \dd s \dd t\\
    &\quad + \int_0^\e \int_0^t  \frac{\dd}{\dd s} \Big[   \langle \widetilde\nabla^2_d k(\gamma_z(s), \theta_w(s)), \dot\theta_w(s) \rangle_{\theta_w(s)}  \Big]\dd s \dd t\\
    &= \int_0^\e \int_0^t  \Big[  \langle H^1_d k(\gamma_z(s),  \theta_w(s)) \dot\gamma_z(s), \dot\gamma_z(s) \rangle_{\gamma_z(s)}  + \langle \widetilde\nabla^1_d k(\gamma_z(s),  \theta_w(s)), \ddot\gamma_z(s) \rangle_{\gamma_z(s)}  \Big]\dd s \dd t\\
    &\quad + \int_0^\e \int_0^t   \Big[ \langle H^2_d k(\gamma_z(s), \theta_w(s)) \dot\theta_w(s), \dot\theta_w(s) \rangle_{\theta_w(s)} + \langle \widetilde\nabla^2_d k(\gamma_z(s), \theta_w(s)), \ddot\theta_w(s) \rangle_{\theta_w(s)}\Big] \dd s \dd t.
\end{align*}
Since the curve $s \mapsto \gamma_z(s)$ is a geodesic, its acceleration $\ddot\gamma_z(s)$ is orthogonal to $\mc Z$ (that is, $\ddot\gamma_z(s)$ is orthogonal to every tangent vector in $\mathrm T_{z} \mc Z$). This implies that $\langle \widetilde\nabla^1_d k(\gamma_z(s),  \theta_w(s)), \ddot\gamma_z(s) \rangle_{\gamma_z(s)}   = 0$. Likewise, we also have $\langle \widetilde\nabla^2_d k(\gamma_z(s), \theta_w(s)), \ddot\theta_w(s) \rangle_{\theta_w(s)} = 0$. Thanks to these, we deduce from the above identity  that 
\begin{align*}
    A
    &= \int_0^\e \int_0^t    \langle H^1_d k(\gamma_z(s),  \theta_w(s)) \dot\gamma_z(s), \dot\gamma_z(s) \rangle_{\gamma_z(s)}  \dd s \dd t\\
    &\quad + \int_0^\e \int_0^t   \langle H^2_d k(\gamma_z(s), \theta_w(s)) \dot\theta_w(s), \dot\theta_w(s) \rangle_{\theta_w(s)}  \dd s \dd t.
\end{align*}
By using the assumption that the Hessians $H^1_d$ and $H^2_d$ are bounded, we then obtain \begin{align*}
    |A|
    &\leq M  \int_0^\e \int_0^t    \Big[ \|\dot\gamma_z(s)\|_{\gamma_z(s)}^2 + \|\dot\theta_w(s)\|_{\theta_w(s)}^2\Big] \dd s \dd t,
\end{align*}
where $M$ is the sup norm of the Hessian of $k$.
But as $\gamma_z(s)$ and   $\theta_w(s)$ are geodesic, they have constant speeds. Therefore, $\|\dot\gamma_z(s)\|_{\gamma_z(s)}= \|\dot\gamma_z(0)\|_{\gamma_z(0)} = \|\Phi(z)\|_z$ and $\|\dot\theta_w(s)\|_{\theta_w(s)}= \|\dot\theta_w(0)\|_{\theta_w(0)} = \|\frac{\delta}{\e}\Phi(z)\|_z$. Using these, we infer further that  
\begin{align*}
    |A|
    &\leq M  \int_0^\e \int_0^t    \Big[ \|\Phi(z)\|_z^2 + (\frac{\delta}{\e})^2\|\Phi(z)\|_z^2\Big] \dd s \dd t
    = \frac{M}{2} 
    \Big[\| \e \Phi(z)\|_z^2 +\| \delta \Phi(w)\|_w^2 \Big]. 
\end{align*}
According to Definition~\ref{def:L-kernel}, we thus conclude that $k$ is a Lipschitz-gradient kernel with constant $M/2$. 
\end{proof}

\paragraph{Quantified estimate of decrease for 
the Riemannian forward Euler scheme \eqref{k_iteration}.}

The next result  quantifies the amount that the value of $\calF$ decreases after each iteration. 

\begin{proposition}[Quantified estimate of decrease]\label{quantified_decrease}
Suppose  that  $k$ is a Lipschitz-gradient kernel with constant $L$. Then for $\rho^{\tau+1}$ given by \eqref{k_iteration} with $s_\tau\in (0, \e_0]$, we have
\[
 \calF[\rho^{\tau+1}] - \calF[\rho^\tau]  \leq -s_\tau\big(1- 2L s_\tau\big)\int_{\mc Z} \|\widetilde\nabla_d [\m_{\rho^\tau} -\m_\varrho](z)\|_z^2 \, \rho^\tau(\dd z).
 \]
\end{proposition}
\begin{proof}[Proof of Proposition~\ref{quantified_decrease}]
Let $\Phi^\tau \Let - \widetilde\nabla_d [\m_{\rho^\tau} -\m_\varrho]$. Then
 from the computation at the beginning of the proof of Lemma~\ref{descent}
and by using  Lemma~\ref{closed-form}, we obtain 
 \begin{align*}
 \calF[\rho^{\tau+1}] - \calF[\rho^\tau]
&= \frac12 \left[\mathrm{MMD}(\exp(s_\tau \Phi^\tau)_{\#}\rho^\tau, \varrho)^2 -  \mathrm{MMD}(\rho^\tau, \varrho)^2\right]\\
&=  \frac12\iint  \Big\{k\Big(\exp_z(s_\tau\Phi^\tau(z)),\exp_w(s_\tau\Phi^\tau(w))\Big) - k(z,w)\Big\} \rho^\tau(\dd z) \rho^\tau(\dd w) \\
&\quad - \iint \Big\{ k\Big(\exp_z(s_\tau\Phi^\tau(z)), w\Big) - k(z,w)\Big\} \rho^\tau(\dd z) \varrho(\dd w).
\end{align*}
Moreover, we have
\begin{align*}
 &\int \langle \widetilde\nabla_d [\m_{\rho^\tau} -\m_\varrho] (z), \Phi^\tau(z)\rangle_z \rho^\tau(\dd z )\\
 &= \iint \langle \widetilde\nabla_d^1 k(z, w), \Phi^\tau(z)\rangle_z \rho^\tau(\dd z ) \rho^\tau(\dd w )
 - \iint \langle \widetilde\nabla_d^1 k(z, w), \Phi^\tau(z)\rangle_z \rho^\tau(\dd z ) \varrho(\dd w )\\
 &= \frac12\left[\iint \langle \widetilde\nabla_d^1 k(z, w), \Phi^\tau(z)\rangle_z \rho^\tau(\dd z ) \rho^\tau(\dd w )
 +\iint \langle \widetilde\nabla_d^2 k(z, w), \Phi^\tau(w)\rangle_w \rho^\tau(\dd z ) \rho^\tau(\dd w )\right]\\
 &\quad - \iint \langle \widetilde\nabla_d^1 k(z, w), \Phi^\tau(z)\rangle_z\rho^\tau(\dd z ) \varrho(\dd w ),
 \end{align*}
 where the last equality is due to the symmetry of $k$ and relation \eqref{metric-gradient-dataset}. Here  $\widetilde\nabla_d^1 k(z, w)$ 
and $\widetilde\nabla_d^2 k(z, w)$ respectively 
denote the gradient $\widetilde\nabla_d$ of the function $z\mapsto k(z,w)$ and  $w\mapsto k(z,w)$. Therefore, it follows that  
 \begin{align*}
 \calF[\rho^{\tau+1}] 
 &- \calF[\rho^\tau]
   -s_\tau\int \langle \widetilde\nabla_d [\m_{\rho^\tau} -\m_\varrho] (z), \Phi^\tau(z)\rangle_z \rho^\tau(\dd z )\\
 &=  \frac12\iint  \left\{k\Big(\exp_z(s_\tau\Phi^\tau(z)),\exp_w(s_\tau\Phi^\tau(w))\Big) - k(z,w)\right.\\
 &\qquad\qquad \left. -\Big[\langle \widetilde\nabla_d^1 k(z, w), s_\tau\Phi^\tau(z)\rangle_z + \langle \widetilde\nabla_d^2 k(z, w), s_\tau\Phi^\tau(w)\rangle_w
 \Big] \right\} \rho^\tau(\dd z) \rho^\tau(\dd w) \\
&\quad - \iint \Big\{ k\Big(\exp_z(s_\tau\Phi^\tau(z)), w\Big) - k(z,w)- \langle \widetilde\nabla_d^1 k(z, w), s_\tau\Phi^\tau(z)\rangle_z\Big\} \rho^\tau(\dd z) \varrho(\dd w).
 \end{align*}
 As $s_\tau \in (0,\e_0]$, we can now use the assumption that $k$ is a Lipschitz-gradient kernel   with constant $L$ to obtain 
 \begin{align*}
 &\calF[\rho^{\tau+1}] - \calF[\rho^\tau]
   +s_\tau \int_{\mc Z} \|\Phi^\tau(z)\|_z^2 \rho^\tau(\dd z)\\
 &\leq   \frac{L}{2} \iint \Big[ \|s_\tau \Phi^\tau(z)\|_z^2 +\|s_\tau \Phi^\tau(w)\|_w^2 \Big]\rho^\tau(\dd z) \rho^\tau(\dd w) + 
L \iint \|s_\tau \Phi^\tau(z)\|_z^2  \rho^\tau(\dd z) \varrho(\dd w)\\
&= 2 L s_\tau^2 \int \|\Phi^\tau(z)\|_z^2 \rho^\tau(\dd z). 
 \end{align*}
 This gives 
 \[
 \calF[\rho^{\tau+1}] - \calF[\rho^\tau]
 \leq  \big( -s_\tau  + 2 L s_\tau^2\big) \int_{\mc Z} \|\Phi^\tau(z)\|_z^2 \rho^\tau(\dd z),
 \]
 and the conclusion of the proposition follows.
\end{proof}

\paragraph{Convergence guarantees.}

For each distribution $\rho$ on $\mc Z$, let $\mathbb{K}_\rho: \mathcal H \to \mathcal H$ be the linear operator defined by
$\mathbb{K}_\rho f(w_1) \Let \langle \tilde{\mathbb{K}}_\rho(w_1, \cdot) , f(\cdot)\rangle_{\mathcal H}$ 
with $\tilde{\mathbb{K}}_\rho: \mc Z\times \mc Z \to \R$ being given  by
\[
\tilde{\mathbb{K}}_\rho(w_1, w_2) = \int_{\mc Z} \langle \widetilde\nabla_d^1 k(z, w_1), \widetilde\nabla_d^1 k(z, w_2)\rangle_z \, \rho(\dd z)
\quad\mbox{for}\quad w_1,\, w_2\in \mc Z.
\]
%The operator $\mathbb{K}_\rho$ is symmetric and  positive,  
%its  eigenvalues are nonnegative, 
%and $\langle \mathbb{K}_\rho [ \m_{\rho} - \m_{\varrho}],   \m_{\rho} - \m_{\varrho}\rangle_{\mathcal H} = \int_{\mc Z} \big\| \nabla_d [ \m_{\rho} - \m_{\varrho}]\big\|_z^2 \, \rho(\dd z)$ (see Lemma A.6 in the Appendix). We 
% shows in Proposition~\ref{convergence} of the Appendix that $\rho_t$  globally converges 
%in $\mathrm{MMD}$ 
%if the  minimum eigenvalue $\lambda_t$ of operator $\mathbb{K}_{\rho_t}$ satisfies an integrability condition. 
%Let us recall that for each probability distribution $\rho$ on $\mc Z$, the operator $\mathbb{K}_\rho: \mathcal H \to \mathcal H$ is  defined by
%$\mathbb{K}_\rho f(w_1) \Let \langle \tilde{\mathbb{K}}_\rho(w_1, \cdot) , f(\cdot)\rangle_{\mathcal H}$ 
%with $\tilde{\mathbb{K}}_\rho: \mc Z\times \mc Z \to \R$ being given  by
%\[
%\tilde{\mathbb{K}}_\rho(w_1, w_2) = \int_{\mc Z} \langle \nabla_d^1 k(z, w_1), \nabla_d^1 k(z, w_2)\rangle_z \, \rho(\dd z)
%\quad\mbox{for}\quad w_1,\, w_2\in \mc Z.
%\]
The next result gives some basic properties of  the operator $\mathbb{K}_\rho$. 
\begin{lemma}\label{operator}
For a differentiable kernel $k$ and for $\rho\in \calP(\mc Z)$, we have  
\begin{enumerate}
\item[i)] $\mathbb{K}_\rho f(w) = \int_{\mc Z} \langle \widetilde\nabla_d^1 k(z, w), \widetilde\nabla_d f(z)\rangle_z \, \rho(\dd z)$ for $f\in\mathcal H$. 
\item[ii)] $\langle \mathbb{K}_\rho f,  g\rangle_{\mathcal H} = \int_{\mc Z} \langle \widetilde\nabla_d f, \widetilde\nabla_d g \rangle_z  \, \rho(\dd z)$ for every $f, \, g\in \mathcal H$. Consequently,
the operator $\mathbb{K}_\rho$ is symmetric and positive, and hence  its spectrum is contained in $[0,+\infty)$.
%\item[iii)] $\langle \mathbb{K}_\rho ( \m_{\rho} - \m_p),  \m_{\rho} - \m_p\rangle_{\mathcal H} = \int_{\mS^n_{++}} \big\| \nabla_{\mathbb{W}}[ \m_{\rho} - \m_p]\big\|_\Sigma^2 \, %\rho(\dd \Sigma)$.
\end{enumerate}
\end{lemma}
\begin{proof}[Proof of Lemma~\ref{operator}]
 By using the definition of the  Riemannian metric $\langle \cdot,\cdot \rangle_z$ given in \eqref{R-metric-Z}, it  can be verified for $f\in \mathcal H$ that 
\begin{equation*}
\Big\langle  \langle \widetilde\nabla_d^1 k(z,w), \widetilde\nabla_d^1 k(z,  \cdot)\rangle_z, f(\cdot) \Big\rangle_{ \mathcal H} = \Big\langle  \widetilde\nabla_d^1 k(z, w), \langle \widetilde\nabla_d^1 k(z, \cdot), f(\cdot)\rangle_{ \mathcal H} \Big\rangle_z.
\end{equation*}
As $f(z) = \langle k(z,\cdot), f(\cdot) \rangle_{ \mathcal H}$, we moreover have $ \widetilde\nabla_d f(z) = \langle  \widetilde\nabla_d^1 k(z, \cdot), f(\cdot) \rangle_{ \mathcal H}$.
Therefore,
\begin{equation}\label{two-products}
\Big\langle  \langle \widetilde\nabla_d^1 k(z, w), \widetilde\nabla_d^1 k(z,  \cdot)\rangle_z, f(\cdot) \Big\rangle_{ \mathcal H} 
= \big\langle  \widetilde\nabla_d^1 k(z, w),  \widetilde\nabla_d  f(z) \big\rangle_z.
\end{equation}
Using the definition of $\mathbb{K}_\rho$ and \eqref{two-products}, we obtain
\begin{align*}
\mathbb{K}_\rho f(w) 
&= \Big\langle \int \langle \widetilde\nabla_d^1 k(z, w), \widetilde\nabla_d^1 k(z, \cdot)\rangle_z \, \rho(\dd z),
f(\cdot)\Big\rangle_{ \mathcal H} \\
&= \int \Big\langle  \langle \widetilde\nabla_d^1 k(z, w), \widetilde\nabla_d^1 k(z, \cdot)\rangle_z,
f(\cdot)\Big\rangle_{ \mathcal H} \, \rho(\dd z)\\
&= \int \big\langle  \widetilde\nabla_d^1 k(z,w),  \widetilde\nabla_d  f(z) \big\rangle_z \, \rho(\dd z),
\end{align*}
which gives i).  Now  for $f,\, g\in  \mathcal H$, we can use part i) and similar arguments leading to \eqref{two-products}
to obtain 
 \begin{align*}
\langle \mathbb{K}_\rho f, g\rangle_{ \mathcal H}
&= \Big\langle \int \big\langle  \widetilde\nabla_d^1 k(z, \cdot),  \widetilde\nabla_d  f(z) \big\rangle_z \, \rho(\dd z),
g(\cdot)\Big\rangle_{ \mathcal H} \\
&= \int \Big\langle  \langle \widetilde\nabla_d^1 k(z, \cdot), \widetilde\nabla_d f(z)\rangle_z,
g(\cdot)\Big\rangle_{ \mathcal H} \, \rho(\dd z)\\
&= \int \Big\langle  \langle \widetilde\nabla_d^1 k(z, \cdot), g(\cdot)\rangle_{ \mathcal H},
\widetilde\nabla_d f(z)\Big\rangle_z \, \rho(\dd z)
= \int \big\langle  \widetilde\nabla_d  g(z),  \widetilde\nabla_d  f(z) \big\rangle_z \, \rho(\dd z).
\end{align*}
This implies in particular that  the operator $\mathbb{K}_\rho$ is  symmetric (i.e. 
$\langle \mathbb{K}_\rho f, g\rangle_{ \mathcal H} = \langle \mathbb{K}_\rho g, f\rangle_{ \mathcal H}$ for $f,\, g\in \mathcal H$) and 
 positive (i.e. $\langle \mathbb{K}_\rho f, f\rangle_{ \mathcal H}\geq 0$ for $f\in \mathcal H$). Since any symmetric, positive, and linear operator 
 must have nonnegative eigenvalues, we have completed the proof.
\end{proof}
Our next result gives a quantified decay rate  for the objective function.
\begin{proposition}[Objective value decay]\label{convergence}
There hold:
\begin{enumerate}[leftmargin = 5mm]
\item[i)] Let $\rho_t$ be given by  \eqref{eq:gradient-flow}, and let $\lambda_t \geq 0$ 
%denote   the minimum eigenvalue %of  operator  %$\mathbb{K}_{\rho_t}$.
be any constant satisfying
 \begin{equation}\label{ass_continuous}
      \langle  \mathbb{K}_{\rho_t}f_t, f_t \rangle_{\mathcal H} \geq \lambda_t \|f_t\|_{\mathcal H}^2\quad\mbox{with}\quad f_t \Let \m_{\rho_t} -\m_{\varrho}.
   \end{equation}
Then 
$\calF[\rho_t]   \leq \calF[\rho_0]   \exp{\big(-2\int_0^t \lambda_s \dd s\big)}$ for any $t\geq 0$. In particular,   $\lim_{t\to\infty}\mathrm{MMD}(\rho_t, \varrho)= 0$ if $\int_0^\infty \lambda_t\, \dd t =+\infty$.
\item[ii)]  Let $\rho^\tau$ be given by scheme \eqref{k_iteration}, and $\lambda_\tau \geq 0$
%denote   the minimum eigenvalue of  operator  $\mathbb{K}_{\rho^\tau}$.
be any constant satisfying
 \begin{equation*}
      \langle  \mathbb{K}_{\rho_t}f^\tau, f^\tau \rangle_{\mathcal H} \geq \lambda_\tau \|f^\tau|_{\mathcal H}^2\quad\mbox{with}\quad f^\tau \Let \m_{\rho^\tau} -\m_{\varrho}.
   \end{equation*}
Assume that 
$k$ is a Lipschitz-gradient kernel
%\footnote{See Definition A.3 %in the Appendix for the %technical definition of a %Lipschitz-gradient kernel}
and  step size $s_\tau$ satisfies  $s_\tau \lambda_\tau < 1$,
then we have
$\calF[\rho^{\tau+1}] \leq \calF[\rho^0] \exp\big(-{\sum_{i=0}^\tau} s_i \lambda_i\big)$ for any $\tau\geq 0$. 
In particular, $\lim_{\tau\to \infty}\mathrm{MMD}(\rho^\tau, \varrho) =0$ if  $ \sum_{\tau=0}^\infty s_\tau \lambda_\tau = +\infty$.
\end{enumerate}
\end{proposition}
Condition $\int_0^\infty \lambda_t\, \dd t =+\infty$ guaranteeing the  convergence in 
$\mathrm{MMD}$ holds true for example if $\lambda_t \geq  c \, t^{-1}$ for some constant $c>0$ and for large $t$. We note also that
Condition~\eqref{ass_continuous} is satisfied if $\lambda_t$ is chosen to be the minimum eigenvalue of  operator  $\mathbb{K}_{\rho_t}$. 
Thus  Proposition~\ref{convergence} implies in particular  that $\rho_t$  globally converges 
in $\mathrm{MMD}$ 
if the  minimum eigenvalue $\lambda_t$ of operator $\mathbb{K}_{\rho_t}$ satisfies the integrability condition $\int_0^\infty \lambda_t\, \dd t =+\infty$. 
The proof of Proposition~\ref{convergence}
relies on the following proposition, which shows  that the dynamic of the mean embedding is governed by  the  equation $\partial_t (\m_{\rho_t} -\m_{\varrho}) = - \mathbb{K}_{\rho_t} (\m_{\rho_t} -\m_{\varrho})$. 
%The next result is a preparation for the proof %of Proposition~\ref{convergence}.
\begin{proposition}[Dynamic of the mean embedding]
\label{prop:dynamic_eq}
Let $t\in [0,\infty)\longmapsto \rho_t$ be the gradient flow given  by equation \eqref{eq:gradient-flow}. For each $t\geq 0$, take $f_t \Let \m_{\rho_t} -\m_{\varrho}$. Then $f_t$ is a solution of the  linear partial differential equation 
\begin{equation}\label{PDE}
\partial_t f_t = - \mathbb{K}_{\rho_t} f_t \quad \mbox{in}\quad [0,\infty)\times \mc Z.
\end{equation}
\end{proposition}
\begin{proof}[Proof of Proposition~\ref{prop:dynamic_eq}]
 From the definition of the mean embedding and by using equation \eqref{eq:gradient-flow}, we have
 \begin{align*}
   \partial_t f_t(w) = \partial_t \m_{\rho_t}(w)
   &= \partial_t \int_{\mc Z} k(z, w) \, \rho_t(\dd z)= \int_{\mc Z} k(z, w) \, \partial_t\rho_t(\dd z)\\
   &= \int_{\mc Z} k(z, w) \, \div_d (\rho_t \widetilde\nabla_d f_t ) (\dd z).
 \end{align*}
 Using the definition of the divergence operator $\div_d$ at the end of Section~3.1, we further obtain 
 \begin{align*}
   \partial_t f_t(w) 
   &= - \int_{\mc Z} \langle \widetilde\nabla_d^1k(z,w), \widetilde\nabla_d f_t(z)\rangle_z \, \rho_t(\dd z).
 \end{align*}
 It then follows from part i) of Lemma~\ref{operator} that $ \partial_t f_t(w) = - \mathbb{K}_{\rho_t} f_t(w)$.
 This completes the proof.
\end{proof}

We are now ready to present the proof of Proposition~\ref{convergence}.
%%%%%%%%%%%%%%%%%%%%%
%\paragraph{For Proposition~\ref{convergence}}

\begin{proof}[Proof of Proposition~\ref{convergence}]
Let $f_t \Let \m_{\rho_t} -\m_{\varrho}$. Then  we have from  Proposition~\ref{decrease-rate} and part ii) of Lemma~\ref{operator} that 
$\partial_t \|f_t\|_{ \mathcal H}^2 =-2 \langle \mathbb{K}_{\rho_t} f_t, f_t  \rangle_{ \mathcal H}$. But as %$\lambda_t$ is the minimum %eigenvalue of 
%the operator %$\mathbb{K}_{\rho_t}$, we obtain
 \begin{equation*}
      \langle  \mathbb{K}_{\rho_t}f_t, f_t \rangle_{\mathcal H} \geq \lambda_t \|f_t\|_{\mathcal H}^2
   \end{equation*}
  by Condition~\ref{ass_continuous}, we infer  that
$\partial_t \|f_t\|_{ \mathcal H}^2 \leq -2 \lambda_t \|f_t\|_{ \mathcal H}^2$, and hence $\partial_t \Big(\log{\|f_t\|_{ \mathcal H}^2}\Big) \leq -2 \lambda_t$. By 
integrating from $0$ to $t$, one gets $\log{\|f_t\|_{ \mathcal H}^2} - \log{\|f_0\|_{ \mathcal H}^2}\leq -2\int_0^t \lambda_s \, \dd s$. We next take  
exponential to obtain 
\[
\|f_t\|_{ \mathcal H}^2 \leq \|f_0\|_{ \mathcal H}^2 \, \exp{\Big(-2\int_0^t \lambda_s \, \dd s\Big)}.
\]
%\viet{sometimes it is $\rho_0$, sometimes it is $\rho^0$. Please check for consistency.}
This can be rewritten as $\calF[\rho_t]   \leq \calF[\rho_0]   \exp{\big(-2\int_0^t \lambda_s \dd s\big)}$ for  $t\geq 0$.
In particular, $\calF[\rho_t]$ (and hence  $\mathrm{MMD}(\rho_t, \varrho)$) tends to zero if
$\int_0^\infty \lambda_t\, \dd t =+\infty$. This completes the proof for part i).

To prove ii), let $f^\tau \Let \m_{\rho^\tau} -\m_\varrho$. Notice that in contrast to the continuous case, upper indices are used for $f^\tau$ and 
$\rho^\tau$ in the discrete case.
Then by using Proposition~\ref{quantified_decrease} together with part ii) of Lemma~\ref{operator} and the assumption $s_\tau\in (0, \frac{1}{4L}]$ we have
\[
\calF[\rho^{\tau+1}] - \calF[\rho^\tau] \leq -\frac12 s_\tau  \langle  \mathbb{K}_{\rho^\tau}f^\tau, f^\tau \rangle_{ \mathcal H}.
\]
But as  
%$\lambda_\tau$ is the minimum %eigenvalue of the operator %$\mathbb{K}_{\rho^\tau}$, we %also have
$
\langle  \mathbb{K}_{\rho^\tau}f^\tau, f^\tau \rangle_{\mathcal H} \geq \lambda_\tau \|f^\tau\|_{\mathcal H}^2$
due to our assumption, we obtain  $ \calF[\rho^{\tau+1}] - \calF[\rho^\tau] \leq - s_\tau \lambda_\tau  \calF[\rho^\tau]$, or 
\[
\calF[\rho^{\tau+1}] \leq (1- s_\tau \lambda_\tau)  \calF[\rho^\tau]
\]
for every $\tau\geq 0$. As $1- s_\tau \lambda_\tau> 0$, it follows by iteration that $ \calF[\rho^{\tau+1}] \leq \calF[\rho^0]
{\displaystyle \prod_{i=0}^\tau (1- s_i \lambda_i)}$. Due to $1- x\leq \exp(-x)$ for every $x\geq 0$, we infer that $\calF[\rho^{\tau+1}] \leq \calF[\rho^0] \exp\big(-{\displaystyle\sum_{i=0}^\tau} s_i \lambda_i\big)  $ for  $\tau\geq 0$. 
In particular, $\calF[\rho^\tau]$ (and hence  $\mathrm{MMD}(\rho^\tau, \varrho)$) tends to zero if
 ${\displaystyle \sum_{\tau=0}^\infty} s_\tau \lambda_\tau = +\infty$.
\end{proof}

%%%%%%%%%%%%%%%%%%%%%%%%%
\paragraph{For Proposition~\ref{conv-noisy-gradient}}

\begin{proof}[Proof of Proposition~\ref{conv-noisy-gradient}]
Let  $h(z)\Let \exp_z(s_\tau \Phi^\tau(z))$ for $z\in \mc \mc Z$.  Then $\rho^{\tau+1}$ can be expressed as 
\begin{equation*}
\rho^{\tau+1} = h_{\#} \rho^{\tau,\beta_\tau} = (h \circ f^{\beta_\tau})_{\#}(\rho^\tau \otimes g).
\end{equation*}
By  the computation at the beginning of the proof of Lemma~\ref{descent}
using  Lemma~\ref{closed-form},  we obtain 
 \begin{align*}
 \calF[\rho^{\tau+1}] 
 &- \calF[\rho^\tau]
= \frac12 \left[\mathrm{MMD}\Big((h \circ f^{\beta_\tau})_{\#}(\rho^\tau\otimes g), \varrho\Big)^2 -  \mathrm{MMD}(\rho^\tau, \varrho)^2\right]\\
&=  \frac12\iiiint  \Big\{k\Big(h(f^{\beta_\tau}(z,u)),h(f^{\beta_\tau}(w,v))\Big) - k(z,w)\Big\} \rho^\tau(\dd z) g(\dd u) \rho^\tau(\dd w) g(\dd v) \\
&\quad - \iiint \Big\{ k\Big(h(f^{\beta_\tau}(z,u)), w\Big) - k(z,w)\Big\} \rho^\tau(\dd z) g(\dd u) \varrho(\dd w).
\end{align*}
Moreover, we have
\begin{align*}
 I &\Let \int \langle \widetilde\nabla_d [\m_{\rho^\tau} -\m_\varrho] (z), \Phi^\tau(z)\rangle_z \rho^{\tau,\beta_\tau}(\dd z )\\
 &= \iint \langle \widetilde\nabla_d^1 k(z, w), \Phi^\tau(z)\rangle_z \rho^{\tau,\beta_\tau}(\dd z ) \rho^\tau(\dd w )
 - \iint \langle \widetilde\nabla_d^1 k(z, w), \Phi^\tau(z)\rangle_z \rho^{\tau,\beta_\tau}(\dd z ) \varrho(\dd w )\\
 &= \frac12\iint \langle \widetilde\nabla_d^1 k(z, w), \Phi^\tau(z)\rangle_z \rho^{\tau,\beta_\tau}(\dd z ) \rho^\tau(\dd w )
 +\frac12\iint \langle \widetilde\nabla_d^2 k(z, w), \Phi^\tau(w)\rangle_w \rho^{\tau,\beta_\tau}(\dd w ) \rho^\tau(\dd z)\\
 &\quad - \iint \langle \widetilde\nabla_d^1 k(z, w), \Phi^\tau(z)\rangle_z\rho^{\tau,\beta_\tau}(\dd z ) \varrho(\dd w )\\
 &= \frac12\iiint \Big\langle \widetilde\nabla_d^1 k(f^{\beta_\tau}(z,u), w), \Phi^\tau(f^{\beta_\tau}(z,u))\Big\rangle_{f^{\beta_\tau}(z,u)} \rho^\tau(\dd z )  g(\dd u)\rho^\tau(\dd w )\\
 &\quad +\frac12\iiint \Big\langle \widetilde\nabla_d^2 k(z, f^{\beta_\tau}(w,v)), \Phi^\tau(f^{\beta_\tau}(w,v))\Big\rangle_{f^{\beta_\tau}(w,v)}  \rho^\tau(\dd w ) g(\dd v) \rho^\tau(\dd z )\\
 &\quad - \iiint \Big\langle \widetilde\nabla_d^1 k(f^{\beta_\tau}(z,u), w), \Phi^\tau(f^{\beta_\tau}(z,u))\Big\rangle_{f^{\beta_\tau}(z,u)}\rho^\tau(\dd z )  g(\dd u) \varrho(\dd w ),
 \end{align*}
 where the third  equality is due to the symmetry of $k$ and relation \eqref{metric-gradient-dataset}. Therefore, it follows that  
 \begin{align*}
 &\calF[\rho^{\tau+1}] - \calF[\rho^\tau]
   -s_\tau I \\
 &=  \frac12\iiiint  \left\{k\Big(h(f^{\beta_\tau}(z,u)),h(f^{\beta_\tau}(w,v))\Big) - k(z,w)\right.\\
 &\qquad\qquad\qquad  \left. -\Big[\langle \widetilde\nabla_d^1 k(f^{\beta_\tau}(z,u), w), s_\tau\Phi^\tau(f^{\beta_\tau}(z,u))\rangle_{f^{\beta_\tau}(z,u)}\right.\\
 &\left. \qquad\qquad\qquad\quad + \langle \widetilde\nabla_d^2 k(z, f^{\beta_\tau}(w,v)), s_\tau\Phi^\tau(f^{\beta_\tau}(w,v))\rangle_{f^{\beta_\tau}(w,v)}
 \Big] \right\} \rho^\tau(\dd z) g(\dd u) \rho^\tau(\dd w) g(\dd v)\\
&\quad - \iiint \Big\{ k\Big(h(f^{\beta_\tau}(z,u)), w\Big) - k(z,w)- \langle \widetilde\nabla_d^1 k(f^{\beta_\tau}(z,u), w), s_\tau\Phi^\tau(f^{\beta_\tau}(z,u))
\rangle_{f^{\beta_\tau}(z,u)}\Big\}\\
&\qquad\qquad\qquad\qquad \qquad\qquad\qquad\qquad \qquad\qquad\qquad\qquad \qquad\qquad\qquad\quad \rho^\tau(\dd z) g(\dd u) \varrho(\dd w).
 \end{align*}
 As $h(z)=\exp_z(s_\tau \Phi^\tau(z))$ and $s_\tau \in (0,\e_0]$, we can now use the Lipschitz-gradient condition \eqref{Lipschitz-cond} for $k$ to obtain 
 \begin{align*}
 &\calF[\rho^{\tau+1}] - \calF[\rho^\tau]
   -s_\tau I\\
 &\leq   \frac{L}{2} \iiiint \Big[ \|s_\tau \Phi^\tau(f^{\beta_\tau}(z,u) )\|_{f^{\beta_\tau}(z,u)}^2 +\|s_\tau \Phi^\tau(f^{\beta_\tau}(w,v))\|_{f^{\beta_\tau}(w,v)}^2 
 \Big]\rho^\tau(\dd z) g(\dd u)\rho^\tau(\dd w) g(\dd v)\\
 &\quad + 
L \iiint \|s_\tau \Phi^\tau(f^{\beta_\tau}(z,u))\|_{f^{\beta_\tau}(z,u)}^2  \rho^\tau(\dd z) \varrho(\dd w) g(\dd u)\\
&= 2 L s_\tau^2 \iint \|\Phi^\tau(f^{\beta_\tau}(z,u))\|_{f^{\beta_\tau}(z,u)}^2  \rho^\tau(\dd z) g(\dd u). 
 \end{align*}
 Using the definition $\rho^{\tau, \beta_\tau} = f^{\beta_\tau}_{\#}(\rho^\tau \otimes g)$ and the fact 
 $I = -\int_{\mc Z} \|\Phi^\tau(z)\|_z^2 \, \rho^{\tau,\beta_\tau}(\dd z )$, we can rewrite this more compactly as 
 \begin{align*}
 \calF[\rho^{\tau+1}] - \calF[\rho^\tau] 
 &\leq -s_\tau\big(1- 2L s_\tau\big) \int_{\mc Z} \|\Phi^\tau(z)\|_z^2 \, \rho^{\tau,\beta_\tau}(\dd z )\\
 &= -s_\tau\big(1- 2L s_\tau\big)\int_{\mc Z} \|\widetilde\nabla_d [\m_{\rho^\tau} -\m_\varrho](z)\|_z^2 \, \rho^{\tau, \beta_\tau}(\dd z).
 \end{align*}
 This together with condition \eqref{noise-level-cond} gives 
 \[
 \calF[\rho^{\tau+1}]  \leq (1-a_\tau) \calF[\rho^\tau]\quad \mbox{with}\quad 
a_\tau \Let \lambda s_\tau\big(1- 2L s_\tau\big)\beta_\tau^2.
 \]
 In particular, we must have  $a_i\leq 1$. By iterating this estimate, we obtain 
 \begin{equation}\label{prod-est}
 \calF[\rho^{\tau+1}]  \leq \calF[\rho^0] \, \prod_{i=0}^\tau (1-a_i). 
 \end{equation}
 %Then by exactly the same argument as then end of the proof of Proposition~\ref{convergence}, we have 
 %$\calF[\rho^{\tau}] \to 0$ if $\sum_{\tau=0}^\infty a_\tau = +\infty$.
Due to $1- x\leq  \exp(-x)$ for every number $x\geq 0$, we get $\prod_{i=0}^\tau (1-a_i)\leq \exp(-\sum_{i=0}^\tau a_i)$. This together with \eqref{prod-est} yields the conclusion of the proposition. 
\end{proof}

%\clearpage
\newpage
\section{Implementation Details And Additional Results} \label{appendix:implementation}
We use $\nabla_{\mathbb B}^1 k(x,\mu, \Sigma, w)$ to denote the last component in \eqref{metric-gradient-dataset} for the gradient $\widetilde\nabla_d$
of the function $ (x,\mu, \Sigma) \mapsto k(x,\mu, \Sigma, w)$. Precisely,
\[
\nabla_{\mathbb B}^1 k(x,\mu, \Sigma, w) \Let 2 [\nabla_\Sigma k(x,\mu, \Sigma, w)]\Sigma + 2 \Sigma [\nabla_\Sigma k(x,\mu, \Sigma, w)].
\]

\subsection{Algorithms}

\begin{algorithm}[ht]
	\caption{Discretized Gradient Flow Algorithm for Scheme~\eqref{k_iteration} -- Detailed Version of Algorithm~\ref{alg:nonoise}}
	\label{alg:nonoise-extended}
	\begin{algorithmic}[1]
		\STATE {\bfseries Input:} a source distribution $\rho^0 = \frac1N \sum_{i=1}^N \delta_{(x^0_i, \mu^0_i, \Sigma^0_i)}$, a sample $\frac1M \sum_{j=1}^M \delta_{(\bar x_j, \bar\mu_j, \bar \Sigma_j)}$ for the target distribution $\varrho$,  a number $T$ of iterations for training,  a sequence of step sizes $s_\tau>0$ with $\tau=0,1,...,T$, and a kernel $k$.
		\STATE {\bfseries Initialization:} 
		\STATE Compute $\displaystyle (\bar\Psi_1, \bar\Psi_2, \bar\Psi_3) (x, \mu, \Sigma) = \frac1M \sum_{j=1}^M (\nabla_x, \nabla_\mu, \nabla_{\mathbb B}^1) k(x,\mu,\Sigma,\bar x_j, \bar\mu_j, \bar \Sigma_j)$\;
        \STATE $\tau \leftarrow 0$\;
 
        \WHILE{$\tau < T$}
        \STATE Compute $\displaystyle (\Psi_1^\tau, \Psi_2^\tau, \Psi_3^\tau)(x,\mu, \Sigma)  =  \frac1N \sum_{i=1}^N  (\nabla_x, \nabla_\mu, \nabla_{\mathbb B}^1) k(x,\mu,\Sigma,x^{\tau}_i,  \mu^{\tau}_i, \Sigma^{\tau}_i)$\;
            \FOR{$i = 1, \ldots, N$}
                \STATE $x^{\tau+1}_i \leftarrow  x^{\tau}_i +s_\tau (\bar\Psi_1 -\Psi_1^\tau)(x^\tau_i, \mu^\tau_i, \Sigma^\tau_i)$\;
                \STATE $\mu^{\tau+1}_i \leftarrow  \mu^{\tau}_i + s_\tau (\bar\Psi_2 - \Psi_2^\tau)(x^\tau_i, \mu^\tau_i, \Sigma^\tau_i)$\;
                \STATE $\Sigma^{\tau+1}_i \leftarrow  \Big( I +  s_\tau\rmL_{\Sigma^\tau_i}\big[(\bar\Psi_3-\Psi_3^\tau)(x^\tau_i, \mu^\tau_i, \Sigma^\tau_i)\big] \Big) \Sigma^\tau_i \Big( I + s_\tau \rmL_{\Sigma^\tau_i}[(\bar\Psi_3-\Psi_3^\tau)(x^\tau_i, \mu^\tau_i, \Sigma^\tau_i)] \Big)$\;

            \ENDFOR
            
        %\STATE Compute $\displaystyle (\Psi_1^{\tau+1}, \Psi_2^{\tau +1}, \Psi_3^{\tau +1}) (x,\mu,\Sigma)  =  \frac1N \sum_{i=1}^N  (\nabla_x, \nabla_\mu,  \nabla_\Sigma) k(x,\mu,\Sigma,x^{\tau+1}_i, \mu^{\tau+1}_i, \Sigma^{\tau+1}_i)$\;
        \STATE Set $\tau \leftarrow \tau+ 1$ 
        \ENDWHILE

		\STATE{\bfseries Output:} $\displaystyle \rho^T = \frac1N \sum_{i=1}^N \delta_{(x^T_i, \mu^T_i, \Sigma^T_i)}$
	\end{algorithmic}
\end{algorithm}

\begin{algorithm}[ht]
	\caption{Discretized Gradient Flow Algorithm for Scheme~\eqref{noisy-gradient-algo}}
	\label{alg:noisy}
	\begin{algorithmic}[1]
		\STATE {\bfseries Input:} a source distribution $\rho^0 = \frac1N \sum_{i=1}^N \delta_{(x_i, \mu_i, \Sigma_i)}$, a target distribution $\varrho = \frac1M \sum_{j=1}^M \delta_{(\bar x_j, \bar\mu_j, \bar \Sigma_j)}$,  number of iterations $T$, step sizes $s_\tau>0$, noise levels $\beta_\tau$, and a kernel $k$.
		\STATE {\bfseries Initialization:} 
		\STATE Compute $\displaystyle (\bar\Psi_1, \bar\Psi_2, \bar\Psi_3) (x, \mu, \Sigma) = \frac1M \sum_{j=1}^M (\nabla_x, \nabla_\mu, \nabla_{\mathbb B}^1) k(x,\mu,\Sigma,\bar x_j, \bar\mu_j, \bar \Sigma_j)$\;
        \STATE $\tau \leftarrow 0$\;
        \WHILE{$\tau < T$}
        \STATE Compute $\displaystyle (\Psi_1^\tau, \Psi_2^\tau, \Psi_3^\tau)(x,\mu, \Sigma)  =  \frac1N \sum_{j=1}^N  (\nabla_x, \nabla_\mu, \nabla_{\mathbb B}^1) k(x,\mu,\Sigma,x^{\tau}_j, 
 \mu^{\tau}_j, \Sigma^{\tau}_j)$\; 
            \FOR{$i = 1, \ldots, N$}
            \STATE Perturb $x_i^{\tau, p} \leftarrow x_i^\tau + \beta_\tau\mc N_{\R^m}(0, 1)$ and $\mu_i^{\tau, p}
            \leftarrow \mu_i^\tau + \beta_\tau \mc N_{\R^n}(0,1)$ 
            \STATE Set $S \leftarrow \beta_\tau \mc N_{\mS^n}(0, 1)$ and perturb $\Sigma_i^{\tau, p} \leftarrow ( I +  \rmL_{\Sigma_i^\tau}[S] ) \Sigma_i^\tau  ( I +  \rmL_{\Sigma_i^\tau}[S] )$ 
                \STATE $x^{\tau+1}_i \leftarrow  x_i^{\tau, p} +s_\tau (\bar\Psi_1 -\Psi_1^\tau)(x^{\tau, p}_i, \mu^{\tau, p}_i, \Sigma^{\tau, p}_i)$\;
                \STATE $\mu^{\tau+1}_i \leftarrow  \mu_i^{\tau, p} + s_\tau (\bar\Psi_2 - \Psi_2^\tau)(x^{\tau, p}_i, \mu^{\tau, p}_i, \Sigma^{\tau, p}_i)$\;
                \STATE $\Sigma^{\tau+1}_i \leftarrow  \Big( I +  s_\tau\rmL_{\Sigma_i^{\tau, p} }\big[(\bar\Psi_3-\Psi_3^\tau)(x^{\tau, p}_i, \mu^{\tau, p}_i, \Sigma^{\tau, p}_i)\big] \Big) \Sigma_i^{\tau, p}  \Big( I + s_\tau \rmL_{\Sigma_i^{\tau, p} }[(\bar\Psi_3-\Psi_3^\tau)(x^{\tau, p}_i, \mu^{\tau, p}_i, \Sigma^{\tau, p}_i)] \Big)$\;

            \ENDFOR
            
        %\STATE Compute $\displaystyle (\Psi_1^{\tau+1}, \Psi_2^{\tau +1}, \Psi_3^{\tau +1}) (x,\mu,\Sigma)  =  \frac1N \sum_{j=1}^N  (\nabla_x, \nabla_\mu,  \nabla_\Sigma) k(x,\mu,\Sigma,x^{\tau+1}_j, \mu^{\tau+1}_j, \Sigma^{\tau+1}_j)$\;
        \STATE Set $\tau \leftarrow \tau+ 1$ 
        \ENDWHILE
 \STATE{\bfseries Output:} $\displaystyle \rho^T = \frac1N \sum_{i=1}^N \delta_{(x^T_i, \mu^T_i, \Sigma^T_i)}$
		
	\end{algorithmic}
\end{algorithm}

\subsection{Kernel and Its Gradient for Implementation} \label{sec:kernel}
%\viet{$y$ is label. Maybe $(x', \mu', \Sigma')$}
We use the kernel $k$ given by:
\[
k\left((x,\mu,\Sigma), (\bar x,\bar\mu,\bar\Sigma)\right) \Let \exp \left(-\alpha \| x-\bar x\|_2^2\right)  \exp \left(-\beta \| \mu-\bar\mu\|_2^2 \right)
\exp \left( - \gamma \|\Sigma - \bar\Sigma\|_2^2 \right),
\]
where $\alpha$, $\beta$ and $\gamma$ are parameters (bandwidth) of the kernel. We note that this kernel is characteristic by \cite[Theorem 4]{JMLR:v18:17-492}.
Then its standard Euclidean gradient is given by
\begin{align*}
\nabla_{(x,\mu,\Sigma)} k\left((x,\mu,\Sigma), (\bar x,\bar\mu,\bar\Sigma)\right) 
=-2 \exp \left(-\alpha \| x-\bar x\|_2^2-\beta \| \mu-\bar\mu\|_2^2 - \gamma \|\Sigma- \bar\Sigma\|_2^2 \right) 
\begin{bmatrix} \alpha (x-\bar x)\\ \beta (\mu-\bar\mu) \\ \gamma (\Sigma-\bar\Sigma)\end{bmatrix}. 
\end{align*}
Thus by plugging into formula \eqref{metric-gradient-dataset}, we obtain 
\begin{align*}\label{Gaussian-kernel-case} 
&\widetilde\nabla_d^1 k\left((x,\mu,\Sigma), (\bar x,\bar\mu,\bar\Sigma)\right)\\
&\quad =-2 \exp \left(-\alpha \| x-\bar x\|_2^2  - \beta \|\mu - \bar\mu\|_2^2 - \gamma \|\Sigma - \bar\Sigma\|_2^2 \right) \begin{bmatrix}  
\alpha (x-\bar x) \\ \beta (\mu - \bar\mu) \\
2\gamma (2\Sigma^2-\Sigma\bar\Sigma -\bar\Sigma\Sigma)\end{bmatrix}. 
\end{align*}
That is, 
\begin{align*}
 &\nabla_x k\left((x,\mu,\Sigma), (\bar x,\bar\mu,\bar\Sigma)\right) 
=-2 \exp \left(-\alpha \| x-\bar x\|_2^2  - \beta \|\mu -\bar\mu\|_2^2 - \gamma \|\Sigma - \bar\Sigma\|_2^2 \right) \alpha (x-\bar x), \\
&\nabla_\mu k\left((x,\mu,\Sigma), (\bar x,\bar\mu,\bar\Sigma)\right) 
=-2 \exp \left(-\alpha \| x-\bar x\|_2^2  - \beta \|\mu - \bar\mu\|_2^2 - \gamma \|\Sigma - \bar\Sigma\|_2^2 \right) \beta (\mu - \bar\mu), \\
& \nabla_{\mbb B}^1 k\left((x,\mu,\Sigma), (\bar x,\bar\mu,\bar\Sigma)\right)\\
&\quad = -2 \exp \left(-\alpha \| x-\bar x\|_2^2  - \beta \|\mu - \bar\mu\|_2^2 - \gamma \|\Sigma -\bar\Sigma\|_2^2 \right) 
2\gamma (2\Sigma^2-\Sigma\bar\Sigma -\bar\Sigma\Sigma).
\end{align*}

\subsection{Label Projection} \label{sec:label-projection}

We here propose an approach to recover new samples in the feature-label space from an empirical distribution in the feature-Gaussian space. Consider that after $T$ iterations of the gradient algorithms, we arrive at a distribution $\rho^T = \frac1N \sum_{i=1}^N \delta_{(x^T_i, \mu^T_i, \Sigma^T_i)}$. We would like to recover a distribution $\nu^T \in \mc P(\mc X \times \mc Y)$ which is induced by $\rho^\tau$. As such, we would like to find a distribution $\nu^T$ of the form
\[
    \nu^T = \frac1N \sum_{i=1}^N \delta_{(x_i^T, y_i^T)}
\]
which corresponds to new target samples $(x_i^T, y_i^T)_{i=1}^N$. Moreover, we are interested in recovering labels within the target domain. To this end, let $\mc Y_{\rm target} = \{ y \in \mc Y: \exists j \in [M] \text{ such that } \bar y_j = y\}$ be the set of labels in the target dataset, and remind that for any $y \in \mc Y_{\rm target}$, $(\bar \mu_{y}, \bar \Sigma_y) \in \R^n \times \mS_+^n$ is the mean vector and the covariance matrix of the distribution of $\phi(X)$ given $Y = y$. Notice that the mean-covariance embeddings $(\bar \mu_{y}, \bar \Sigma_{y})$ for $y \in \mc Y_{\rm target}$ depend only on the target domain data, and it does not depend on the incumbent distribution $\rho^T$, nor does it depend on the source dataset. Moreover, we can also compute $\bar N_y$ as the number of samples from the target dataset with label $y$.

Because $(\bar \mu_{y}, \bar \Sigma_{y})$ is readily computed, we can consider $(\bar \mu_{y}, \bar \Sigma_{y})$ as the centroids and simply find an assignment that minimizes the sum of distances from $(\mu^T_i, \Sigma^T_i)$ to these centroids. We thus can assign each sample from $\rho^T$ to the the target labels by solving the linear program
    \be \label{eq:assignment}
        \begin{array}{cl}
            \min & \displaystyle \sum_{i = 1}^N \sum_{y \in \mc Y_{\rm target}} \theta_{iy} \sqrt{\| \mu^T_i - \bar \mu_y \|_2^2 + \mbb B(\Sigma^T_i, \bar \Sigma_y)^2}  \\
            \text{s.t.} & \displaystyle \sum_{y \in \mc Y_{\rm target}} \theta_{iy} = \frac{1}{N} \quad \forall i = 1, \ldots, N, \qquad \sum_{i=1}^N \theta_{iy} = \frac{\bar N_y}{N} \quad \forall y \in \mc Y_{\rm target}, \quad \theta \in [0, 1]^{N \times |\mc Y_{\rm target}|},
        \end{array}
    \ee
Notice that the assignment problem above does not utilize the information from the covariate $x_i^T$. Let $\theta\opt$ be the optimal solution of the above optimization problem. Then the dataset $(x_i^T, z_i^T)_{i=1}^N$ recovered from $\rho^\tau$ is
\[
        \nu^T = \frac{1}{N} \sum_{i=1}^N \delta_{(x_i^T, y_i^T)}, \qquad y_i^T = \sum_{y \in \mc Y_{\rm target}} y \mathbbm{1}(\theta_{iy}\opt = \max\{\theta_i\opt\} ) \quad \forall i = 1, \ldots, N.\]
We used the POT library to solve the label recovery problem~\eqref{eq:assignment}.

%%%%%%%%%%%%%%%%%%%%%%%%
% \subsection{Additional Numerical Results}\label{sec:additional_numerical}
\subsection{Results on Mixture of Gaussians}\label{sec:gaussian_results}
We test our algorithm on a toy example: a mixture of Gaussian distributions to another mixture of Gaussian distributions. 

The source distribution and target distribution are:
\begin{align*}
    p_s(x)&=\frac{1}{4}\mathcal{N}\left(\begin{pmatrix}
2.0\\
-0.3
\end{pmatrix},\begin{pmatrix}
0.14&-0.00\\
-0.00& 0.22
\end{pmatrix}\right)+\frac{1}{4}\mathcal{N}\left(\begin{pmatrix}
2.0\\
0.3
\end{pmatrix},\begin{pmatrix}
0.43&0.18\\
0.18& 0.26
\end{pmatrix}\right)\\&+\frac{1}{4}\mathcal{N}\left(\begin{pmatrix}
-0.3\\
2.0
\end{pmatrix},\begin{pmatrix}
0.66&0.02\\
0.02& 0.63
\end{pmatrix}\right)+\frac{1}{4}\mathcal{N}\left(\begin{pmatrix}
0.3\\
-2.0
\end{pmatrix},\begin{pmatrix}
0.39&-0.02\\
-0.02& 0.13
\end{pmatrix}\right)\\
 p_t(x)&=\frac{1}{4}\mathcal{N}\left(\begin{pmatrix}
2.9\\
0.1
\end{pmatrix},\begin{pmatrix}
0.16&0.03\\
0.03&  0.20
\end{pmatrix}\right)+\frac{1}{4}\mathcal{N}\left(\begin{pmatrix}
0.9\\
0.5
\end{pmatrix},\begin{pmatrix}
0.22&0.16\\
0.16&  0.46
\end{pmatrix}\right)\\&+\frac{1}{4}\mathcal{N}\left(\begin{pmatrix}
0.8\\
2.2
\end{pmatrix},\begin{pmatrix}
0.63&0.02\\
0.02&  0.66
\end{pmatrix}\right)+\frac{1}{4}\mathcal{N}\left(\begin{pmatrix}
1.4\\
-1.8
\end{pmatrix},\begin{pmatrix}
0.18&0.10\\
0.10&  0.36
\end{pmatrix}\right)
\end{align*}
From each distribution, we sample 25 particles and flow the particles' positions, means, and covariance simultaneously using Alg.~\ref{alg:nonoise}. After the algorithm converges, we recover the particles' label in the feature-label space by solving problem \eqref{eq:assignment}.

\begin{figure}[ht]
\centering
\begin{minipage}{0.32\textwidth}
\centering
\includegraphics[width=\textwidth]{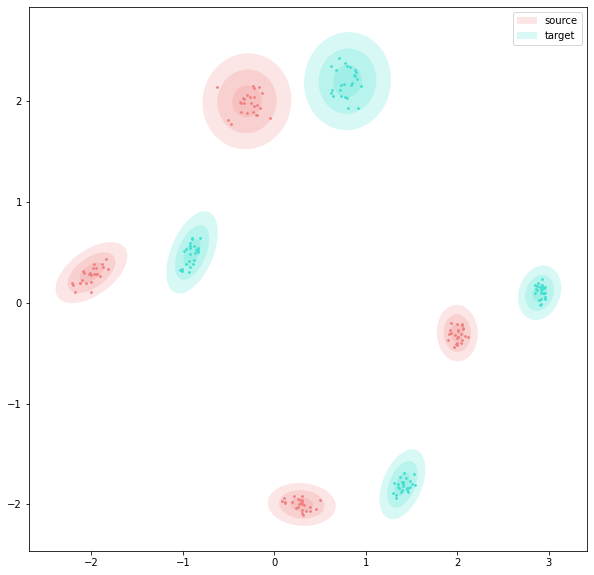}
\end{minipage}
\begin{minipage}{0.32\textwidth}
\centering
\includegraphics[width=\textwidth]{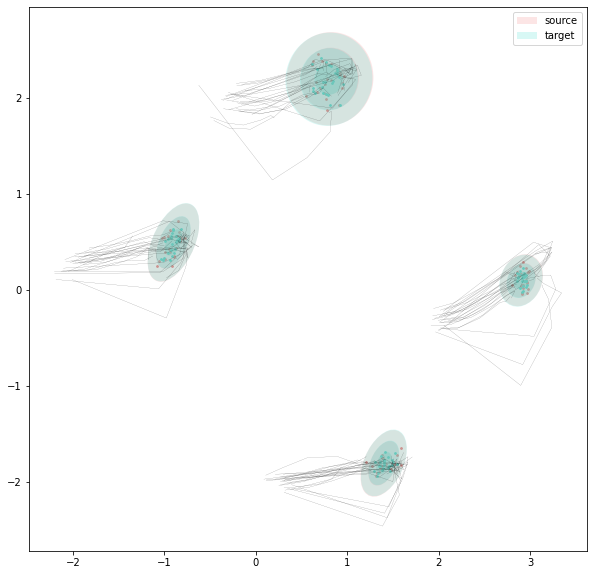}
\end{minipage}
\begin{minipage}{0.32\textwidth}
\centering
\includegraphics[width=\textwidth]{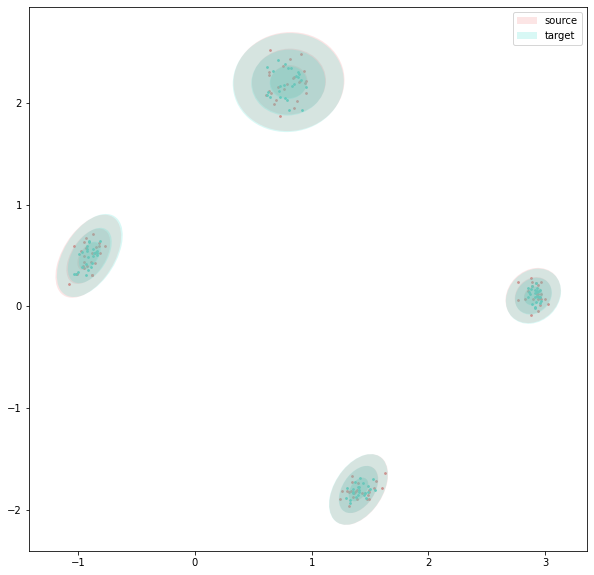}
\end{minipage}
\caption{The results of flowing a mixture of 4 Gaussian distributions to a mixture of 4 Gaussian distributions. We demonstrate the initialization (left), the trace of particles in first 200 steps (middle), and the results at step 1000 (right).}
\label{fig:4gauss}
\end{figure}
We test how our algorithm deals with flowing a mixture of 2 Gaussian distributions to a mixture of 4 Gaussian distributions. From the trace of first 200 steps, we demonstrate that each source Gaussian distribution splits into 2 Gaussian distributions.
The source distribution and target distribution are:
\begin{align*}
    p_s(x)&=\frac{1}{2}\mathcal{N}\left(\begin{pmatrix}
0.0\\
0.0
\end{pmatrix},\begin{pmatrix}
0.18&-0.24\\
-0.24& 0.70
\end{pmatrix}\right)+\frac{1}{2}\mathcal{N}\left(\begin{pmatrix}
5.8\\
0.0
\end{pmatrix},\begin{pmatrix}
0.44&0.00\\
0.00& 0.87
\end{pmatrix}\right)\\
 p_t(x)&=\frac{1}{4}\mathcal{N}\left(\begin{pmatrix}
2.0\\
0.7
\end{pmatrix},\begin{pmatrix}
0.63&-0.30\\
-0.30&  0.26
\end{pmatrix}\right)+\frac{1}{4}\mathcal{N}\left(\begin{pmatrix}
2.2\\
-0.8
\end{pmatrix},\begin{pmatrix}
0.77&-0.18\\
-0.18&  0.55
\end{pmatrix}\right)\\&+\frac{1}{4}\mathcal{N}\left(\begin{pmatrix}
7.0\\
0.8
\end{pmatrix},\begin{pmatrix}
0.63&-0.30\\
-0.30&  0.26
\end{pmatrix}\right)+\frac{1}{4}\mathcal{N}\left(\begin{pmatrix}
7.7\\
-0.8
\end{pmatrix},\begin{pmatrix}
0.77&-0.18\\
-0.18&  0.55
\end{pmatrix}\right)
\end{align*}

\begin{figure}[!ht]
\centering
\begin{minipage}{0.32\textwidth}
\centering
\includegraphics[width=\textwidth]{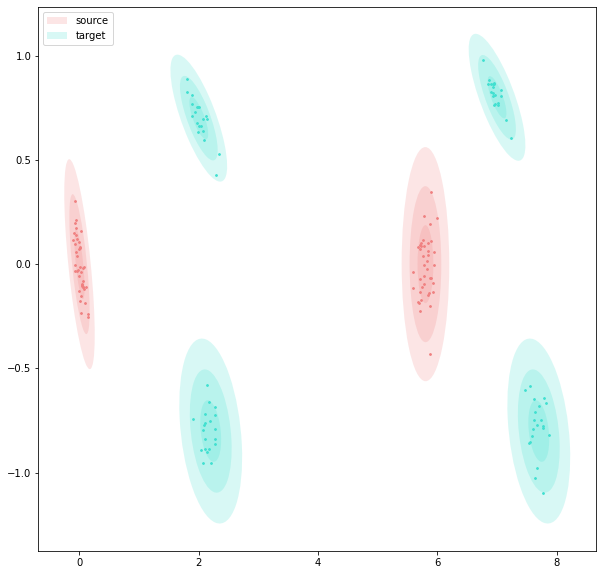}
\end{minipage}
\begin{minipage}{0.32\textwidth}
\centering
\includegraphics[width=\textwidth]{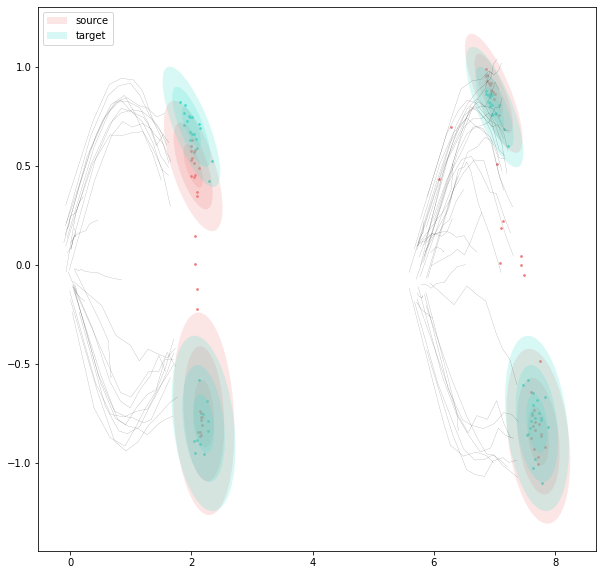}
\end{minipage}
\begin{minipage}{0.32\textwidth}
\centering
\includegraphics[width=\textwidth]{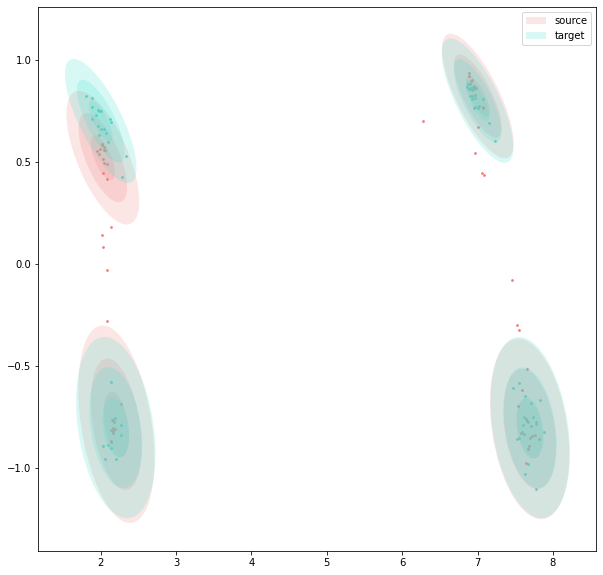}
\end{minipage}
\caption{The results of flowing a mixture of 2 Gaussian distributions to a mixture of 4 Gaussian distributions. We demonstrate the initialization (left), the trace of particles in first 200 steps (middle), and the results at step 2400 (right). We use method in Section~\ref{sec:label-projection} to relabel the flowed samples and the labels correspond to their positional belongings correctly.}
\label{fig:4gauss-part2}
\end{figure}

\newpage
\subsection{Implementation Details}\label{sec:implementation_details}
Our algorithms and experiments are implemented in PyTorch. The *NIST images are resized to $20\times20$, thus the feature space is of dimension $m = 400$. The SVHN and Tiny ImageNet images are of size $3\times 64\times 64$. The classifiers for *NIST datasets are of LeNet-5 architecture. The classifiers for SVHN and Tiny ImageNet are of a ResNet-18 network. For the experiments in Figure 3, we train the classifiers with the set of data 10 epochs with Adam optimizer and learning rate~$2\times 10^{-3}$.

When flowing images in *NIST datasets and flowing a mixture of Gaussians, we use the parameters and methods described in Table~\ref{tab:hyper}. As we tune the parameters, we notice the algorithm converges with a range of parameters and we report one setting in Table~\ref{tab:hyper}.
\begin{table*}[!ht]
    \newcommand{\mround}[1]{\round{#1}{2}}
    \newcommand{\nastar}{\multicolumn{1}{c}{\hspace{15pt}--~*}}
    \newcommand{\na}{\multicolumn{1}{c}{--}}
    \centering \small
    {
  \centering \small
\resizebox{
  \ifdim\width>\textwidth
    \textwidth
  \else
    \width
  \fi
}{!}{%
\begin{tabular}{r @{\extracolsep{5pt}} r r r r @{}}
\toprule
& \shortstack{\bf SVHN\&TIN}
& \shortstack{\bf *NIST}
& \shortstack{\bf Gaussian (Figure~\ref{fig:4gauss})}
& \shortstack{\bf Gaussian (Figure~\ref{fig:4gauss-part2})}
\\ 
\midrule
$\alpha$
& 0.002
& 0.001
& 0.3
& 0.3
\\
$\beta$
& 0.01
& 0.002
& 0.15
& 0.1
\\
$\gamma$
& 1.0
& 100
& 1.0
& 0.5
\\
initial $s_{\tau}$
& 1.0
& 0.3
& 0.05
& 0.03
\\
noise level
& 0.1
& 0.01
& 0 
& 0.1
\\
$T$
& 6000
& 150
& 2000
& 2500
\\
Optimizer
&  RMSprop~\cite{hinton2012neural}
& RMSprop
& RMSprop
& RMSprop
\\
\bottomrule
\end{tabular}
}
 \caption{Parameters and Optimizer}
  \label{tab:hyper}
    }
\end{table*}

Our method assumes images of each class form one Gaussian distribution. In reality, the data can be a mixture of Gaussian. To satisfy the Gaussianity assumption, in the preprocessing step, we use a clustering method ($k$-nearest neighbors) and pick only data from one mode for each class. As a consequence, the data used in the experiment satisfies the conditional Gaussian assumption. For example, the images of the digit 1 can have two modes: slanted left or slanted right. In this case, we can generate two labels (1L, 1R), and the methodology developed in this paper can be applied in a straightforward manner. When testing our transfer learning scheme, we apply the same clustering method on the test dataset, so our test set is within the same mode as our training set. 

We store the preprocessed data and apply dimension reduction method on the data's means and covariance matrices, so the Lyapunov equation is much faster to solve. We use the cluster's mean and covariance matrix to approximate the 1-shot and 5-shor data's mean and covariance matrix. In 1-shot learning, the covariance matrix is an identity matrix. All the code and data are available in the supplementary file. We use $k$-nearest neighbors algorithm to solve the labels of the flowed data, as it performs better with the noisy scheme.

\newpage
\subsection{Additional Results on *NIST Datasets}
We conduct additional experiments of flowing between KMNIST and FashionMNIST datasets. The results of our flows are illustrated in the same fashion as Figure 2 and are in the supplementary folder. In each subfigure, each column represents a snapshot of a certain time-step and the samples flow from the source (left) to the target (right). To check if the algorithm is converging, we compute the MMD between the source dataset and the target dataset. Figure~\ref{fig:MMD} is an example of MMD decreasing in transferring the *NIST datasets:
\begin{figure}[!ht]
\centering
\includegraphics[width=0.4\textwidth]{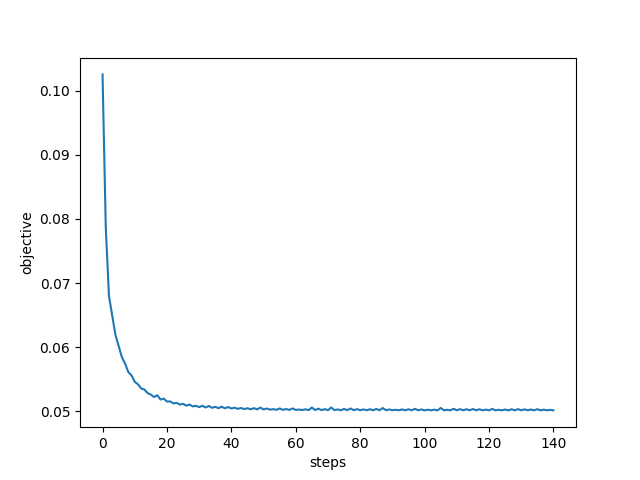}
\caption{MMD keeps decreasing as we flow the samples from the source domain to the target domain}
\label{fig:MMD}
\end{figure}

% \begin{figure}[!ht]
% \centering
% \begin{minipage}{0.49\textwidth}
% \centering
% \includegraphics[width=\textwidth]{appendix_graph/K_F.png}
% \end{minipage}
% \begin{minipage}{0.49\textwidth}
% \centering
% \includegraphics[width=\textwidth]{appendix_graph/F_K.png}
% \end{minipage}
% \caption{Sample path visualizations between FashionMNIST dataset and KMNIST dataset}
% \label{fig:KMNIST_FMNIST_flow}
% \end{figure}
% We also attach the same results as Fig.~\ref{fig:NIST_flow} in high resolution in Fig.~\ref{fig:K_MNIST_flow} and Fig.~\ref{fig:F_MNIST_flow}. Each picture illustrates one experiment of gradient flow between two datasets and the samples flow from the source (left) to the target (right).
% \begin{figure}[!ht]
% \centering
% \begin{minipage}{0.49\textwidth}
% \centering
% \includegraphics[width=\textwidth]{appendix_graph/M_K.png}
% \end{minipage}
% \begin{minipage}{0.49\textwidth}
% \centering
% \includegraphics[width=\textwidth]{appendix_graph/K_M.png}
% \end{minipage}
% \caption{Sample path visualizations between KMNIST dataset and MNIST dataset}
% \label{fig:K_MNIST_flow}
% \end{figure}
% \begin{figure}[!ht]
% \centering
% \begin{minipage}{0.49\textwidth}
% \centering
% \includegraphics[width=\textwidth]{appendix_graph/M_F.png}
% \end{minipage}
% \begin{minipage}{0.49\textwidth}
% \centering
% \includegraphics[width=\textwidth]{appendix_graph/F_M.png}
% \end{minipage}
% \caption{Sample path visualizations between FashionMNIST dataset and MNIST dataset}
% \label{fig:F_MNIST_flow}
% \end{figure}
From FashionMNIST dataset to KMNIST dataset and from KMNIST dataset to FashionMNIST dataset, we also conduct transfer learning experiments. We use the same model architecture and training settings as in Fig.~\ref{fig:few_shot}. We illustrate the accuracy and error bars of the 1-shot learning and 5-shot learning in Fig.~\ref{fig:few_shot_K_F}. Our flowed samples ($S_T$) increase the accuracy of the transferred classifiers in both 1-shot and 5-shot learning.
\begin{figure}[!ht]
\centering
\begin{minipage}{\textwidth}
\centering
\includegraphics[width=0.9\textwidth]{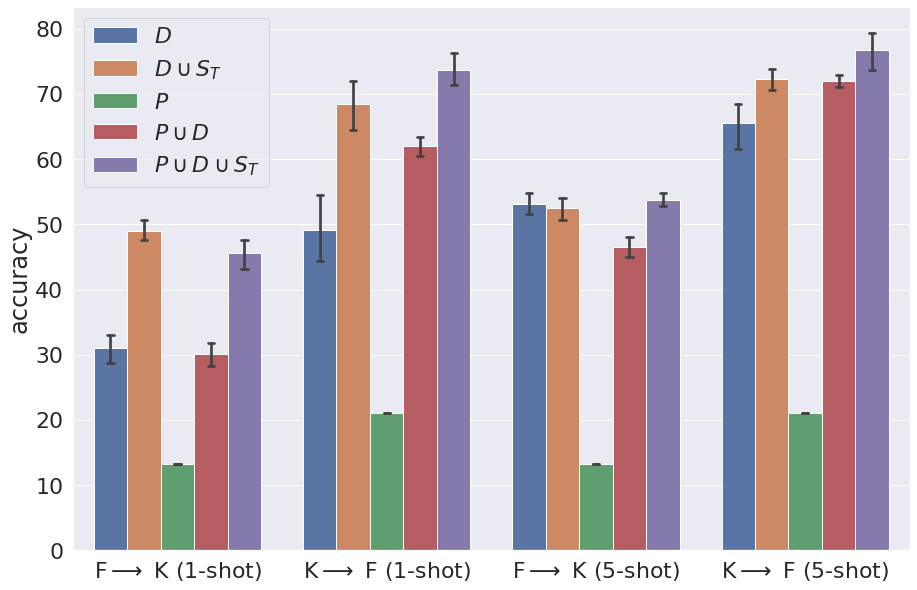}
\end{minipage}
\caption{Average target domain accuracy on the test split for transfer learning with one-shot (left) and ﬁve-shot (right). Results are taken over 10 independent replications, and the range of accuracy is displayed by the error bars.}
\label{fig:few_shot_K_F}
\end{figure}

\newpage
\subsection{Transfer Learning Results on SVHN and Tiny ImageNet}\label{sec:SVHN_TIN_supp}

We run the same transfer learning experiments as section 5.1 on SVHN and Tiny ImageNet datasets and we report the accuracy on the test split\footnote{The Tiny ImageNet dataset does not have labels for the test split, so we use the validation split.}. From the results, we can see $S_T$ always improve the accuracy.

On the higher-dimensional data, it is a lot more difficult to train a strong classifier from only 1 or 5 samples each class, so the accuracy is not so high compared to the *NIST datasets. Another reason is there are only 50 samples in each class of the validation dataset of Tiny ImageNet, so the number of test samples might be too small.

\begin{figure}[!ht]
\centering
\begin{minipage}{\textwidth}
\centering
\includegraphics[width=0.9\textwidth]{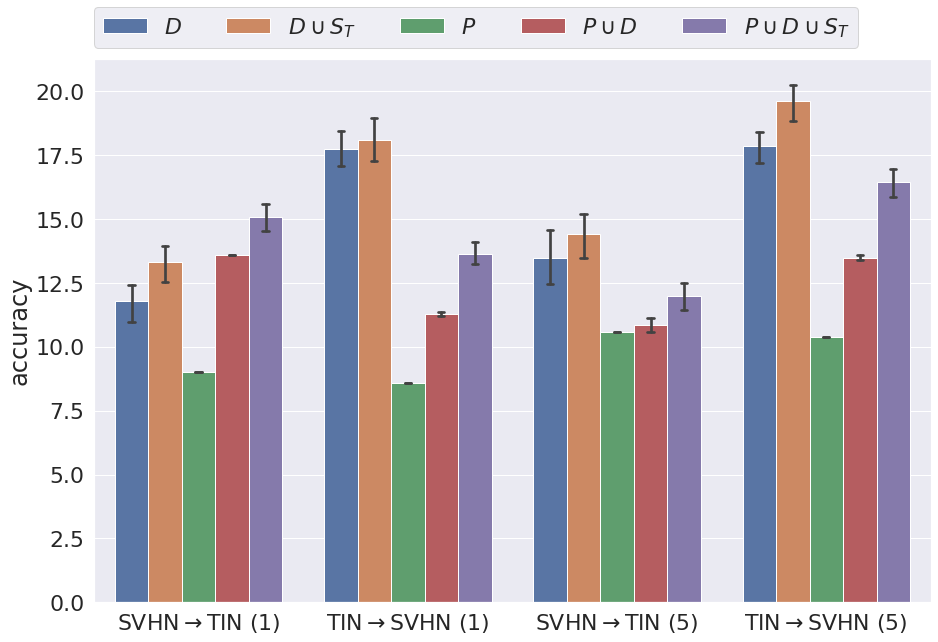}
\end{minipage}
\caption{Average target domain accuracy on the test split for transfer learning with one-shot (left) and ﬁve-shot (right). Results are taken over 10 independent replications, and the range of accuracy is displayed by the error bars. The $\bm{1}$ and $\bm{5}$ in the parenthesis mean 1-shot and 5-shot transfer learning respectively.}
\label{fig:few_shot_K_F}
\end{figure}

\newpage
\subsection{Comparison with baseline} \label{sec:comparison}
We compare with~\cite{AlvarezGF20}'s approach as the baseline. We adopted the same values of parameters (number of steps, step size) from the paper of \cite{AlvarezGF20} and experimented with different values of entropy regularization $\lambda$. Entropy regularization $\lambda$ is a hidden parameter in their code and they reported using the value of $\lambda=100$ in transfer learning experiments. Also, we experimented with different methods in their code and found that “xyaugm” gives the best qualitative gradient flow results. 

We test their algorithm using the same transfer learning setting as ours (using the same clustered data as our experiments), see  Table~\ref{tab:K2M} and  Table~\ref{tab:F2M} for results. The results show that their method doesn't improve the classification accuracy after adding the transferred data $S_T$, if we compare the accuracy $D\cup S_T$ and $P\cup D\cup S_T$ with $D$ and $P\cup D$ correspondingly.

\begin{table}[!ht]
    \centering
    \begin{tabular}{lrr}
        \toprule
        Accuracy  & $D\cup S_T$ & $P\cup D\cup S_T$ \\
        \midrule
$\lambda=0.001$   & 0.1941 & 0.2478\\
$\lambda=0.01$   & 0.1175   & 0.2998 \\
$\lambda=1.0$  & 0.1739 & 0.2951\\
$\lambda=100$   & 0.2093 & 0.3025 \\
        \bottomrule
    \end{tabular}
    \caption{KMNIST $\shortrightarrow$ MNIST}
    \label{tab:K2M}
\end{table}
\begin{table}[!ht]
    \centering
    \begin{tabular}{lrr}
        \toprule
        Accuracy  & $D\cup S_T$ & $P\cup D\cup S_T$ \\
        \midrule
$\lambda=0.001$   & 0.4488  &  0.3083 \\
$\lambda=0.01 $   & 0.3915  & 0.2897 \\
$\lambda=1.0$  & 0.5268 & 0.4627 \\
$\lambda=100  $   & 0.2917 & 0.3307 \\
        \bottomrule
    \end{tabular}
    \caption{FMNIST $\shortrightarrow$ MNIST}
    \label{tab:F2M}
\end{table}

% \begin{table}[!ht]
% %\begin{wraptable}{r}{0.7\linewidth}
%     \centering
% %\begin{minipage}[t]{0.33\linewidth}\centering
% \begin{minipage}[t]{0.48\linewidth}\centering
% \caption{KMNIST $\shortrightarrow$ MNIST}
% \label{tab:K2M}
% \begin{tabular}{rrr}
% \multicolumn{1}{c}{\bf Accuracy}  &\multicolumn{1}{c}{$D\cup S_T$} &\multicolumn{1}{c}{$P\cup D\cup S_T$}
% \\ \hline \\
% $\lambda=0.001$   & 0.1941 & 0.2478\\
% $\lambda=0.01$   & 0.1175   & 0.2998 \\
% $\lambda=1.0$  & 0.1739 & 0.2951\\
% $\lambda=100$   & 0.2093 & 0.3025 \\
% \end{tabular}
% \end{minipage}\hfill%
% %\begin{minipage}[t]{0.33\linewidth}\centering
% \begin{minipage}[t]{0.48\linewidth}\centering
% \caption{FMNIST $\shortrightarrow$ MNIST}
% \label{tab:F2M}
% \begin{tabular}{rrr}
% \multicolumn{1}{c}{\bf Accuracy}  &\multicolumn{1}{c}{$D\cup S_T$} &\multicolumn{1}{c}{$P\cup D\cup S_T$}
% \\ \hline \\
% $\lambda=0.001$   & 0.4488  &  0.3083 \\
% $\lambda=0.01 $   & 0.3915  & 0.2897 \\
% $\lambda=1.0$  & 0.5268 & 0.4627 \\
% $\lambda=100  $   & 0.2917 & 0.3307 \\
% \end{tabular}
% \end{minipage}
% \end{table}
% %\end{wraptable}

For runtime comparison, the default device for \cite{AlvarezGF20}’s code is on the CPU. As we run their code on the GPU, the kernel crashed without giving any informative errors. Thus, we will compare our codes’ runtime per step on the CPU. While our approach takes about $0.512$ second, the approach in \cite{AlvarezGF20} requires $74.78$ seconds.

\newpage
\subsection{Comparison with Mixup Method}
We run mixup augmentation on F$\shortrightarrow$M, K$\shortrightarrow$M, M$\shortrightarrow$F, and M$\shortrightarrow$K on the same one-shot learning setting and same data as ours using the method in~\cite{zhang2017mixup} with $\alpha=0.2$ (suggested by~\cite{zhang2017mixup}). The results are in Table~\ref{tab:F2M}. When comparing $D$ with $D\cup S_T$ and $P\cup D$ with $P\cup D\cup S_T$, the added mixup samples ($S_T$) decrease the accuracy. A reason could be that there are too few examples in the target domain, thus mixup does not add much complexity to the target domain. The mixup experiments in [A] show improvements $<1\%$, which is much smaller than ours in Fig.3. 

Last, we want to emphasize that our gradient flow method aims to increase the number of samples of one distribution, so our method is a complement to augmentation methods.Both gradient flow and augmentation methods are applied in practice at the same time.

\begin{table}[!ht]
    \centering
    \begin{tabular}{lrrrr}
        \toprule
        Training Data & F$\shortrightarrow$M & K$\shortrightarrow$M & M$\shortrightarrow$F &  M$\shortrightarrow$K  \\
        \midrule
$D$ &  0.524 & 0.529 & 0.506 & 0.242\\
$D\cup S_T$   & 0.488 & 0.458 & 0.472 & 0.141\\
$P$     &    0.097 & 0.033 & 0.037 & 0.089 \\
$P\cup D$   & 0.510 & 0.543 & 0.541 &0.140 \\
$P\cup D\cup S_T$  & 0.467 & 0.363 & 0.494 & 0.120 \\
        \bottomrule
    \end{tabular}
    \caption{FMNIST $\shortrightarrow$ MNIST}
    \label{tab:F2M}
\end{table}

% D: 0.524
% D & S_T: 0.488
% P & D: 0.510
% P, D & S_T: 0.467
\newpage
\subsection{Comparison with Traditional Augmentation Methods}\label{sec:comparison_transformation}
We conduct the traditional data augmentation methods, including random rotation and Gaussian noise, on the datasets and compare the classification accuracy with our method. The random rotation is from 0 to 90 degrees. The Gaussian noise is with kernel size (5,9) and standard deviation 0.1 and 5. All the comparison results are in Table~\ref{tab:1_shot_rotation}--\ref{tab:5_shot_gaussian}.

Through this experiment, we show that using random rotation to augment the dataset does not improve accuracy much and sometimes even decreases the accuracy in all of the eight transfer learning settings. We note that ($D$, $P$, $P\cup D$) are settings without augmentation while ($D\cup S_T$, $P\cup D\cup S_T$) are settings with augmentation. Comparing with the results of our proposed gradient flow approach in Figure~\ref{fig:few_shot}, augmentation (by random rotation) has much lower performances than ours.

\begin{table}[!ht]
 \centering
\begin{tabular}{rrrrr}
\toprule
Accuracy & $F\rightarrow M$ &$K\rightarrow M$ &$M\rightarrow F$ &$M\rightarrow K$\\
\midrule
$D$   &0.535 &0.512 &0.483 &0.240 \\
$D\cup S_T$   &0.542 &0.498 & 0.489 & 0.164 \\
$P$  &0.097 &0.033 &0.037 &0.089 \\
$P\cup D$   &0.510 &0.543 & 0.541 &0.140 \\
$P\cup D \cup S_T$ &0.551 &0.581 &0.454 & 0.124\\
\bottomrule
\end{tabular}
\caption{1-shot learning augmented with random rotation}
\label{tab:1_shot_rotation}
\end{table}

\begin{table}[!ht]
\centering
\begin{tabular}{rrrrr}
\toprule
Accuracy & $F\rightarrow M$ &$K\rightarrow M$ &$M\rightarrow F$ &$M\rightarrow K$\\
\midrule
$D$  &0.788 &0.837 &0.519 &0.234\\
$D\cup S_T$   &0.569 &0.545 &0.490 &0.174 \\
$P$  &0.097 &0.033 &0.037 &0.089 \\
$P\cup D$  &0.772 &0.763 &0.581 &0.178\\
$P\cup D \cup S_T$ &0.738 &0.768 &0.593 &0.199\\
\bottomrule
\end{tabular}
\caption{5-shot learning augmented with random rotation}
\label{tab:5_shot_rotation}
\end{table}
 
\begin{table}[!ht]
\centering
\begin{tabular}{rrrrr}
\toprule
Accuracy & $F\rightarrow M$ &$K\rightarrow M$ &$M\rightarrow F$ &$M\rightarrow K$\\
\midrule
$D$  &0.524 &0.529 &0.506 &0.242\\
$D\cup S_T$  &0.512 &0.494 &0.452 &0.167 \\
$P$  &0.097 &0.033 &0.037 &0.089 \\
$P\cup D$  &0.510 &0.543 &0.541 &0.140\\
$P\cup D \cup S_T$ &0.557 &0.590 &0.454 &0.127\\
\bottomrule
\end{tabular}
\caption{1-shot learning augmented with random Gaussian noise}
\label{tab:1_shot_gaussian}
\end{table}

\begin{table}[!ht]
\centering
\begin{tabular}{rrrrr}
\toprule
Accuracy & $F\rightarrow M$ &$K\rightarrow M$ &$M\rightarrow F$ &$M\rightarrow K$\\
\midrule
$D$  &0.827 &0.824 &0.549 &0.265\\
$D\cup S_T$   &0.792 &0.803 &0.734 &0.320\\
$P$  &0.097 &0.032 &0.037 &0.089\\
$P\cup D$  &0.779 &0.764 &0.581 &0.178\\
$P\cup D \cup S_T$ &0.832 &0.797 &0.729 &0.258\\
\bottomrule
\end{tabular}
\caption{5-shot learning augmented with random Gaussian noise}
\label{tab:5_shot_gaussian}
\end{table}

We show that adding Gaussian noise improves the accuracy in some settings, for example, M$\rightarrow$F and M$\rightarrow$K in 5-shot, but decreases accuracy in many settings. Our gradient flow method improves the accuracy in all the settings when we compare the accuracy without $S_T$ and with $S_T$. 
 
Finally, we again note that those simple augmentations, for example (adding Gaussian noise or using random rotation which are designed onto the target dataset) can be used as a \textbf{complement} to the gradient flow method, which transfers data from the source dataset to the target dataset.

\newpage
\subsection{Random Gaussian Noise to MNIST}\label{appendix:random_gauss_mnist}

In this section, we conduct an experiment that transfers from random noise vectors to MNIST images. The noise vectors are independently drawn from the normal distribution. Thus, the mean of the noise vectors is a zero vector and the covariance matrix is an identity matrix. From this experiment, we want to show that our method can transfer from any distribution to the target distribution.

\begin{figure}[!ht]
\centering
\begin{minipage}{0.5\textwidth}
\centering
\includegraphics[width=\textwidth]{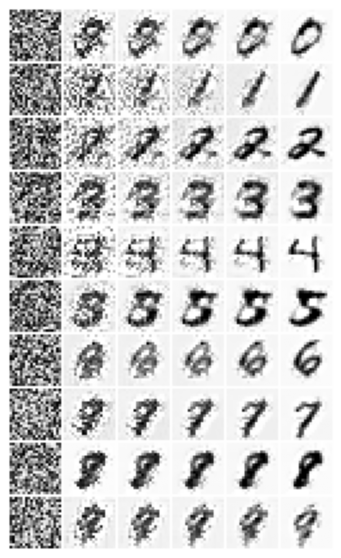}
\end{minipage}
\caption{Our method successfully transfers from random noise vectors to MNIST images.}
\label{fig:randomGauss_MNIST}
\end{figure}

\newpage
\subsection{Evaluation Metric by SSIM and Optimal Transport}
We construct a new measure that connects a human perception distance and optimal transport to measure the distance between the distribution of transferred samples and the distribution of target samples.

The target data consists of $M$ samples $(\bar x_j, \bar y_j)_{j=1}^M$, and the generated data consists of $N$ samples $(x_i^T, y_i^T)_{i=1}^N$. We use the structural similarity index measure (SSIM)~\cite{wang2004image} to measure image similarity between the generated images and the target images. SSIM measures the similarity between two grayscale images. SSIM is between 0 and 1, and $\mathrm{SSIM}=1$ means that the two images are identical. Thus, we define a SSIM-based cost to construct a matrix $C\in\mathbb{R}^{N\times M}$:
$
    C_{ij} = 1-\mathrm{SSIM}(x_i^T,\bar x_j), \text{for } i=1, \ldots, N;~j=1, \ldots, M.
$ 

We use the matrix $C$ as a ground transportation cost and solve the Earth Mover Distance problem~\cite{rubner1998metric,rubner2000earth} to obtain the Wasserstein distance between target image and generated image distributions, which are 
$
    \frac{1}{N} \sum_{i=1}^N \delta_{(x_i^T, y_i^T)}$and $\frac{1}{M} \sum_{j=1}^M \delta_{(\bar x_j, \bar y_j)}.
$
We compute the distance in 1-shot FashionMNIST to MNIST experiment and Figure~\ref{fig:OT} plots the distance of each class over 10 independent runs. The plot shows our algorithm successfully decreases the Wasserstein distance between the distribution of generated images and the target distribution during each step.
\begin{figure}[!ht]
\centering
\includegraphics[width=0.7\columnwidth]{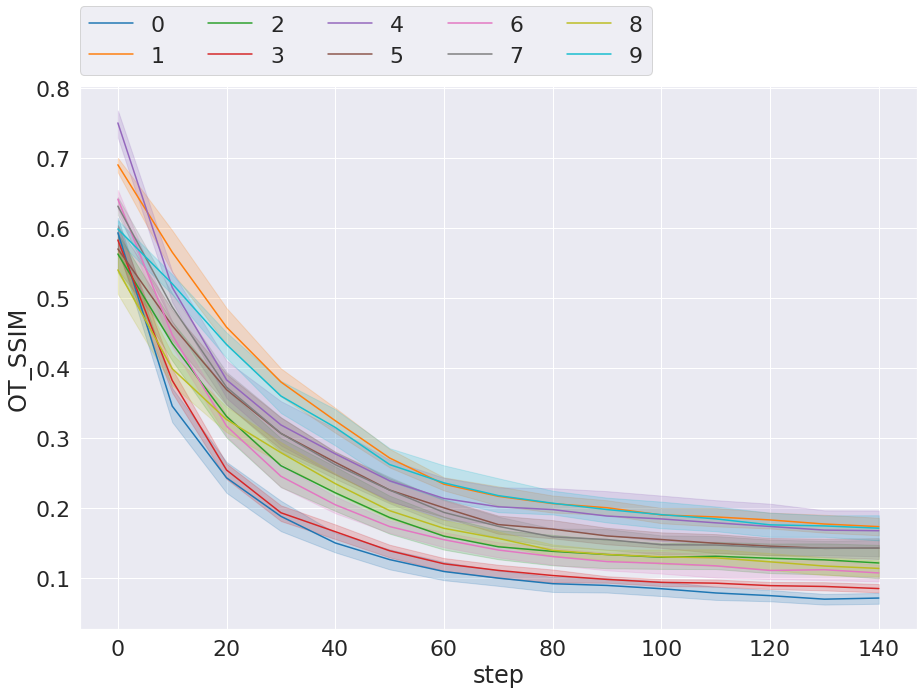}
\caption{This plots the OT\_SSIM distance between the distribution of transferred images and of the target images over 140 steps. The distance decreases as we flow the samples from the source domain to the target domain for every class.}
\label{fig:OT}
\end{figure}

\end{document}